\pdfoutput=1

\documentclass[11pt]{article}

\usepackage[preprint]{acl}

\usepackage{times}
\usepackage{latexsym}

\usepackage[T1]{fontenc}

\usepackage[utf8]{inputenc}

\usepackage{microtype}

\usepackage{inconsolata}

\usepackage{graphicx}

\usepackage{enumitem}

\usepackage{hyperref}
\usepackage{url}

\usepackage{ebproof}
\usepackage{caption}
\usepackage{subcaption}
\usepackage{booktabs}
\usepackage{wrapfig}
\usepackage{algpseudocode}
\usepackage{algorithm}

\usepackage{amsmath}
\usepackage{amssymb}
\usepackage{mathtools}
\usepackage{amsthm}
\usepackage{float}

\theoremstyle{plain}
\newtheorem{theorem}{Theorem}[section]

\newtheorem{corollary}[theorem]{Corollary}
\theoremstyle{definition}
\newtheorem{definition}[theorem]{Definition}
\newtheorem{assumption}[theorem]{Assumption}
\theoremstyle{remark}

\title{Pelican Soup Framework: \\
A Theoretical Framework for Language Model Capabilities}

\author{Ting-Rui Chiang \\
  University of Southern California \\
  \texttt{tingruic@usc.edu} \\\And
  Dani Yogatama \\
  University of Southern California \\
  \texttt{yogatama@usc.edu} \\}

\begin{document}
\maketitle
\begin{abstract}
In this work, we aim to better understand how pretraining allows LLMs to (1) generalize to unseen instructions and (2) perform in-context learning, even when the verbalizers are irrelevant to the task.
To this end, we propose a simple theoretical framework, Pelican Soup, based on the logical consistency of the training data, the notion of ``reference-meaning association'', and a simple formalism for natural language processing tasks.
Our framework demonstrates how linguistic and psychology studies can inform our understanding of language models and is connected to several other existing theoretical results.
As an illustration of the usage of our framework, we derive a bound on in-context learning loss with our framework.
Finally, we support our framework with empirical experiments and provide possible future research directions.
\end{abstract}

\section{Introduction}

Large language models (LLMs) have demonstrated the capability to perform downstream natural language processing (NLP) tasks.
By following instructions, LLMs can perform tasks with zero-shot examples, demonstrating their reasoning capability.
With some input-output examples provided in the prompt, LLMs can also perform tasks without instructions, which is known as in-context learning (ICL) \citep{chowdhery2022palm}.
In particular, \citet{icl} show that LLMs can perform ICL for classification tasks even when the verbalizers (labels present in the demonstration) are semantically irrelevant to the task, e.g., foo/bar instead of negative/positive \cite{wei2023larger}.
However, it is unclear how pretraining leads to these capabilities.


To explain how LLMs acquire these capabilities, we propose a simple theoretical framework, the Pelican Soup framework in \S\ref{sec:framework}. 
Our framework is based on some very general assumptions, such as the logical consistency of training documents (Assumption~\ref{assumption:consistency}) and the variable meanings with which a phrase can be associated in different contexts (Assumption~\ref{assumption:pronoun}).
Our framework also includes a simple formalism for NLP tasks, which helps explain why LMs can follow instructions.


In \S\ref{sec:icl-bound}, we show how we can use this framework to analyze LLM's ICL capability, which mitigates the limitations of previous theoretical analyses. 
For example, in the first theoretical analyses of ICL, \citet{xie2022an} assumes that the general text for training LMs is from a hidden Markov model (HMM), which may be an oversimplification of natural language.
In comparison, our framework does not require this strong assumption.

Our framework offers unique insights compared to existing theories.
For example, although the generation process by \citet{zhang2023and} is more general than the HMM assumption by \citet{xie2022an}, it lacks groundings in real-world linguistic phenomena.  
Our framework mitigates this limitation, as it helps us better explain the physical meaning of the terms in the bound on the ICL loss and shows how the terms reflect real-world practices, such as the choice of verbalizers and the distribution of test examples.
Compared to the work by \citet{hahn2023theory}, our theory allows us to bound ICL loss without relying assumptions on the grammar that generate the data.

Furthermore, in \S\ref{sec:extension}, inspired by the cognitive science theories \citet{Fodor1975-FODTLO,LOT2,computational-origin-of-repr}, early development of artificial intelligence ~\citep{SISKIND199639,murphy2004big} and formal linguistics~\citet{carnap1968logische,bresnan1982mental,steedman1987combinatory,steedman1996surface,sag1999syntactic}, we provide an extension of our framework to explain why generalization is possible.
The extension also connects our framework to other theoretical results.
For example, our extension instantiates the \textit{complex skills} in the theory by \cite{arora2023theory}.
In \S\ref{sec:desc-length}, with an additional assumption, we can relate the ICL bound to the description length of the underlying input-output mapping function as done by \citet{hahn2023theory}.


Our work informs future LLM research directions.
Scientifically, we shed light on how properties of the training data contribute to the ICL and the instruction-following capabilities of LLMs.
The framework also shows how linguistic and psychology studies can inform our understanding of modern NLP.
Practically, we highlight the importance of acquiring knowledge about the interrelation between concepts through pretraining.
As shown in previous studies, the language modeling objective is inefficient for knowledge acquisition \cite{allen2023physics,chiang-etal-2024-retrieval}.
We suggest developing a better pretraining technique is crucial for future NLP development.  
Our proposed experimental setups can also facilitate future studies on LLMs' acquisition of the capabilities.

\section{The Pelican Soup Framework}
\label{sec:framework}

We aim our theoretical framework at explaining why LLMs can perform well on prompts for downstream tasks even though the prompts have a different distribution than the training corpus. 
Therefore, our framework includes assumptions qualifying the training corpus distribution in (\S{\ref{sec:dist-train}}) and a formalism for NLP tasks (\S{\ref{sec:dist-task}}).
Later, we will show how this framework allows us to bound the loss of ICL.

\section{Motivation}

The Pelican Soup game inspires our framework.
It is a game involving a puzzle master who has a story in their mind.
The game participants ask the puzzle master questions.
The goal of the participants is to recover the story based on the answer to the yes/no questions.
An observation is that once the participants recover the story, they can answer any questions about the story.
This is similar to ICL, where the model is required to ``figure out'' the underlying task based on the input-output pairs in the given demonstration.
Once the model identifies the underlying task, it will be able to predict the label of any input.
Therefore, the story has a similar role as a latent variable defining the task, and the yes/no questions are similar to the demonstrations for ICL. 
We include an example in \S\ref{sec:pelican-soup-example}.

Given the above observation, we can study ICL by considering why humans can solve Pelican Soup riddles.
We conjecture that this is because the person who makes the story and the ones who solve the riddle share the same (or similar) commonsense \citep{McCarthy1960ProgramsWC} about logical relationships among the events in this world \citep{schank1988scripts}.
Based on this shared commonsense, the game participants can eliminate hypothesis contradictory to the information that the asked yes/no questions reveal, which allows the participants to uncover the underlying story.
This idea of ``shared commonsense'' inspires us to introduce the notion of ``no ambiguity'' and ``consistency'' in our framework.




\subsection{Training Data Distribution}
\label{sec:dist-train}

Our theory framework is based on the interrelations between the semantics of sentences.
Thus, we first make a general assumption:
\begin{assumption}[No ambiguity]
    \label{assumption:sentence}
    For all pairs of sentences $x_1, x_2$ in a language, we assume humans can determine the relationship between $x_1$ and $x_2$ is contradiction, entailment or neutral.
\end{assumption}

Assumption~\ref{assumption:sentence} allows us to qualify sentences that can cooccur in a paragraph with non-zero probability.
That is, a paragraph generally does not contain self-contradictory information:
\begin{assumption}[Consistency]
    \label{assumption:consistency}
    Any paragraph with non-zero probability mass does not contain two sentences contradicts to each other. 
\end{assumption}

To show how modeling natural language leads to the ICL capability, we further introduce the notion of expression-meaning association as a latent variable. 
It reflects the fact that language allows us to associate meanings with the surface form expressions quite freely.
For example, when ``she'' or the human name ``Emily'' is present in a paragraph, it is associated with a certain person of certain characteristics, which exhibit its meaning.

Meanwhile, the usage of the expression depends on the meaning it is associated with and is consistent within its context.
For example, if ``she'' is associated with the sentence ``a person who has a house'', then by Assumption~\ref{assumption:consistency}, the sentence ``she has no property'' will have 0 probability mass.
Moreover, when we want to refer to ``the person who has a house'' instead of repeating the sentence again, we use ``she'' as an abbreviation.

For simplicity, we only consider single-word expressions and assume that such association is consistent throughout a document. 
\begin{assumption}[Expression-meaning association]
\label{assumption:pronoun}
    There is a set of words $\Gamma$ such that for every document in the training data, some $r \in \Gamma$ in the document is associated with a meaning represented as a set of sentences $Z_r$ with a prior distribution $\Pr(Z_r)$.
    Any $z \in Z_r$ present in the document can be replaced with $r$ without breaking the logical consistency of the document.
\end{assumption}


Adjectives such as ``good'' and ``bad'' are also expressions that can be associated with variable meanings, and their meanings also depend on the context.
However, the association may not be as variable as pronouns'. 
We reflect this with a prior distribution for the meaning with which an expression is associated later in our theoretical analysis.

Finally, we assume a document is a set of paragraphs where some expressions in $\Gamma$ are present:   
\begin{assumption}[Document]
\label{assumption:article}
    A document is a concatenation of paragraphs containing $r \in \Gamma$ separated with a delimiter $d$ (e.g., a blank line).
\end{assumption}

Based on the above assumptions, we can define a \textit{perfect LM}.
\begin{definition}[Perfect LM]
\label{def:perfect}
    We define an LM to be perfect if it only generates outputs satisfying the above assumptions. 
\end{definition}

In the following, we will show that a perfect LM can follow instruction and do ICL.

\subsection{A Formalism for NLP Tasks}
\label{sec:dist-task}

With Assumption~\ref{assumption:sentence}, we propose a simple formalism: For any objective or prescriptive NLP task \citep{rottger-etal-2022-two} that maps an input $x$ to a set of acceptable outputs $Y$, that task can be described with some task instructions $u$ such that $u \wedge x \wedge y$ does not cause a contradiction if and only if $y \in Y$.

For example, if the task is a generation task, such as solving a math word problem, the instruction $u$ may specify the format (e.g., ``Think step-by-step and output the answer after \#\#\#\#'') and the input $x$ is a math word problem to be solve.
LLMs' response is an acceptable answer if and only if the response does not cause a contradiction to $u$ and $x$, i.e., it follows the instruction and the reasoning process aligns with the problem $x$.
When it is a classification task, the instruction $u$ should contain descriptions about each label $y$.
For example, we can formulate the sentiment analysis task over movie reviews as $u = \langle v_+, v_- \rangle = \langle \text{``I like the movie''}, \text{``I dislike the movie''} \rangle$. 


Under this formalism, it is trivial that perfect LLMs can follow instructions and solve tasks. 
It is because when the prompt for an LLM is the concatenation of the instruction and the task input, $u; x$, Assumption~\ref{assumption:consistency} (the consistency assumption) ensures that a perfect LLM only generates $y$ such that $y$ is logically consistent to its prompt $u; x$. 
Based on our NLP task formulation, if $y$ is logically consistent to $u$ and $x$, then $y$ is in $Y$, meaning that $y$ is an acceptable answer.

Two intricate questions remain, how can LLMs perform ICL? and why is it possible for an LM to generalize to unseen instructions?
We discuss these two questions in \S\ref{sec:icl-bound} and \S\ref{sec:extension}.



\section{Bounding ICL Loss}
\label{sec:icl-bound}

To see how we can analyze ICL with our framework, we first setup a latent variable model with our framework using
Assumption~\ref{assumption:pronoun} (the expression-meaning association assumption).
In this latent variable model, the latent variables are the association between the expressions in the document and their meanings.
Assumption~\ref{assumption:consistency} (the consistency assumption) defines the support of the distribution of the latent variable conditioning on a context: Associations that break the logical consistency of the context have zero probability. 
Assumption~\ref{assumption:consistency} also implies that, give a prefix, continuations that imply invalid associations have zero probability mass.

We analyzing ICL with this latent variable model by focusing on the latent variable for the association between the verbalizers and the meanings.
Given a classification task, the underlying latent variable value $z^*$ that produces the demonstration is the association where verbalizers are associated with the class descriptions of the task (following the definition of tasks in \S\ref{sec:dist-task}).
For example, when ``foo'' is used to represent the positive class of an sentiment classification task, ``foo'' should be associated with the description of the positive class, i.e., ``I like the movie''.
If an LLM can infer this underlying association $z^*$ based on the examples in the demostration and output a continuation satisfying Assumption~\ref{assumption:consistency}, then it can predict the correct verbalizer for a test example.

Formally, adapting and combining the analyses by \citet{zhang2023and} and \citet{hahn2023theory}, we have the following theorem.
\begin{theorem}[Average ICL Likelihood]
Denote a sequence of input-output pairs as $S_t = x_1, r_1, d, x_2, r_2, \cdots, x_t, r_t, d$, where $r_i$ is the correct verbalizer with which the label of $x_i$ is associated for $i = 1, 2, \cdots, t$ and $d$ is the delimiter that separates the examples.
Let the description of a task that maps inputs to classes $\mathcal{Y}$ be $\{ v_y \}_{y \in \mathcal{Y}}$ and $z^*$ represents that the task descriptions $\{ v_y \}_{y \in \mathcal{Y}}$ are associated with the corresponding verbalizers $\{ r_y \}_{y \in \mathcal{Y}} \subset \Gamma$ used for ICL.
We have for any $T \in \mathbf{Z}^+$, the average log likelihood that the LM predicts the correct verbalizer $r_t$ for $t = 1, 2, \cdots, T$ is bounded as:
\begin{equation}
\begin{split}
    &\frac{1}{T}\sum_{t=0}^T \log \Pr(r_{t} | x_{t}, S_{t - 1}) \\
&\geq \begin{aligned}[t]
    & \frac{1}{T} \log \Pr(z^*) \\
    & + \frac{1}{T} \sum_{t=1}^{T} \log \Pr(r_{t}, d | x_{t}, z^*, S_{t - 1}) \\
    & + \frac{1}{T} \sum_{t=1}^{T} \log \frac{\Pr(x_{t} | z^*, S_{t - 1})}{\Pr(x_{t} | S_{t - 1})}
\end{aligned} 
\end{split}
\label{eq:main}
\end{equation}
\label{thm:main}
\end{theorem}
When the last two terms on the right-hand side are nonnegative, Eq.~\ref{eq:main} shows that the average log likelihood converges to $0$ in $\mathcal{O}(1/T)$. We discuss the terms on the right-hand side below.

The second term is $0$.
This is because Assumption~\ref{assumption:consistency} ensures that the probability $\Pr(r|x_t)$ is zero if $r$ is not an acceptable output for $x_t$ (i.e., when $r \ne r_t$ for classification tasks).

We then look at the last term. 
This term is $0$ when $x_t$ is conditionally independent of $z^*$ as assumed by \citet{zhang2023and}. 
However, this may be an oversimplification because, in natural language, the transition from $x_t$ to its next token depends on the content of $x_t$.
Fortunately, this assumption may actually be unnecessary for convergence because $x_t$ is an example from a downstream task related to $z^*$; it is likely that
\begin{equation*}
    \Pr(x_{t} | z^*, S_{t - 1}) \geq \Pr(x_{t} | S_{t - 1}),
\end{equation*}
which implies that this term is non-negative, and we can thus ignore this term. More rigorously, we can show the following corollary.
\begin{corollary}[Expected Average ICL Loss]
Assuming $z^*$ is the association suggesting that the sub-sequences separated by the delimiter $d$ are independent of each other, namely, $\Pr(x_1, r_1, d, x_2, r_2, d, \cdots, x_t, r_t, d | z^*) = \prod_{t=1}^T \Pr(x_t, r_t, d | z^*)$.
If the downstream task data distribution $\mathcal{D}_{X}$ follow $\Pr(x | z^*)$, then we can bound the average ICL likelihood over the distribution as:
\begin{equation}
\begin{split}
    &\mathop{\mathbb{E}}_{x_1, x_2, \cdots, x_T \sim \mathcal{D}_{X}^T } \left[ \frac{1}{T} \sum_{t=0}^T \log \Pr(r_t | x_{t}, S_{t-1})  \right] \\
    &\geq \frac{1}{T} \log \Pr(z^*).
    \label{eq:simplified}
\end{split}
\end{equation}
\label{cor:dst-exp}
\end{corollary}

The right-hand side of Eq.~\ref{eq:simplified} characterizes the convergence rate and reflects the difficulty of doing ICL.
It reflects that when the association between label description and verbalizer is uncommon in the training data (e.g., associating ``positive'' with ``This movie is bad.''), doing ICL is more difficult. 

Note that we can extend the results to generation tasks.
For generation tasks, we usually use a separator (or a short text span) between $x_t$ and $r_t$.
We can see the separator as an expression that can be associated with different meanings, so the latent space for $z$ is the meaning the separator can be associated with and $z^*$ is the association between the separator and the task instruction (details in \S\ref{sec:bound-for-gen}).

\section{Generalization}
\label{sec:extension}

In \S\ref{sec:framework}, we assume a latent model, which, however, poses a dilemma.
Language can encode various meanings. 
If the sequence length is unconstrained, the corresponding semantic space can even be infinite.
However, if the latent space is infinite, the limited training data would not cover the entire latent space.
Without any assumption on the latent space (e.g., the relation between the states in the space), it is impossible to discuss the generalization to unseen latent states.
Thus, we provide an extension of our theoretical framework.

\begin{assumption}[Meaning representation]
    There exists (1) a finite set of \textit{atom concepts} $\Sigma$, (2) a knowledge base $\mathrm{KB}$ consisting of logical rules between the atom concepts in $\Sigma$, and (3) a function $f$ that can map any sentence in language to its meaning represented as a logical formula with operands in $\Sigma$ such that for any two sentences $s_1, s_2$, the logical relation  between $s_1$ and $s_2$ judged by humans is the same as $f(s_1)$ and $f(s_2)$ given the rules in the knowledge base $\mathrm{KB}$.
\end{assumption}

The three elements of this assumption correspond to theories in various fields.
The notion of atom concepts aligns with cognitive psychology studies that hypothesize the existence of a set of mental tokens \cite{Fodor1975-FODTLO,LOT2}
and a recent study \cite{computational-origin-of-repr} suggesting semantics can be encoded with the combination of only a few symbols. 
The notion of knowledge base follows the early formulation of AI~\citep{SISKIND199639,murphy2004big}.
As for the existence of a parsing function $f$, it follows the long history of linguistics studying the relationships between natural and formal languages \citep{carnap1968logische,bresnan1982mental,steedman1987combinatory,steedman1996surface,sag1999syntactic,frege1879begriffsschrift,peirce1883theory}.

This assumption suggests that 
if we have the parsing function $f$, solving NLP tasks only requires a finite-length program that can do logical reasoning by manipulating logical symbols according to logical induction rules. 
If a model can learn this program, then it can perform a task even if this task is not in the training data.
This assumption of a finite $\Sigma$ also instantiates the concept of ``language skills'' by \citet{arora2023theory}, and their theoretical results are thus applicable.

\begin{figure*}[h]
    \centering

\includegraphics[width=0.95\textwidth]{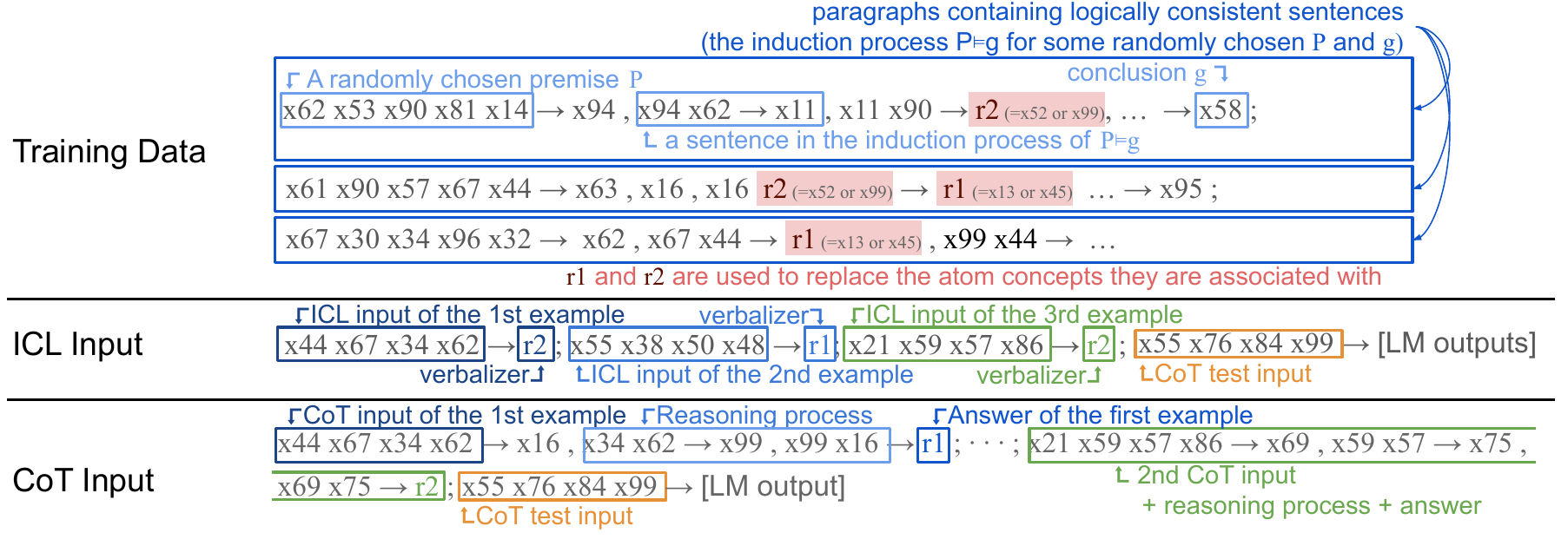}
    \caption{Calcutec examples for training, in-context learning (ICL), and chain-of-thought (CoT).}
    \label{tab:calcutec-examples}
\end{figure*}

\section{Relating to Description Length}
\label{sec:desc-length}

We can see $\Pr(r_t, d | x_t, z^*, S_{t-1})$ as the difficulty of the example by having an additional assumption:
\begin{assumption}
    In some documents in the training data, the paragraphs are constituted with steps in a logical induction process, with some steps randomly dropped.
    \label{assumption:step-by-step}
\end{assumption}
This kind of document may be prevalent in training data.
Essays arguing some claims are one example.
To be convincing, these essays should proceed like a proving process that induces their conclusions.
Documents describing a series of events can be another example, as events follow commonsense and develop progressively.

With this assumption and some regularity assumptions on the data distribution, we can have
\begin{equation}
    \Pr(r_t, d | x_t, z^*, S_{t-1}) \leq c \cdot \ell(x_t),
    \label{eq:desc-length}
\end{equation}
where $\ell(x_t)$ is the number of reasoning steps required to solve the task, and $c$ is a constant.
This $\ell(x_t)$ corresponds to the description length of the function that maps the inputs to their label in the loss bound by \citet{hahn2023theory} (more discussion in Appendix~\ref{sec:discuss-desc-length}).

\section{Inspecting Generalization Empirically}

Although we have shown that perfect LMs can follow instructions (\S\ref{sec:dist-task}) and perform ICL (\S\ref{sec:icl-bound}) when training data satisfy the assumptions in \S\ref{sec:framework}, it remains unverified whether imperfect real-world LMs exhibit the same capabilities.
For LMs to demonstrate these abilities, they must generalize well from the training set to the prompts used during inference.
However, analyzing how deep models generalize inherently requires additional assumptions about the model, which is beyond the scope of this paper.
Thus, to address this gap, we leverage our extended framework (\S\ref{sec:extension}) to characterize distribution shifts and design experiments that empirically examine the generalization of LM.


\subsection{Inspecting the ICL Capability}
\label{sec:calcutec}

Firstly, we present a synthetic setup, Calcutec, as a concrete instantiation of our theoretical framework.
With Calcutec, we show that Transformers can acquire ICL capability by modeling the linguistic characteristics specified in our framework.


\subsubsection{Calcutec}

\paragraph{Setup}
Following the extended framework in \S\ref{sec:extension}, we construct a pseudo-language:
\begin{itemize}
    \item Logic model: 
    We use a subset of propositional logic as our logic model.
We only consider Horn clauses~\citep{horn}, i.e., formulas in the form $A \wedge B \to C$.
    \item Atom concepts: We use 100 symbols as our set of atom concepts $\Sigma$.
    \item $\mathrm{KB}$: We generate a knowledge base by generating 5 formulas of the form $\sigma_1 \wedge \sigma_2 \to \sigma$ for each $\sigma \in \Sigma$, where $\sigma_1, \sigma_2$ are sampled from $\Sigma \backslash \{ \sigma \}$ uniformly at random.
    \item We have a set $\Gamma = \{r_i\}_{i=1}^4$ representing the expressions described in Assumption~\ref{assumption:pronoun}.
\end{itemize}

\paragraph{Training Dataset.} 
Following Assumption~\ref{assumption:article}, a document is a concatenation of paragraphs separated by delimiter ``;'' and ends with ``.''.
In our synthetic language model training dataset, each document contains 16 paragraphs.

Following Assumption~\ref{assumption:consistency}, each paragraph consists of sentences logically coherent to each other.
Because sentences in the real world are not ordered arbitrarily, we follow Assumption~\ref{assumption:step-by-step} and generate random paragraphs following the structure of logical proofs.
Each paragraph represents an induction process of $P \models g$ for some randomly selected $P \subset \Sigma$ and $g \in \Sigma$.
Each sentence in the paragraph is a sentence representing a reasoning step for $P \models g$.
We separate the sentences in the sequence by commas. 
To simulate the noise in the real world, we further apply perturbations that skip some steps with a skip rate $p_{skip}$.
\footnote{Models can acquire in-context learning ability even with $p_{skip} = 0$ (Figure~\ref{fig:icl-no-drop} in the appendix).}

Following Assumption~\ref{assumption:pronoun}, after we generate a document, we randomly associate some symbols $A, B \subset \Sigma$ with $r_a, r_b \in \Gamma$ respectively.
We reflect this association in the generated document by replacing symbols in $A$ and $B$ with expression $r_a, r_b \in \Gamma$ respectively (details in Appendix~\ref{sec:calcutec-train-gen}).


\paragraph{Downstream Tasks.}
Following the formalism in \S\ref{sec:dist-task}, we define a binary classification task by defining the descriptions $v_+$ and $v_-$ of the positive and negative classes, respectively.
We use the disjunctions of atom concepts (i.e. in the form of $a_1 \vee a_2 \vee \cdots$) as descriptions of classes.
We create five downstream tasks using different disjunctions.
Each input is a subset of variables in $\Sigma$ from which we ensure only one of the classes can be induced.

\paragraph{Demonstration.}
We represent an input-label pair as  $x_1 x_2 \cdots \to r$, where $x_1 x_2 \cdots$ is the input part and $r \in \{r_+, r_- \} \subset \Gamma$ is an expression in $\Gamma$ serving as the verbalizer.

\paragraph{Chain-of-thought.}
A chain-of-thought is in the same format as the training data but ends with an expression $r \in \{r_+, r_-\}$, e.g., $x_1 x_2 \cdots \to x_3 ; x_3 \cdots x_4 \to r_+$.
This chain-of-thought reflects the step-by-step induction process from the inputs to the label. We show an example in Figure~\ref{tab:calcutec-examples}.


\subsubsection{Distribution Shifts}
\label{sec:inspecting-dist-shifts}

We make experimental designs to simulate the real-world distribution shifts from training to inference:

\paragraph{Format Mismatch.} The reasoning steps are present in the training data but not in the prompts.

\paragraph{Verbalizer Mismatch.}
When we choose the expressions in $\Gamma$, we assign the probability mass 45\%, 45\%, 5\%, 5\% to $r_1, r_2, r_3, r_4$.
In this way, we can inspect the effect of using less frequent verbalizers.

\paragraph{Unseen Tasks.}
To investigate whether the model can generalize to a new combination of formulas unseen in the training data when we generate our training data, we ensure that the expressions in $\Gamma$ are only associated with the disjunctions of two atom concepts $s_1, s_2$ from a strict subset of all possible combinations $\Sigma \times \Sigma$.
We then test the trained model on tasks where $v_+$ and $v_-$ are the disjunctions of the unseen combinations.
We also test the models on tasks where $v_+$ and $v_-$ are the disjunctions of three atom concepts $\in \Sigma \times \Sigma \times \Sigma$.


\subsubsection{Experiment Details}

We use $p_{skip} = 0.25$ in our experiment, generating 60,000 documents with 16 paragraphs (as described above).
Among them, we use 10k for validation.
We train a 6-layer Transformer~\citep{transformer} model until the loss on the validation set converges using the standard autoregressive loss.
We include additional setups in \S\ref{sec:extra-setups}.

\subsubsection{Results and Discussion}

\begin{table}
    \centering
    \begin{tabular}{c | c  c | c c }
    \toprule
            & \multicolumn{2}{c|}{$r_1, r_2$} & \multicolumn{2}{c}{$r_3, r_4$} \\
    Task    & ICL  & CoT  & ICL  & CoT  \\
    \midrule
    Single & 57.1 & 91.7 & 55.6 & 92.0 \\
    Double & 53.5 & 76.3 & 51.1 & 77.1 \\
    Triple & 53.0 & 73.0 & 51.7 & 73.4 \\ 
    \bottomrule
    \end{tabular}
    \caption{The 4-shot accuracy of ICL versus chain-of-thought (CoT) using different verbalizers.}
    \label{tab:cot-y12}
\end{table}

Figure~\ref{fig:icl} shows that LMs trained with Calcutec can perform in-context learning.
This evidence supports our Pelican Soup framework and aligns with our theoretical analysis in \S\ref{sec:icl-bound}:
The ICL accuracy follows a trend $\exp(c /T)$ for some $c < 0$, as our theoretical result in Eq.~\ref{eq:main}.

We further inspect the ICL performance under the distribution shifts described in \S{\ref{sec:inspecting-dist-shifts}}.
For the infrequent verbalizer, we observe similar performance regardless of the frequency of the verbalizers ($r_1, r_2$ versus $r_3, r_4$).
For the unseen tasks, Figure~\ref{fig:icl} shows that the model has similar performance in tasks defined with seen and unseen combinations of atom concepts (dotted and solid lines), even generalizing to tasks defined with three latent concepts (green lines).
In sum, the results show that the model can generalize well under several distributional shifts.


\begin{figure}
    \centering
    \begin{subfigure}{0.235\textwidth}
        \centering
        \includegraphics[width=\textwidth]{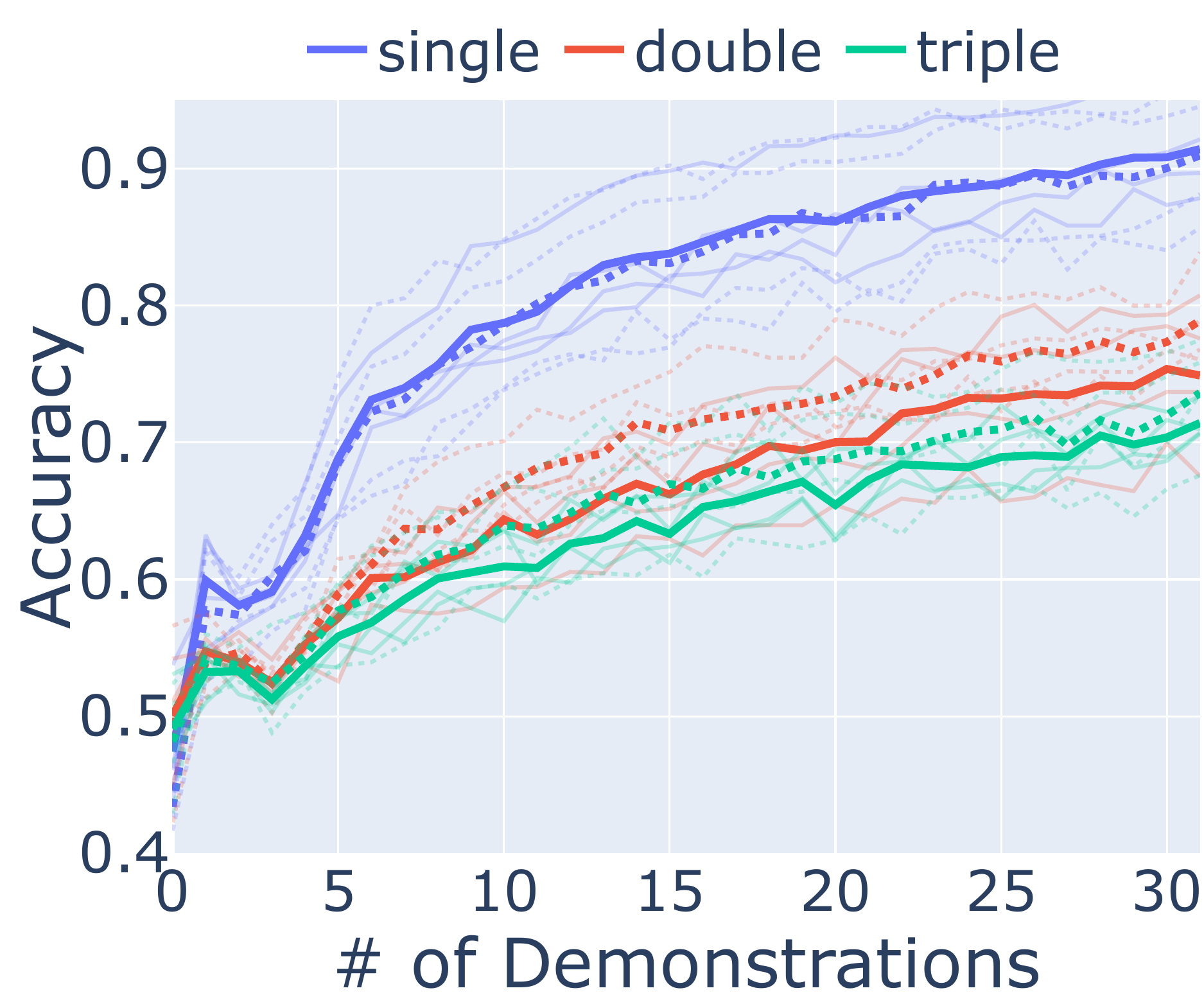}
        \caption{Verbalizers = $r_1, r_2$}
    \end{subfigure}
    \begin{subfigure}{0.235\textwidth}
        \centering
        \includegraphics[width=\textwidth]{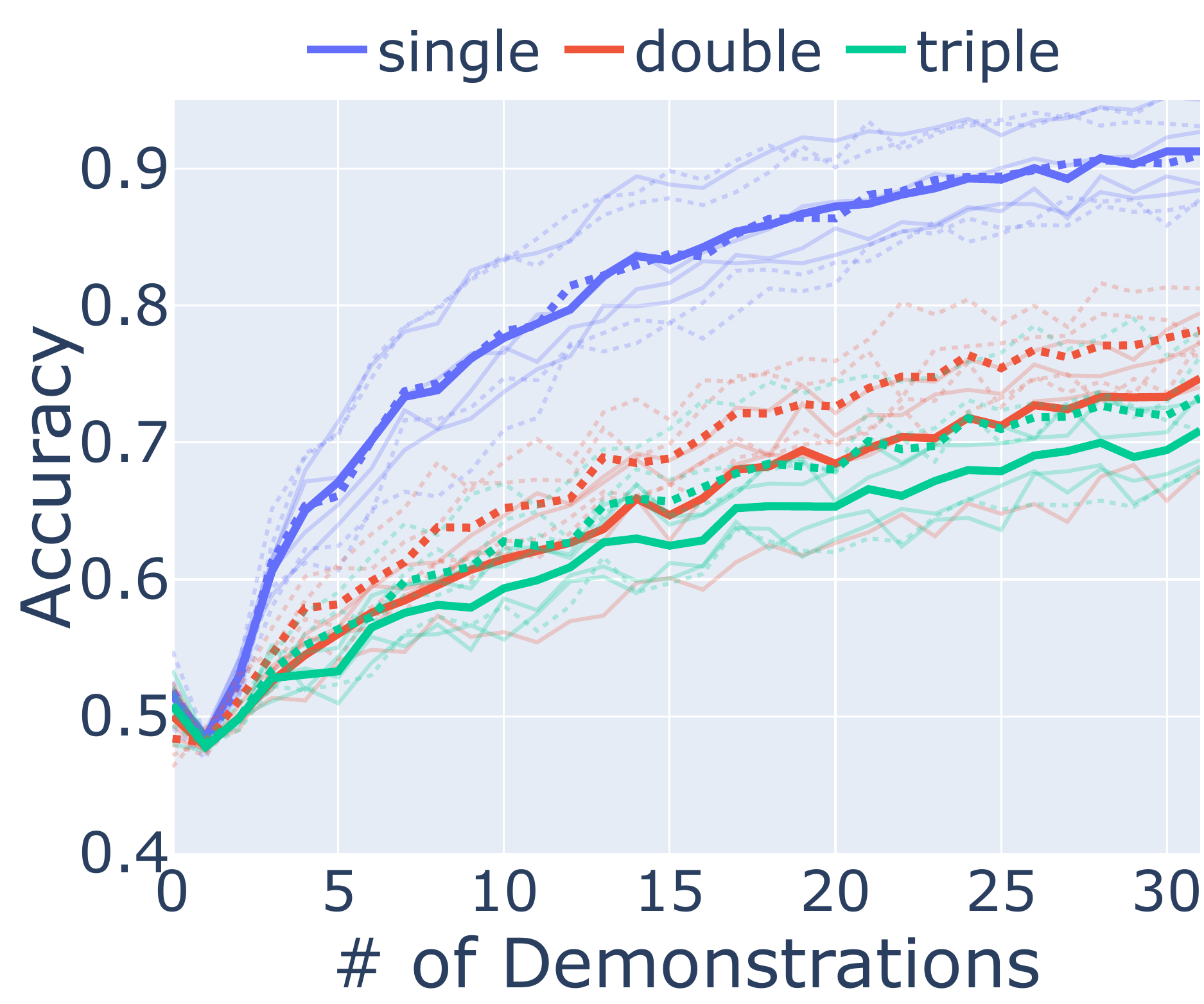}
        \caption{Verbalizers = $r_3, r_4$}
    \end{subfigure}
    
    \caption{
    In-context learning accuracy with Calcutec when using different verbalizers ($r_1, r_2$ or $r_3, r_4$).
    The dotted lines represent the performance of \textit{unseen combinations} described in \S{\ref{sec:inspecting-dist-shifts}}. The colors represent the number of atom concepts each class ($v_+$ or $v_-$) is associated with.
    The main lines represent the average accuracy of 5 tasks.
    Lines in the lighter color represent the individual tasks. 
    }
    \label{fig:icl}
\end{figure}

We also experiment with 4-shot learning using chain-of-thought.
Table~\ref{tab:cot-y12} shows that the model also benefits from chain-of-thought.
We conjecture that it is because chain-of-thought has a format more similar to the format for training.

\subsection{Inspect Instruction-following Capability}

In addition to ICL, we inspect another intriguing capability of LLMs, the capability of following unseen instructions.
Our extended framework in \S\ref{sec:extension} suggests that we can see following unseen instructions as doing a sentence completion task that involves atom concepts (or say ``skills'' in \citet{arora2023theory}) whose composition is different from the ones involved when completing sentences in training data.
We also suggest that LMs acquire the capability of following these unseen instructions by modeling the interrelations between the atom concepts involved.
To show the plausibility of this, we experiment with an arithmetic task.

\subsubsection{Arithmetic Task}
\label{sec:sum}

\begin{table}[]
    \centering
    \begin{tabular}{l|l}
    \toprule
    Single-step & \small{0 + 1 * 2 + 3 * 4 = 10 .} \\
    Multi-step & \small{0 + 1 * 2 = 2 . 3 * 4 = 12 . 2 + 12 = 14 .} \\
    Test - seen & \small{2 * 3 + 4 * 5 + 6 = } \\
    Test - unseen & \small{2 + 10 + 4 * 11 * 6 =} \\
    \bottomrule
    \end{tabular}
    \caption{Training (first two rows) and test (last two rows) examples used in the arithmetic task \S\ref{sec:sum}.}
    \label{tab:sum-example}
\end{table}

We utilize the algebraic structure of the integers under modulo addition/multiplication to construct a language where each expression is constantly associated with a meaning and the atom concepts include 0 to 15, the symbol ``+'', ``*'', ``='', ``.''.
Each sentence in this language is in the form of an equation with additions and multiplications.

We generate LM training data in a multi-step and a single-step setup.
In the multi-step setup, each example contains the steps required to compute the result, each of which involves two to five numbers (details in \S\ref{sec:summation-details}).
We expect that LMs can learn the interrelation between atom concepts by modeling these steps. 
In the single-step setup, the final result is computed in a single step.

To simulate the scenario where completing the prompts in the test set involves a different distribution of atom concepts from that of the training examples,
we ensure that the five numbers for each training example are either all from $\{0, \cdots,  9\}$ or all from $\{6, \cdots 15\}$.
We then evaluate the model with a \textit{seen} and an \textit{unseen} setups. 
In the \textit{seen} setup, we prompt the model with 5 numbers sampled in the same way as the training set. 
In the \textit{unseen} setup, we prompt the model to complete equations with 5 numbers where the first, third, and fifth numbers are from $\{0, \cdots, 5\}$ and the other two numbers are from $\{10, \cdots, 15\}$ (examples in Table~\ref{tab:sum-example}).

\subsubsection{Results and Discussion}

\begin{table}[]
\centering
\begin{tabular}{lcc}
\toprule
 \footnotesize{test $\backslash$ train}       & Multi-step & Single-step \\
\midrule
Seen   &   97.5 {\small(0.3)}  &  99.3 {\small(0.6)} \\
Unseen &   75.3 {\small(2.0)}  &  49.7 {\small(1.0)} \\
\bottomrule
\end{tabular}
\caption{The average accuracy of LMs trained for the arithmetic tasks with 5 random seeds.}
\label{tab:sum-results}
\end{table}

Table~\ref{tab:sum-results} shows that all LMs achieve a near-perfect accuracy in the \textit{seen} setup but the LMs trained with the multi-step training set achieve significantly better accuracy in the \textit{unseen} setup.  
This indicates that modeling the multiple steps in the training data may allow the model to learn the interrelation between the atom concepts and generalize to prompts that involve unseen compositions of atom concepts to some extent, as suggested by our extended framework in \S\ref{sec:extension}.
This is also aligned with the success of symbolic chain-of-thought distillation~\citep{li-etal-2023-symbolic,hsieh-etal-2023-distilling,shridhar-etal-2023-distilling}.

\begin{table}
    \centering
    \begin{tabular}{l | c  c  c  c}
        \toprule
        task & SST-2 & CR & MR & Subj \\
        \midrule
        direct  & 63.0 & 61.7 & 59.2 & 51.0 \\
        direct w/ foo/bar & 51.7 & 53.7 & 52.2 & 57.1 \\
        pronoun & 65.3 & 62.9 & 56.7 & 62.2 \\
        \bottomrule
    \end{tabular}
    \caption{The accuracy of using task-specific templates/verbalizers (direct)~\citep{min-etal-2022-noisy} v.s. using task-agnostic templates/pronouns for 16-shot in-context learning with GPT2-Large. }
    \label{tab:pronoun-verbalizer}
\end{table}

\section{Real-world Evidence}
\label{sec:pronoun-icl}

We inspect how likely that LMs learn ICL by modeling the expression-meaning association as our framework suggests.
We specifically inspect a small model GPT2-Large and the setting in which we use pronouns as verbalizers.
Because pronouns are reference words frequently associated with different meanings, we expect that even a small model can do ICL well with pronouns.
To do so, we experiment with the template \textit{``\texttt{[input]}'', \texttt{[verbalizer]} thought.} and use ``he'', ``she'' as verbalizers.
We follow the setup in \citet{min-etal-2022-noisy} and compare the accuracy of binary classification tasks, including SST-2~\cite{socher-etal-2013-recursive}, CR~\cite{cr}, MR~\cite{pang-lee-2005-seeing}, and Subj~\cite{subj}, using GPT2-Large.

Table~\ref{tab:pronoun-verbalizer} shows that this task-agnostic template with pronouns is competitive with those task-specific templates.
This aligns with our conjecture that the high frequency of pronouns in the training set allows smaller LMs to acquire the ICL capability.
It also shows that, unlike what \citet{wei2023larger} claims, not only larger models can do in-context learning with task-irrelevant verbalizers.
On the other hand,  using task-irrelevant verbalizers ``foo'' and ``bar'' has lower performance (which aligns with the observation of \citet{wei2023larger}). 
This may be because learning to perform ICL with those low-frequency words requires more training data, and thus a larger model size.

\section{Related Work}
Since \citet{icl} discovered large language models' in-context learning ability, some theoretical works have attempted to explain how language models acquire this ability.
Based on a hidden Markov model (HMM) assumption on the language generation process, \citet{xie2022an} suggested that in-context learning is an implicit Bayesian inference process.
\citet{hahn2023theory} defined the generation process with Compositional Attribute Grammar, which is weaker than the HMM assumption, explaining the in-context learning ability with the minimum description length.
They also studied the compositionality of natural language tasks with function compositions.
\citet{zhang2023and} assumed a more general latent variable model.
\citet{arora2023theory} analyze the emergence of skills based on the scaling law~\citep{hoffmann2022an}.
While their analysis assumes a set of atomic skills for NLP tasks, our framework is based on a set of atom concepts. 

There were also many empirical studies on the in-context learning ability.
Some works focused on the effect of the instruction \citep{webson-pavlick-2022-prompt,lampinen-etal-2022-language,pmlr-v203-jang23a}, while some focused on the examples in the demonstration \citep{liu-etal-2022-makes,lu-etal-2022-fantastically,sorensen-etal-2022-information,min-etal-2022-rethinking,yoo-etal-2022-ground,ye2022complementary,chang2022careful,ye2022complementary,wang2023large,kossen2023context}.
\citet{shin-etal-2022-effect} found that not all training corpora led to in-context learning ability.
\citet{prystawski2023think} used synthetic data to suggest that the pretraining dataset's locality structure contributes to the reasoning steps' effectiveness.
\citet{wang2023towards} studied the reasoning steps in chain-of-thought.
\citet{akyurek2024context} formulated ICL as learning a formal language from demonstrations and benchmarked model families.

Some previous work studied in-context learning as a meta-learning-like problem \citep{chen-etal-2022-meta}.
Some works focused on the relationships between in-context learning and optimization algorithms \citep{garg2022what,von2022transformers,rek2023what,fu2023transformers,guo2023transformers}.
Some works inspected the mechanism of ICL in transformer models \citep{hendel-etal-2023-context,bietti2023birth,todd2023function,shen2023pretrained,bai2023transformers}.
\citet{chan_data_2022} studied the properties of dataset distribution that could contribute to the in-context learning ability.
\citet{litransformers} provided generalization bounds based on the stability of Transformer models and the distance of downstream tasks.
We instead focus on how the pretraining data in natural language contributes to the ICL learning ability. 

\section{Conclusion and Future Work}

In this work, we propose a framework that explains how linguistic phenomena in the training corpus lead to LLMs' ICL and instruction-following capability. 
Compared with previous works \cite{xie2022an,zhang2023and}, our latent model better reflects the complexity of language.
By introducing the notion of knowledge base and logic system, our framework provides insights into how LLMs can generalize from pretraining to downstream tasks, instantiating a setup compatible with the assumptions made by \citet{arora2023theory}.
We also relate our bound to the function description length discussed by \citet{hahn2023theory}.

Our framework illuminates a few possible directions for improving LLMs:

\begin{enumerate}
    \item Our work highlights the importance of learning the interrelation between meanings.
As previous works have shown that the language modeling objective is inefficient for this purpose \citep{allen2023physics,chiang-etal-2024-retrieval}, we suggest that developing a more sophisticated learning algorithm is crucial.
\item Our theory illustrates how linguistic studies on parsing functions can inform LLM research, offering a new theoretical foundation for analyzing LLM generalization.
\item The experimental results of our arithmetic task show that Transformer models can generalize to unseen prompts by modeling the intermediate step-by-step reasoning process.
This may be related to the success of the symbolic chain-of-thought distillation~\citep{li-etal-2023-symbolic,hsieh-etal-2023-distilling,shridhar-etal-2023-distilling}.
Investigating and strengthening the mechanism can improve the efficiency of LM training. 
\end{enumerate}

\section{Limitations}

A limitation of our framework is that, as most theoretical studies do, we simplify the real-world scenario to draw insights.
One simplification we make is that, we do not take the noise in LLMs' training data into account.
As we may need to make more assumption on the noise to establish a generic theoretical result, we leave it for future study.
Another simplification is that, we assume that the language model can perfectly model the distribution of natural language (i.e., only generates outputs satisfying the assumptions in \S{\ref{sec:dist-train}}, especially Assumption~{\ref{assumption:consistency}}).
However, it is unlikely to be the case in practice.
On the one hand, the training data may not cover all the test cases.
On the other hand, LLMs may not perfectly generalize from the training set.
We need to make more assumptions on the training/test data distribution and/or have a deeper understanding on how deep learning models generalize to alleviate this assumption.
Therefore, we deem this out of the scope of this paper.

We also note some limitations of our experiments.
As commonly seen in theoretical work, our experiment setup in \S\ref{sec:calcutec} is a simplification of natural language.
Our intention is to isolate and illustrate specific theoretical mechanisms without confounding variables introduced by the full complexity of natural language. 
The experiment in \S\ref{sec:pronoun-icl} shows the correlation between better ICL performance and the use of pronouns as verbalizer.
Although this serves as evidence aligning with our theoretical framework, more studies are required to establish the causal relationship.

\bibliography{custom}

@inproceedings{icl,
 author = {Brown, Tom and Mann, Benjamin and Ryder, Nick and Subbiah, Melanie and Kaplan, Jared D and Dhariwal, Prafulla and Neelakantan, Arvind and Shyam, Pranav and Sastry, Girish and Askell, Amanda and Agarwal, Sandhini and Herbert-Voss, Ariel and Krueger, Gretchen and Henighan, Tom and Child, Rewon and Ramesh, Aditya and Ziegler, Daniel and Wu, Jeffrey and Winter, Clemens and Hesse, Chris and Chen, Mark and Sigler, Eric and Litwin, Mateusz and Gray, Scott and Chess, Benjamin and Clark, Jack and Berner, Christopher and McCandlish, Sam and Radford, Alec and Sutskever, Ilya and Amodei, Dario},
 booktitle = {Advances in Neural Information Processing Systems},
 editor = {H. Larochelle and M. Ranzato and R. Hadsell and M.F. Balcan and H. Lin},
 pages = {1877--1901},
 publisher = {Curran Associates, Inc.},
 title = {Language Models are Few-Shot Learners},
 url = {https://proceedings.neurips.cc/paper_files/paper/2020/file/1457c0d6bfcb4967418bfb8ac142f64a-Paper.pdf},
 volume = {33},
 year = {2020}
}

@article{chowdhery2022palm,
  title={Palm: Scaling language modeling with pathways},
  author={Chowdhery, Aakanksha and Narang, Sharan and Devlin, Jacob and Bosma, Maarten and Mishra, Gaurav and Roberts, Adam and Barham, Paul and Chung, Hyung Won and Sutton, Charles and Gehrmann, Sebastian and others},
  journal={arXiv preprint arXiv:2204.02311},
  year={2022}
}

@inproceedings{
xie2022an,
title={An Explanation of In-context Learning as Implicit Bayesian Inference},
author={Sang Michael Xie and Aditi Raghunathan and Percy Liang and Tengyu Ma},
booktitle={International Conference on Learning Representations},
year={2022},
url={https://openreview.net/forum?id=RdJVFCHjUMI}
}

@inproceedings{McCarthy1960ProgramsWC,
  title={Programs with common sense},
  author={John McCarthy},
  year={1960}
}

@book{bresnan1982mental,
  title={The mental representation of grammatical relations},
  author={Bresnan, Joan and Bresnan, Joan Wanda},
  volume={1},
  year={1982},
  publisher={MIT press}
}

@book{carnap1968logische,
  title={Logische syntax der sprache},
  author={Carnap, Rudolf and others},
  year={1968},
  publisher={Springer}
}

@article{steedman1987combinatory,
  title={Combinatory grammars and parasitic gaps},
  author={Steedman, Mark},
  journal={Natural Language \& Linguistic Theory},
  volume={5},
  number={3},
  pages={403--439},
  year={1987},
  publisher={Springer}
}

@book{steedman1996surface,
  title={Surface structure and interpretation},
  author={Steedman, Mark},
  volume={30},
  year={1996},
  publisher={MIT press Cambridge, MA}
}

@book{sag1999syntactic,
  title={Syntactic theory: A formal introduction},
  author={Sag, Ivan A and Wasow, Thomas and Bender, Emily M and Sag, Ivan A},
  volume={92},
  year={1999},
  publisher={Center for the Study of Language and Information Stanford, CA}
}

@article{frege1879begriffsschrift,
  title={Begriffsschrift, a formula language, modeled upon that of arithmetic, for pure thought},
  author={Frege, Gottlob and others},
  journal={From Frege to G{\"o}del: A source book in mathematical logic},
  volume={1931},
  pages={1--82},
  year={1879}
}

@book{peirce1883theory,
  title={A theory of probable inference.},
  author={Peirce, Charles S},
  year={1883},
  publisher={Little, Brown and Co}
}

@article{wei2023larger,
  title={Larger language models do in-context learning differently},
  author={Wei, Jerry and Wei, Jason and Tay, Yi and Tran, Dustin and Webson, Albert and Lu, Yifeng and Chen, Xinyun and Liu, Hanxiao and Huang, Da and Zhou, Denny and others},
  journal={arXiv preprint arXiv:2303.03846},
  year={2023}
}

@inproceedings{min-etal-2022-noisy,
    title = "Noisy Channel Language Model Prompting for Few-Shot Text Classification",
    author = "Min, Sewon  and
      Lewis, Mike  and
      Hajishirzi, Hannaneh  and
      Zettlemoyer, Luke",
    booktitle = "Proceedings of the 60th Annual Meeting of the Association for Computational Linguistics (Volume 1: Long Papers)",
    month = may,
    year = "2022",
    address = "Dublin, Ireland",
    publisher = "Association for Computational Linguistics",
    url = "https://aclanthology.org/2022.acl-long.365",
    doi = "10.18653/v1/2022.acl-long.365",
    pages = "5316--5330",
}

@inproceedings{min-etal-2022-rethinking,
    title = "Rethinking the Role of Demonstrations: What Makes In-Context Learning Work?",
    author = "Min, Sewon  and
      Lyu, Xinxi  and
      Holtzman, Ari  and
      Artetxe, Mikel  and
      Lewis, Mike  and
      Hajishirzi, Hannaneh  and
      Zettlemoyer, Luke",
    booktitle = "Proceedings of the 2022 Conference on Empirical Methods in Natural Language Processing",
    month = dec,
    year = "2022",
    address = "Abu Dhabi, United Arab Emirates",
    publisher = "Association for Computational Linguistics",
    url = "https://aclanthology.org/2022.emnlp-main.759",
    pages = "11048--11064",
}

@article{horn,
 ISSN = {00224812},
 URL = {http://www.jstor.org/stable/2268661},
 author = {Alfred Horn},
 journal = {The Journal of Symbolic Logic},
 number = {1},
 pages = {14--21},
 publisher = {Association for Symbolic Logic},
 title = {On Sentences Which are True of Direct Unions of Algebras},
 urldate = {2023-05-08},
 volume = {16},
 year = {1951}
}

@article{hahn2023theory,
  title={A Theory of Emergent In-Context Learning as Implicit Structure Induction},
  author={Hahn, Michael and Goyal, Navin},
  journal={arXiv preprint arXiv:2303.07971},
  year={2023}
}

@InProceedings{litransformers,
  title = 	 {Transformers as Algorithms: Generalization and Stability in In-context Learning},
  author =       {Li, Yingcong and Ildiz, Muhammed Emrullah and Papailiopoulos, Dimitris and Oymak, Samet},
  booktitle = 	 {Proceedings of the 40th International Conference on Machine Learning},
  pages = 	 {19565--19594},
  year = 	 {2023},
  editor = 	 {Krause, Andreas and Brunskill, Emma and Cho, Kyunghyun and Engelhardt, Barbara and Sabato, Sivan and Scarlett, Jonathan},
  volume = 	 {202},
  series = 	 {Proceedings of Machine Learning Research},
  month = 	 {23--29 Jul},
  publisher =    {PMLR},
  url = 	 {https://proceedings.mlr.press/v202/li23l.html}
}

@misc{chan_data_2022,
  title = {Data Distributional Properties Drive Emergent In-Context Learning in Transformers},
  author = {Chan, Stephanie C. Y. and Santoro, Adam and Lampinen, Andrew K. and Wang, Jane X. and Singh, Aaditya and Richemond, Pierre H. and McClelland, Jay and Hill, Felix},
  journal = {Neural Information Processing Systems},
  year = {2022},
}

@article{prystawski2023think,
  title={Why think step-by-step? Reasoning emerges from the locality of experience},
  author={Prystawski, Ben and Goodman, Noah D},
  journal={arXiv preprint arXiv:2304.03843},
  year={2023}
}

@inproceedings{
garg2022what,
title={What Can Transformers Learn In-Context? A Case Study of Simple Function Classes},
author={Shivam Garg and Dimitris Tsipras and Percy Liang and Gregory Valiant},
booktitle={Advances in Neural Information Processing Systems},
editor={Alice H. Oh and Alekh Agarwal and Danielle Belgrave and Kyunghyun Cho},
year={2022},
url={https://openreview.net/forum?id=flNZJ2eOet}
}

@inproceedings{
rek2023what,
title={What learning algorithm is in-context learning? Investigations with linear models},
author={Ekin Aky{\"u}rek and Dale Schuurmans and Jacob Andreas and Tengyu Ma and Denny Zhou},
booktitle={The Eleventh International Conference on Learning Representations },
year={2023},
url={https://openreview.net/forum?id=0g0X4H8yN4I}
}

@inproceedings{webson-pavlick-2022-prompt,
    title = "Do Prompt-Based Models Really Understand the Meaning of Their Prompts?",
    author = "Webson, Albert  and
      Pavlick, Ellie",
    booktitle = "Proceedings of the 2022 Conference of the North American Chapter of the Association for Computational Linguistics: Human Language Technologies",
    month = jul,
    year = "2022",
    address = "Seattle, United States",
    publisher = "Association for Computational Linguistics",
    url = "https://aclanthology.org/2022.naacl-main.167",
    doi = "10.18653/v1/2022.naacl-main.167",
    pages = "2300--2344",
}

@InProceedings{pmlr-v203-jang23a,
  title = 	 {Can Large Language Models Truly Understand Prompts? A Case Study with Negated Prompts},
  author =       {Jang, Joel and Ye, Seonghyeon and Seo, Minjoon},
  booktitle = 	 {Proceedings of The 1st Transfer Learning for Natural Language Processing Workshop},
  pages = 	 {52--62},
  year = 	 {2023},
  editor = 	 {Albalak, Alon and Zhou, Chunting and Raffel, Colin and Ramachandran, Deepak and Ruder, Sebastian and Ma, Xuezhe},
  volume = 	 {203},
  series = 	 {Proceedings of Machine Learning Research},
  month = 	 {03 Dec},
  publisher =    {PMLR},
  url = 	 {https://proceedings.mlr.press/v203/jang23a.html}
}

@inproceedings{yoo-etal-2022-ground,
    title = "Ground-Truth Labels Matter: A Deeper Look into Input-Label Demonstrations",
    author = "Yoo, Kang Min  and
      Kim, Junyeob  and
      Kim, Hyuhng Joon  and
      Cho, Hyunsoo  and
      Jo, Hwiyeol  and
      Lee, Sang-Woo  and
      Lee, Sang-goo  and
      Kim, Taeuk",
    booktitle = "Proceedings of the 2022 Conference on Empirical Methods in Natural Language Processing",
    month = dec,
    year = "2022",
    address = "Abu Dhabi, United Arab Emirates",
    publisher = "Association for Computational Linguistics",
    url = "https://aclanthology.org/2022.emnlp-main.155",
    pages = "2422--2437",
}

@inproceedings{lu-etal-2022-fantastically,
    title = "Fantastically Ordered Prompts and Where to Find Them: Overcoming Few-Shot Prompt Order Sensitivity",
    author = "Lu, Yao  and
      Bartolo, Max  and
      Moore, Alastair  and
      Riedel, Sebastian  and
      Stenetorp, Pontus",
    booktitle = "Proceedings of the 60th Annual Meeting of the Association for Computational Linguistics (Volume 1: Long Papers)",
    month = may,
    year = "2022",
    address = "Dublin, Ireland",
    publisher = "Association for Computational Linguistics",
    url = "https://aclanthology.org/2022.acl-long.556",
    doi = "10.18653/v1/2022.acl-long.556",
    pages = "8086--8098",
}

@inproceedings{lampinen-etal-2022-language,
    title = "Can language models learn from explanations in context?",
    author = "Lampinen, Andrew  and
      Dasgupta, Ishita  and
      Chan, Stephanie  and
      Mathewson, Kory  and
      Tessler, Mh  and
      Creswell, Antonia  and
      McClelland, James  and
      Wang, Jane  and
      Hill, Felix",
    booktitle = "Findings of the Association for Computational Linguistics: EMNLP 2022",
    month = dec,
    year = "2022",
    address = "Abu Dhabi, United Arab Emirates",
    publisher = "Association for Computational Linguistics",
    url = "https://aclanthology.org/2022.findings-emnlp.38",
    pages = "537--563",
}

@inproceedings{liu-etal-2022-makes,
    title = "What Makes Good In-Context Examples for {GPT}-3?",
    author = "Liu, Jiachang  and
      Shen, Dinghan  and
      Zhang, Yizhe  and
      Dolan, Bill  and
      Carin, Lawrence  and
      Chen, Weizhu",
    booktitle = "Proceedings of Deep Learning Inside Out (DeeLIO 2022): The 3rd Workshop on Knowledge Extraction and Integration for Deep Learning Architectures",
    month = may,
    year = "2022",
    address = "Dublin, Ireland and Online",
    publisher = "Association for Computational Linguistics",
    url = "https://aclanthology.org/2022.deelio-1.10",
    doi = "10.18653/v1/2022.deelio-1.10",
    pages = "100--114",
}

@inproceedings{chang2022careful,
    title = "Data Curation Alone Can Stabilize In-context Learning",
    author = "Chang, Ting-Yun  and
      Jia, Robin",
    booktitle = "Proceedings of the 61st Annual Meeting of the Association for Computational Linguistics (Volume 1: Long Papers)",
    month = jul,
    year = "2023",
    address = "Toronto, Canada",
    publisher = "Association for Computational Linguistics",
    url = "https://aclanthology.org/2023.acl-long.452",
    doi = "10.18653/v1/2023.acl-long.452",
    pages = "8123--8144",
}

@inproceedings{sorensen-etal-2022-information,
    title = "An Information-theoretic Approach to Prompt Engineering Without Ground Truth Labels",
    author = "Sorensen, Taylor  and
      Robinson, Joshua  and
      Rytting, Christopher  and
      Shaw, Alexander  and
      Rogers, Kyle  and
      Delorey, Alexia  and
      Khalil, Mahmoud  and
      Fulda, Nancy  and
      Wingate, David",
    booktitle = "Proceedings of the 60th Annual Meeting of the Association for Computational Linguistics (Volume 1: Long Papers)",
    month = may,
    year = "2022",
    address = "Dublin, Ireland",
    publisher = "Association for Computational Linguistics",
    url = "https://aclanthology.org/2022.acl-long.60",
    doi = "10.18653/v1/2022.acl-long.60",
    pages = "819--862",
}

@inproceedings{ye2022complementary,
    title = "Complementary Explanations for Effective In-Context Learning",
    author = "Ye, Xi  and
      Iyer, Srinivasan  and
      Celikyilmaz, Asli  and
      Stoyanov, Veselin  and
      Durrett, Greg  and
      Pasunuru, Ramakanth",
    booktitle = "Findings of the Association for Computational Linguistics: ACL 2023",
    month = jul,
    year = "2023",
    address = "Toronto, Canada",
    publisher = "Association for Computational Linguistics",
    url = "https://aclanthology.org/2023.findings-acl.273",
    doi = "10.18653/v1/2023.findings-acl.273",
    pages = "4469--4484",
}

@inproceedings{shin-etal-2022-effect,
    title = "On the Effect of Pretraining Corpora on In-context Learning by a Large-scale Language Model",
    author = "Shin, Seongjin  and
      Lee, Sang-Woo  and
      Ahn, Hwijeen  and
      Kim, Sungdong  and
      Kim, HyoungSeok  and
      Kim, Boseop  and
      Cho, Kyunghyun  and
      Lee, Gichang  and
      Park, Woomyoung  and
      Ha, Jung-Woo  and
      Sung, Nako",
    booktitle = "Proceedings of the 2022 Conference of the North American Chapter of the Association for Computational Linguistics: Human Language Technologies",
    month = jul,
    year = "2022",
    address = "Seattle, United States",
    publisher = "Association for Computational Linguistics",
    url = "https://aclanthology.org/2022.naacl-main.380",
    doi = "10.18653/v1/2022.naacl-main.380",
    pages = "5168--5186",
}

@inproceedings{chen-etal-2022-meta,
    title = "Meta-learning via Language Model In-context Tuning",
    author = "Chen, Yanda  and
      Zhong, Ruiqi  and
      Zha, Sheng  and
      Karypis, George  and
      He, He",
    booktitle = "Proceedings of the 60th Annual Meeting of the Association for Computational Linguistics (Volume 1: Long Papers)",
    month = may,
    year = "2022",
    address = "Dublin, Ireland",
    publisher = "Association for Computational Linguistics",
    url = "https://aclanthology.org/2022.acl-long.53",
    doi = "10.18653/v1/2022.acl-long.53",
    pages = "719--730",
}

@inproceedings{
wang2023towards,
title={Towards Understanding Chain-of-Thought Prompting: An Empirical Study of What Matters},
author={Boshi Wang and Sewon Min and Xiang Deng and Jiaming Shen and You Wu and Luke Zettlemoyer and Huan Sun},
booktitle={ICLR 2023 Workshop on Mathematical and Empirical Understanding of Foundation Models},
year={2023},
url={https://openreview.net/forum?id=L9UMeoeU2i}
}

@article{von2022transformers,
  title={Transformers learn in-context by gradient descent},
  author={von Oswald, Johannes and Niklasson, Eyvind and Randazzo, Ettore and Sacramento, Jo{\~a}o and Mordvintsev, Alexander and Zhmoginov, Andrey and Vladymyrov, Max},
  journal={arXiv preprint arXiv:2212.07677},
  year={2022}
}

@inproceedings{transformer,
 author = {Vaswani, Ashish and Shazeer, Noam and Parmar, Niki and Uszkoreit, Jakob and Jones, Llion and Gomez, Aidan N and Kaiser, \L ukasz and Polosukhin, Illia},
 booktitle = {Advances in Neural Information Processing Systems},
 editor = {I. Guyon and U. Von Luxburg and S. Bengio and H. Wallach and R. Fergus and S. Vishwanathan and R. Garnett},
 pages = {},
 publisher = {Curran Associates, Inc.},
 title = {Attention is All you Need},
 url = {https://proceedings.neurips.cc/paper_files/paper/2017/file/3f5ee243547dee91fbd053c1c4a845aa-Paper.pdf},
 volume = {30},
 year = {2017}
}

@inproceedings{rottger-etal-2022-two,
    title = "Two Contrasting Data Annotation Paradigms for Subjective {NLP} Tasks",
    author = "Rottger, Paul  and
      Vidgen, Bertie  and
      Hovy, Dirk  and
      Pierrehumbert, Janet",
    booktitle = "Proceedings of the 2022 Conference of the North American Chapter of the Association for Computational Linguistics: Human Language Technologies",
    month = jul,
    year = "2022",
    address = "Seattle, United States",
    publisher = "Association for Computational Linguistics",
    url = "https://aclanthology.org/2022.naacl-main.13",
    doi = "10.18653/v1/2022.naacl-main.13",
    pages = "175--190",
}

@article{schank1988scripts,
  title={Scripts, plans, goals, and understanding: An inquiry into human knowledge structures},
  author={Schank, Roger C and Abelson, Robert P},
  year={1988},
  publisher={Morgan Kaufmann}
}

@article{SISKIND199639,
title = {A computational study of cross-situational techniques for learning word-to-meaning mappings},
journal = {Cognition},
volume = {61},
number = {1},
pages = {39-91},
year = {1996},
note = {Compositional Language Acquisition},
issn = {0010-0277},
doi = {https://doi.org/10.1016/S0010-0277(96)00728-7},
url = {https://www.sciencedirect.com/science/article/pii/S0010027796007287},
author = {Jeffrey Mark Siskind}
}

@book{murphy2004big,
  title={The big book of concepts},
  author={Murphy, Gregory},
  year={2004},
  publisher={MIT press}
}

@book{Fodor1975-FODTLO,
	author = {Jerry A. Fodor},
	editor = {},
	publisher = {Harvard University Press},
	title = {The Language of Thought},
	year = {1975}
}

@book{LOT2,
    author = {Fodor, Jerry A.},
    title = "{LOT 2: The Language of Thought Revisited}",
    publisher = {Oxford University Press},
    year = {2008},
    month = {08},
    isbn = {9780199548774},
    doi = {10.1093/acprof:oso/9780199548774.001.0001},
    url = {https://doi.org/10.1093/acprof:oso/9780199548774.001.0001},
}

@article{computational-origin-of-repr,
author = {Piantadosi, Steven T.},
title = {The Computational Origin of Representation},
year = {2021},
issue_date = {Mar 2021},
publisher = {Kluwer Academic Publishers},
address = {USA},
volume = {31},
number = {1},
issn = {0924-6495},
url = {https://doi.org/10.1007/s11023-020-09540-9},
doi = {10.1007/s11023-020-09540-9},
journal = {Minds Mach.},
month = {mar},
pages = {1–58},
numpages = {58}
}

@article{arora2023theory,
  title={A theory for emergence of complex skills in language models},
  author={Arora, Sanjeev and Goyal, Anirudh},
  journal={arXiv preprint arXiv:2307.15936},
  year={2023}
}

@inproceedings{
hoffmann2022an,
title={An empirical analysis of compute-optimal large language model training},
author={Jordan Hoffmann and Sebastian Borgeaud and Arthur Mensch and Elena Buchatskaya and Trevor Cai and Eliza Rutherford and Diego de las Casas and Lisa Anne Hendricks and Johannes Welbl and Aidan Clark and Tom Hennigan and Eric Noland and Katherine Millican and George van den Driessche and Bogdan Damoc and Aurelia Guy and Simon Osindero and Karen Simonyan and Erich Elsen and Oriol Vinyals and Jack William Rae and Laurent Sifre},
booktitle={Advances in Neural Information Processing Systems},
editor={Alice H. Oh and Alekh Agarwal and Danielle Belgrave and Kyunghyun Cho},
year={2022},
url={https://openreview.net/forum?id=iBBcRUlOAPR}
}

@article{zhang2023and,
  title={What and How does In-Context Learning Learn? Bayesian Model Averaging, Parameterization, and Generalization},
  author={Zhang, Yufeng and Zhang, Fengzhuo and Yang, Zhuoran and Wang, Zhaoran},
  journal={arXiv preprint arXiv:2305.19420},
  year={2023}
}

@article{wang2023large,
  title={Large language models are implicitly topic models: Explaining and finding good demonstrations for in-context learning},
  author={Wang, Xinyi and Zhu, Wanrong and Wang, William Yang},
  journal={arXiv preprint arXiv:2301.11916},
  year={2023}
}

@article{akyurek2024context,
  title={In-Context Language Learning: Arhitectures and Algorithms},
  author={Aky{\"u}rek, Ekin and Wang, Bailin and Kim, Yoon and Andreas, Jacob},
  journal={arXiv preprint arXiv:2401.12973},
  year={2024}
}

@article{fu2023transformers,
  title={Transformers Learn Higher-Order Optimization Methods for In-Context Learning: A Study with Linear Models},
  author={Fu, Deqing and Chen, Tian-Qi and Jia, Robin and Sharan, Vatsal},
  journal={arXiv preprint arXiv:2310.17086},
  year={2023}
}

@article{guo2023transformers,
  title={How do transformers learn in-context beyond simple functions? a case study on learning with representations},
  author={Guo, Tianyu and Hu, Wei and Mei, Song and Wang, Huan and Xiong, Caiming and Savarese, Silvio and Bai, Yu},
  journal={arXiv preprint arXiv:2310.10616},
  year={2023}
}

@inproceedings{
bietti2023birth,
title={Birth of a Transformer: A Memory Viewpoint},
author={Alberto Bietti and Vivien Cabannes and Diane Bouchacourt and Herve Jegou and Leon Bottou},
booktitle={Thirty-seventh Conference on Neural Information Processing Systems},
year={2023},
url={https://openreview.net/forum?id=3X2EbBLNsk}
}

@article{todd2023function,
  title={Function vectors in large language models},
  author={Todd, Eric and Li, Millicent L and Sharma, Arnab Sen and Mueller, Aaron and Wallace, Byron C and Bau, David},
  journal={arXiv preprint arXiv:2310.15213},
  year={2023}
}

@inproceedings{hendel-etal-2023-context,
    title = "In-Context Learning Creates Task Vectors",
    author = "Hendel, Roee  and
      Geva, Mor  and
      Globerson, Amir",
    editor = "Bouamor, Houda  and
      Pino, Juan  and
      Bali, Kalika",
    booktitle = "Findings of the Association for Computational Linguistics: EMNLP 2023",
    month = dec,
    year = "2023",
    address = "Singapore",
    publisher = "Association for Computational Linguistics",
    url = "https://aclanthology.org/2023.findings-emnlp.624",
    doi = "10.18653/v1/2023.findings-emnlp.624",
    pages = "9318--9333",
}

@article{shen2023pretrained,
  title={Do pretrained Transformers Really Learn In-context by Gradient Descent?},
  author={Shen, Lingfeng and Mishra, Aayush and Khashabi, Daniel},
  journal={arXiv preprint arXiv:2310.08540},
  year={2023}
}

@article{kossen2023context,
  title={In-context learning in large language models learns label relationships but is not conventional learning},
  author={Kossen, Jannik and Rainforth, Tom and Gal, Yarin},
  journal={arXiv preprint arXiv:2307.12375},
  year={2023}
}

@article{bai2023transformers,
  title={Transformers as Statisticians: Provable In-Context Learning with In-Context Algorithm Selection},
  author={Bai, Yu and Chen, Fan and Wang, Huan and Xiong, Caiming and Mei, Song},
  journal={arXiv preprint arXiv:2306.04637},
  year={2023}
}

@inproceedings{socher-etal-2013-recursive,
    title = "Recursive Deep Models for Semantic Compositionality Over a Sentiment Treebank",
    author = "Socher, Richard  and
      Perelygin, Alex  and
      Wu, Jean  and
      Chuang, Jason  and
      Manning, Christopher D.  and
      Ng, Andrew  and
      Potts, Christopher",
    editor = "Yarowsky, David  and
      Baldwin, Timothy  and
      Korhonen, Anna  and
      Livescu, Karen  and
      Bethard, Steven",
    booktitle = "Proceedings of the 2013 Conference on Empirical Methods in Natural Language Processing",
    month = oct,
    year = "2013",
    address = "Seattle, Washington, USA",
    publisher = "Association for Computational Linguistics",
    url = "https://aclanthology.org/D13-1170",
    pages = "1631--1642",
}

@inproceedings{pang-lee-2005-seeing,
    title = "Seeing Stars: Exploiting Class Relationships for Sentiment Categorization with Respect to Rating Scales",
    author = "Pang, Bo  and
      Lee, Lillian",
    editor = "Knight, Kevin  and
      Ng, Hwee Tou  and
      Oflazer, Kemal",
    booktitle = "Proceedings of the 43rd Annual Meeting of the Association for Computational Linguistics ({ACL}{'}05)",
    month = jun,
    year = "2005",
    address = "Ann Arbor, Michigan",
    publisher = "Association for Computational Linguistics",
    url = "https://aclanthology.org/P05-1015",
    doi = "10.3115/1219840.1219855",
    pages = "115--124",
}

@inproceedings{cr,
author = {Hu, Minqing and Liu, Bing},
title = {Mining and summarizing customer reviews},
year = {2004},
isbn = {1581138881},
publisher = {Association for Computing Machinery},
address = {New York, NY, USA},
url = {https://doi.org/10.1145/1014052.1014073},
doi = {10.1145/1014052.1014073},
booktitle = {Proceedings of the Tenth ACM SIGKDD International Conference on Knowledge Discovery and Data Mining},
pages = {168–177},
numpages = {10},
location = {Seattle, WA, USA},
series = {KDD '04}
}

@inproceedings{subj,
    title = "A Sentimental Education: Sentiment Analysis Using Subjectivity Summarization Based on Minimum Cuts",
    author = "Pang, Bo  and
      Lee, Lillian",
    booktitle = "Proceedings of the 42nd Annual Meeting of the Association for Computational Linguistics ({ACL}-04)",
    month = jul,
    year = "2004",
    address = "Barcelona, Spain",
    url = "https://aclanthology.org/P04-1035",
    doi = "10.3115/1218955.1218990",
    pages = "271--278",
}

@article{allen2023physics,
  title={Physics of language models: Part 3.1, knowledge storage and extraction},
  author={Allen-Zhu, Zeyuan and Li, Yuanzhi},
  journal={arXiv preprint arXiv:2309.14316},
  year={2023}
}

@inproceedings{chiang-etal-2024-retrieval,
    title = "On Retrieval Augmentation and the Limitations of Language Model Training",
    author = "Chiang, Ting-Rui  and
      Yu, Xinyan  and
      Robinson, Joshua  and
      Liu, Ollie  and
      Lee, Isabelle  and
      Yogatama, Dani",
    editor = "Duh, Kevin  and
      Gomez, Helena  and
      Bethard, Steven",
    booktitle = "Proceedings of the 2024 Conference of the North American Chapter of the Association for Computational Linguistics: Human Language Technologies (Volume 2: Short Papers)",
    month = jun,
    year = "2024",
    address = "Mexico City, Mexico",
    publisher = "Association for Computational Linguistics",
    url = "https://aclanthology.org/2024.naacl-short.21",
    doi = "10.18653/v1/2024.naacl-short.21",
    pages = "229--238",
}

@inproceedings{li-etal-2023-symbolic,
    title = "Symbolic Chain-of-Thought Distillation: Small Models Can Also {``}Think{''} Step-by-Step",
    author = "Li, Liunian Harold  and
      Hessel, Jack  and
      Yu, Youngjae  and
      Ren, Xiang  and
      Chang, Kai-Wei  and
      Choi, Yejin",
    editor = "Rogers, Anna  and
      Boyd-Graber, Jordan  and
      Okazaki, Naoaki",
    booktitle = "Proceedings of the 61st Annual Meeting of the Association for Computational Linguistics (Volume 1: Long Papers)",
    month = jul,
    year = "2023",
    address = "Toronto, Canada",
    publisher = "Association for Computational Linguistics",
    url = "https://aclanthology.org/2023.acl-long.150",
    doi = "10.18653/v1/2023.acl-long.150",
    pages = "2665--2679",
}

@inproceedings{hsieh-etal-2023-distilling,
    title = "Distilling Step-by-Step! Outperforming Larger Language Models with Less Training Data and Smaller Model Sizes",
    author = "Hsieh, Cheng-Yu  and
      Li, Chun-Liang  and
      Yeh, Chih-kuan  and
      Nakhost, Hootan  and
      Fujii, Yasuhisa  and
      Ratner, Alex  and
      Krishna, Ranjay  and
      Lee, Chen-Yu  and
      Pfister, Tomas",
    editor = "Rogers, Anna  and
      Boyd-Graber, Jordan  and
      Okazaki, Naoaki",
    booktitle = "Findings of the Association for Computational Linguistics: ACL 2023",
    month = jul,
    year = "2023",
    address = "Toronto, Canada",
    publisher = "Association for Computational Linguistics",
    url = "https://aclanthology.org/2023.findings-acl.507",
    doi = "10.18653/v1/2023.findings-acl.507",
    pages = "8003--8017",
}

@inproceedings{shridhar-etal-2023-distilling,
    title = "Distilling Reasoning Capabilities into Smaller Language Models",
    author = "Shridhar, Kumar  and
      Stolfo, Alessandro  and
      Sachan, Mrinmaya",
    editor = "Rogers, Anna  and
      Boyd-Graber, Jordan  and
      Okazaki, Naoaki",
    booktitle = "Findings of the Association for Computational Linguistics: ACL 2023",
    month = jul,
    year = "2023",
    address = "Toronto, Canada",
    publisher = "Association for Computational Linguistics",
    url = "https://aclanthology.org/2023.findings-acl.441",
    doi = "10.18653/v1/2023.findings-acl.441",
    pages = "7059--7073",
}

\appendix

\section{A Pelican Soup Example}
\label{sec:pelican-soup-example}

We include an example of a Pelican Soup game:

\textbf{Puzzle master:} A men walks into a restaurant and orders pelican soup. After taking a sip, he loses his mind. Why?

\textbf{Participants:} Is it because the soup is not cooked well?

\textbf{Puzzle master:} No.

\textbf{Participants:} Is it because the soup toxic?

\textbf{Puzzle master:} No.

\textbf{Participants:} Does the soup remind him something?

\textbf{Puzzle master:} Yes.

\textbf{Participants:} Did someone cook pelican soup for him?

\textbf{Puzzle master:} Yes.

\textbf{Participants:} Is that person still alive?

\textbf{Puzzle master:} No.

For the sake of aesthetics, we do not include the latent story here. If you are interested, please check it online.

\section{Proof of Theorem~\ref{thm:main}}
\label{sec:icl-proof}

Let $S_t = x_1, r_2, d, x_2, r_2, d \cdots, x_t, r_t, d$.

\begin{align*}
&\Pr(z | S_t) \\
=& \frac{\Pr(S_t | z)\Pr(z)}{\sum_z \Pr(S_t | z) \Pr(z)} \\
=& \frac{\Pr(z) \prod_{i=1}^t \Pr(x_i, r_i, d | z, S_{i-1})}{
\sum_{z'} \Pr(z') \prod_{i=1}^t \Pr(x_i, r_i, d | z', S_{i-1})}
\end{align*}

\begin{align*}
&P(x_{t + 1}, r_{t + 1}, d | S_{t}) \\
=& \sum_{z} \Pr(x_{t + 1}, r_{t + 1}, d | z, S_{t}) \Pr(z| S_t) \\
=& \frac{\sum_{z} \Pr(z) \prod_{i=1}^{t+1} \Pr(x_i, r_i, d | z, S_{i-1})}{\sum_{z'} \Pr(z) \prod_{i=1}^t\Pr(x_i, r_i, d | z', S_{i-1})}.
\end{align*}

Thus, it holds that

\begin{align*}
&-\sum_{t=0}^T \log \Pr(x_{t + 1}, r_{t + 1}, d | S_{t}) \\
=&  
\begin{aligned}[t]
&-\sum_{t=0}^T \Big( \\
&\log  \sum_{z} \Pr(z) \prod_{i=1}^{t+1} \Pr(x_i, r_i, d | z, S_{i-1}) \\
& -\log  \sum_{z} \Pr(z) \prod_{i=1}^{t} \Pr(x_i, r_i, d | z, S_{i-1})  \\
&\Big)
\end{aligned}
\\
=& -\log \sum_{z} \Pr(z | K) \prod_{i=1}^{T + 1} \Pr(x_i, r_t, d | z, S_{i-1}) \\
\leq& -\log \Pr(z^*) \prod_{i=1}^{T + 1} \Pr(x_i, r_t, d | z^*, S_{i-1}) \\
=& -\log \Pr(z^*) \\
&- \sum_{i=1}^{T + 1} \log \Pr(x_i, r_i, d | z^*, S_{i-1}) \\
=& -\log \Pr(z^*) \\
&- \sum_{i=1}^{T} \log \Pr(r_{i}, d | x_{i}, z^*, S_{i - 1}) \\
&- \sum_{i=1}^{T}  \log \Pr(x_{i} | z^*, S_{i - 1}).
\end{align*}

Thus,

\begin{align*}
&-\frac{1}{T}\sum_{t=0}^T \log \Pr(r_{t} | x_{t}, S_{t - 1}) \\
\leq& - \frac{1}{T} \bigg( 
\log \Pr(z^*) \\
&+\sum_{i=1}^{T} \log \Pr(r_{t}, d | x_{t}, z^*, S_{t - 1}) \\
&+ \sum_{i=1}^{T}  \log \frac{\Pr(x_{t} | z^*, S_{t - 1})}{\Pr(x_{t} | S_{t - 1})} ) \\
&+ \frac{1}{T}\sum_{i=1}^{T} \log \Pr(d | r_t, x_t, S_{t-1} \bigg). \\
\leq& - \frac{1}{T} \bigg(
\log \Pr(z^*) \\
&+\sum_{i=1}^{T} \log \Pr(r_{t}, d | x_{t}, z^*, S_{t - 1}) \\ 
&+ \sum_{i=1}^{T}  \log \frac{\Pr(x_{t} | z^*, S_{t - 1})}{\Pr(x_{t} | S_{t - 1})} \bigg)
\end{align*}

\section{Bounding ICL for Generation}
\label{sec:bound-for-gen}

We can extend the ICL bound shown in Theorem~\ref{thm:main} to generation tasks.
When using ICL for generation, in the demonstration part of the prompt, people usually include a separator between the input and the output.
For example, if the task is to translate a sentence to English, people may write the demonstration in the format  ``[input] should be converted to [output].''
This ``should be converted to'' is an expression that can be associated to the ``meaning'' of the task, namely ``should be mapped to a translation in English.''
The intuition of the theorem is that, a language model may be able to uncover this association through the demonstration examples via the same mechanism as how it perform ICL for classification tasks.

\begin{theorem}[Average ICL Loss for Generation]
Let $z^*$ represent the association between a separator $\xi$ and a task descriptions such that for any input $x$, $P(x, \xi, r | z^*) > 0$ only if $r$ is one of the correct outputs as specified by the task. 
Let $K$ be the constraints used for decoding, and $\dot{g}$ be the event where a document follows certain formats.
Let $R_t$ be the set of correct outputs for $x_t$, $S_t = \{ x_1, \xi, r_2, d, x_2, \xi, r_2, \cdots, x_t, \xi, r_t, d | r_1 \in R_1, r_2 \in R_2, \cdots, r_t \in R_t \}$.
We have for any integer $T > 0$, the average cross-entropy loss of ICL is bounded as:
\begin{equation}
\begin{split}
    &-\frac{1}{T}\sum_{t=0}^T \log \Pr(R_{t} | x_{t}, \xi, S_{t - 1}) \\
&\leq \begin{aligned}[t]
    & - \frac{1}{T} \log \Pr(z^*) \\
    & - \frac{1}{T} \sum_{i=1}^{T} \log \Pr(R_{t}, d | x_{t}, \xi , z^*, S_{t - 1}) \\
    & - \frac{1}{T} \sum_{i=1}^{T} \log \frac{\Pr(x_{t}, \xi | z^*, S_{t - 1})}{\Pr(x_{t}, \xi | S_{t - 1})}
\end{aligned} 
\end{split}
\label{eq:icl-gen}
\end{equation}
\label{thm:icl-gen}
\end{theorem}
\begin{proof}
If we see $\xi$ as a part of the $x$'s in the proof for Theorem~\ref{thm:main}, then following the steps in \S\ref{sec:icl-proof}, we can have Eq.~\ref{eq:icl-gen}.
\end{proof}

\section{Proof of Corollary \ref{cor:dst-exp}}
The second term in the right-hand side of Eq.~\ref{eq:main} is zero when the decoding constrain $K$ is imposed. 
Therefore, it suffices to prove the last term is non-negative in expectation.

\begin{align*}
&\mathop{\mathbb{E}}_{x_1, x_2, \cdots, x_T \sim \mathcal{D}_{X}^T }  \sum_{i=1}^{T}  \log \frac{\Pr(x_{t} | z^*, S_{t - 1}, K)}{\Pr(x_{t} | S_{t - 1}, K)} \\
=& \mathop{\mathbb{E}}_{x_1, x_2, \cdots, x_T \sim \mathcal{D}_{X}^T }  \sum_{i=1}^{T}  \log \frac{\Pr(x_{t} | z^*, K)}{\Pr(x_{t} | S_{t - 1}, K)} \\
=& \sum_{x_1, x_2, \cdots, x_T }  \Pr(x_{t} | z^*, K) \sum_{i=1}^{T}  \log \frac{\Pr(x_{t} | z^*, K)}{\Pr(x_{t} | S_{t - 1}, K)} \\
=& \mathrm{KLD}(\Pr(x_{t} | z^*, K) || \Pr(x_{t} | S_{t - 1}, K))
\geq 0 \\
\end{align*}

\section{The Connection between $P(r_t | x_t, z^*)$ and Function Description Length by \citet{hahn2023theory}}
\label{sec:discuss-desc-length}

Firstly, we make some regularity assumptions: Given a step-by-step reasoning process $\pi = s_1, s_2, \cdots, s_n$ for the induction process of $P \models Q$, in the training data,
\begin{enumerate}
    \item each step may be dropped independently to each other with probability $p_{drop}$.
    \item $\Pr(s_i | P, s_1, s_2, \cdots, s_{i-1}) > p_{min}$ for all $i \in [n]$.
\end{enumerate}

We first show how we derive Eq.~\ref{eq:desc-length}: Based on Assumption~\ref{assumption:step-by-step},
\begin{equation*}
\begin{split}
    &\Pr(r_t | x_t, z^*) \\
    &= \sum_{\pi \in \Pi} \Pr(\pi, r_t | x_t, z^*) \Pr(\text{$\pi$ is dropped}),    
\end{split}
\end{equation*}
where $\Pi$ is a set of token sequences representing reasoning steps that induce $r_t$ from $x_t$.
Let $\pi^*$ be the shortest proof in $\Pi$, we have
\begin{align*}
    &\log \Pr(r_t | x_t, z^*) \\
    =& \log \sum_{\pi \in \Pi} \Pr(\pi, r_t | x_t, z^*) \Pr(\text{$\pi$ is dropped}) \\
    \geq& \log \Pr(\pi^*, r_t | x_t, z^*) \Pr(\text{$\pi^*$ is dropped}) \\
    \geq& p_{min} \log \ell(\pi^*) + p_{drop} \log \ell(\pi^*).
\end{align*}

Then we can discuss the connection between $\Pr(r_t | x_t, z^*, \ddot{g})$ and the function description length by \citet{hahn2023theory}.
We can see the dropped reasoning steps in $\pi^*$ as the hidden (tree) structure that maps $x_t$ to $r_t$ as the derivation tree $\tau_\phi$ in the bound of \citet{hahn2023theory}.
The length of the reasoning steps thus corresponds to the description length of the derivation tree $\mathrm{D}(\tau_\phi)$. 

A major difference between the bound by \citet{hahn2023theory} and our bound is that their bound has $\mathrm{D}(\tau_\phi)$ constant to $T$ while our bound has $\sum_t \log \Pr(r_t | x_t, z^*, \ddot{g})$, which potentially grows proportionally to $T$.
The cause of this difference is that, \citet{hahn2023theory} assumes a structure that repetitively applies a function mapping in a document, and the number of repetition is independent to the complexity of the function mapping.
In comparison, our framework does not make this assumption.

\section{Details of the Calcutec Experiment}

\subsection{Generation Process of the LM Training Data in Calcutec }
\label{sec:calcutec-train-gen}

We generate a paragraph based on Assumption~\ref{assumption:pronoun} in the following step:
\begin{enumerate}[leftmargin=0.25in]
    \item We pick a symbol $s$ from the symbols associated with $r_a$ uniformly at random.
    \item We randomly generate a proof for $\mathrm{KB}, P \models g$, where $P \subset \Sigma$ is the premise and $g \in \Sigma$ is the goal of the proof. We ensure that this proof contains the topic $s$.
    \item We convert the proof tree to a sequence of proving steps by traversing the proving tree in a topological order with ties broken randomly.
    Each node in the proof tree corresponds to a rule in $\mathrm{KB}$, so the resulting sequence of proving steps consists of horn clauses in the form $a_1 a_2 \to b$. We separate the clauses in the sequence with commas. 
    \item We rewrite the first step of the proving process to contain the premises of the proof.
    Specifically, we replace the antecedent in the first formula with the premise $P$.
    We find that this step is necessary to prevent the language model trained on it from hallucinating irrelevant variables randomly. 
    It is important for our experiment for chain-of-thought, but is not necessary for language models to learn the in-context learning ability.
\end{enumerate}

\begin{algorithm*}
    \begin{algorithmic}
    \State Sample $r_a, r_b$ from $\{r_1, r_2, r_3, r_4\}$ with probability 0.45, 0.45, 0.05, 0.05.
    \State Sample topic $S = \{s_1, s_2 \} \subset \Sigma$.
    \State Initialize a document $D$ with empty string.
    \For{$p = 1, 2, \dots, n_{par}$}
        \While{True}
            \State Sample $s \in S$.
            \State Sample a set $X \subset \Sigma$ such that $\bigwedge_{x \in X} x \models s$.
            \State Run the resolution algorithm to get the set $M = \{ m | X  \models m \}$.
            \State Find an extra premise $x'$ that can increase the depth of deepest proof tree for $X \models m$.
            \State Run the resolution algorithm to get the set $M' = \{ m | X \cup \{x'\} \models m \}$.
            \If {$|M'| > \frac{|\Sigma|}{2}$}
                \State Reject the sampled $X \cup \{ x' \}$. 
                \Comment{We don't want a premise that entails everything.}
                \State Restart the while loop.
            \EndIf
            \State Sample a $g \in M'$ such that the proof tree for $X' \models g$ contains $s$ and its depth $> d_{min}$. 
            \Comment{We use $d_{min} = 4$ in our experiments.}
            \State Do topological sort to flatten the proof tree and convert it into a string. 
            \State Append the string to $D$. 
        \EndWhile
    \EndFor

    \For{$s \in S$}
        \State $D \leftarrow$ $D$.replace($s$, $r_a$)
    \EndFor

    \State Let $S' = \{s'_1, s'_2 \} \in \Sigma$ be the top-2 frequent non $r_a$ symbols in $D$.
    \For{$s' \in S'$}
        \State $D \leftarrow$ $D$.replace($s'$, $r_b$)
    \EndFor
    \end{algorithmic}
    \caption{Pseudo code for the generation process of a Calcutec document used for training.}
    \label{alg:calcutec-train}
\end{algorithm*}

\subsection{Perturbations in Calcutec}
\label{sec:perturbations}

We apply two types of perturbations over the reasoning steps in Calcutec described in \S{\ref{sec:calcutec}}:
\begin{enumerate}[leftmargin=0.25in]
    \item Random merge: At probability $p_{merge}$, for every two consecutive clauses where the consequence of the first one is in the antecedents of the second one, say $a_1 a_2 \to b_1$ and $b_1 a_3 \to b_2$, we merge them into a single clause $a_1 a_2 a_3 \to b_2$.   
    \item Random drop: Given a clause $a_1 a_2 \cdots a_n \to b$.
    We drop each of the antecedents $a \in \{a_1, a_2, \cdots a_n\}$ at probability $p_{drop}$.
    We apply this dropping to every clause in the proof except the first one to ensure that we do not drop the premises. 
\end{enumerate}
We use $p_{merge} = p_{drop} = p_{skip}$.

Additionally, when flattening the proof trees with topological sort, we break the ties randomly. We also randomize the order of the symbols in the antecedents. 

\begin{figure*}[H]
    \centering
    \begin{subfigure}{0.24\linewidth}
        \includegraphics[width=\linewidth]{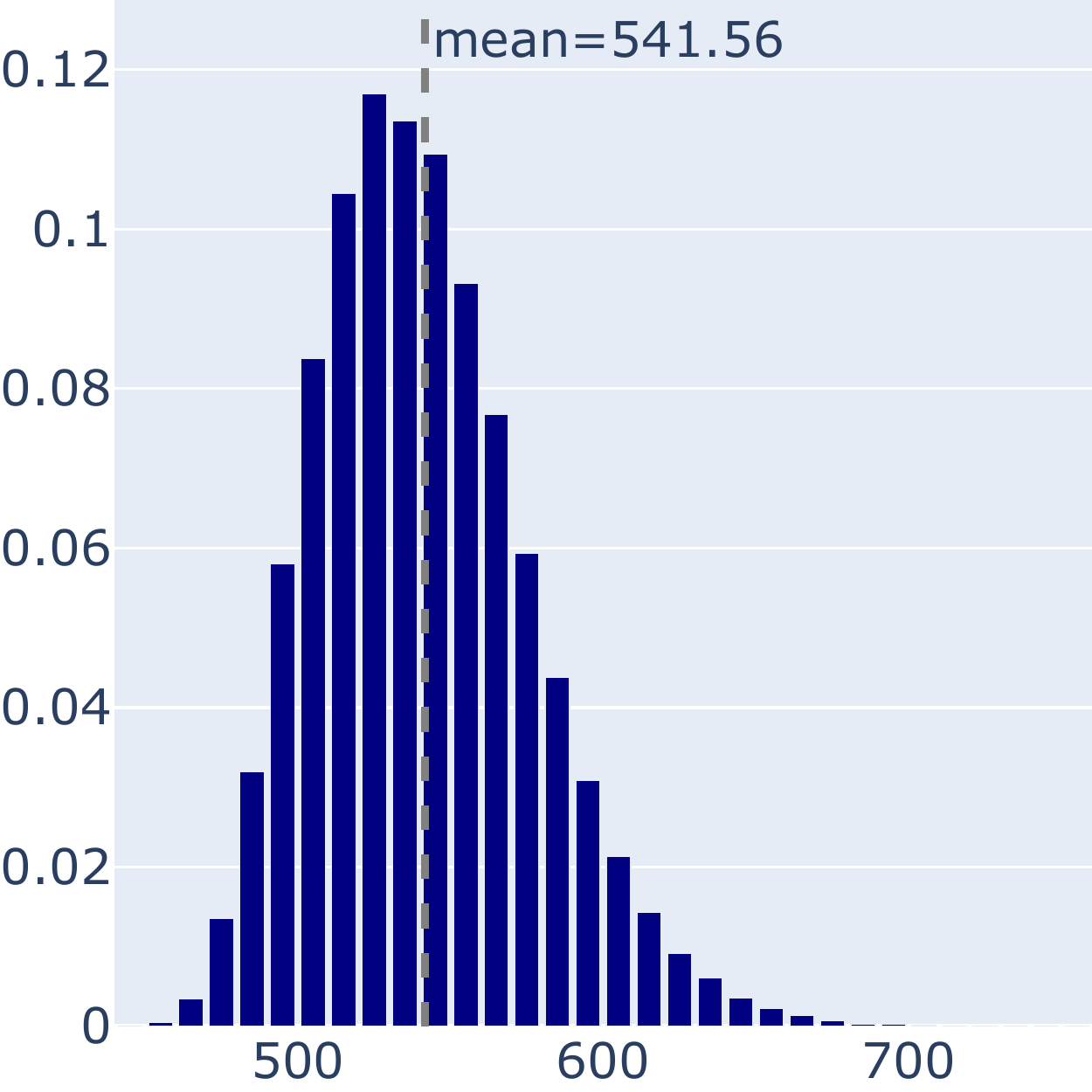}     
        \subcaption[]{document (tokens)}
    \end{subfigure}
    \begin{subfigure}{0.24\linewidth}
        \includegraphics[width=\linewidth]{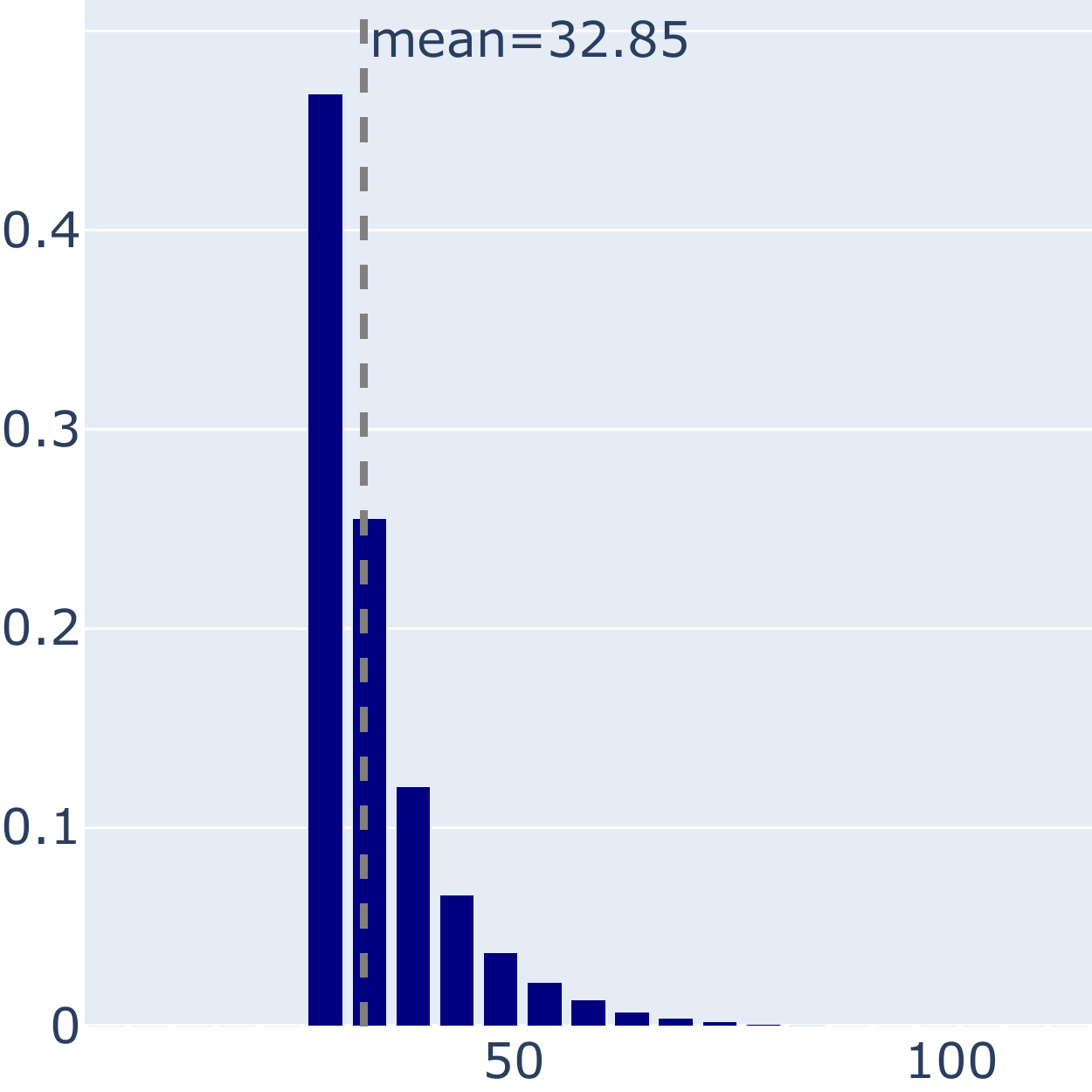}
        \subcaption[]{paragraph (tokens)}
    \end{subfigure}
    \begin{subfigure}{0.24\linewidth}
        \includegraphics[width=\linewidth]{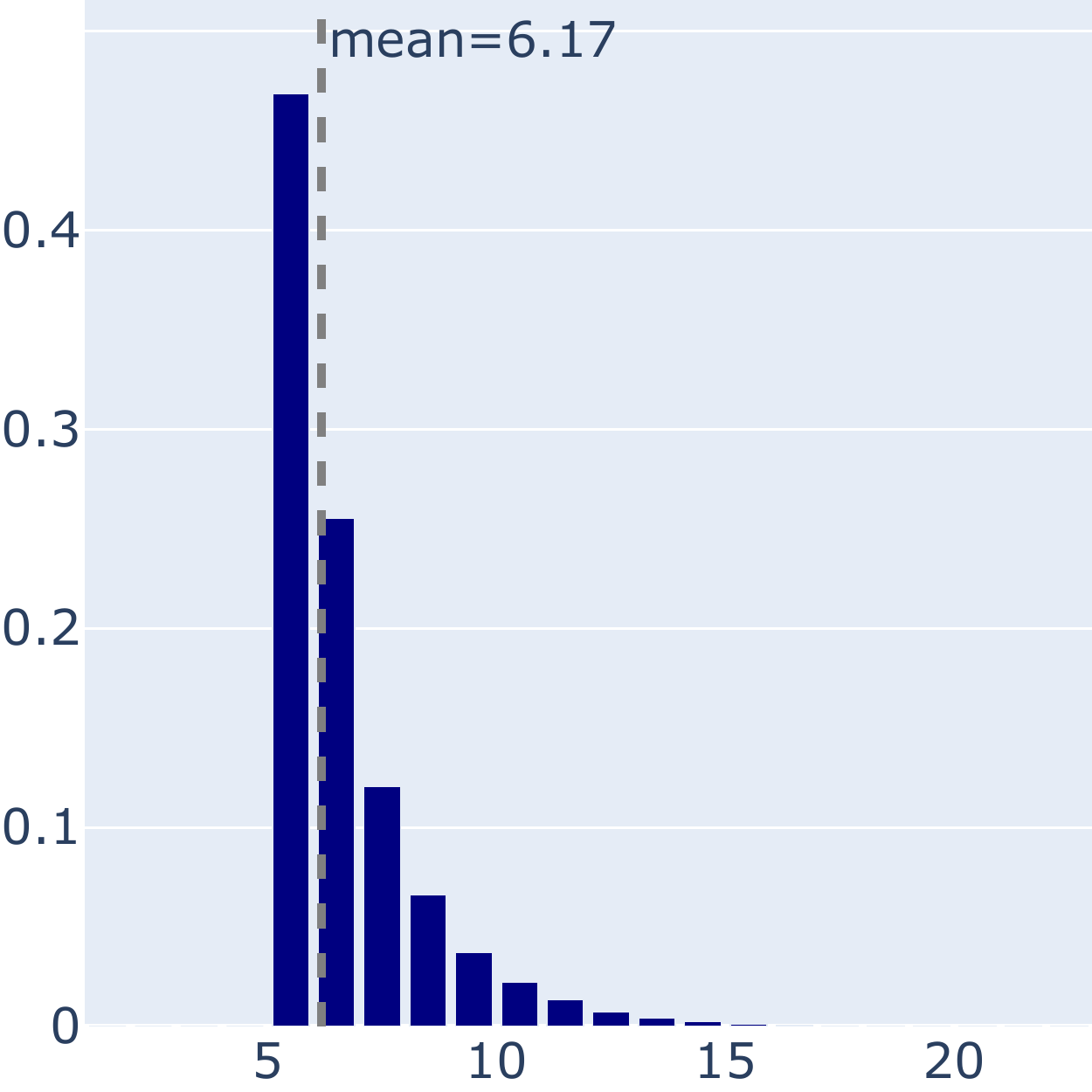}
        \subcaption[]{paragraph (steps)}
    \end{subfigure}
    \begin{subfigure}{0.24\linewidth}
        \includegraphics[width=\linewidth]{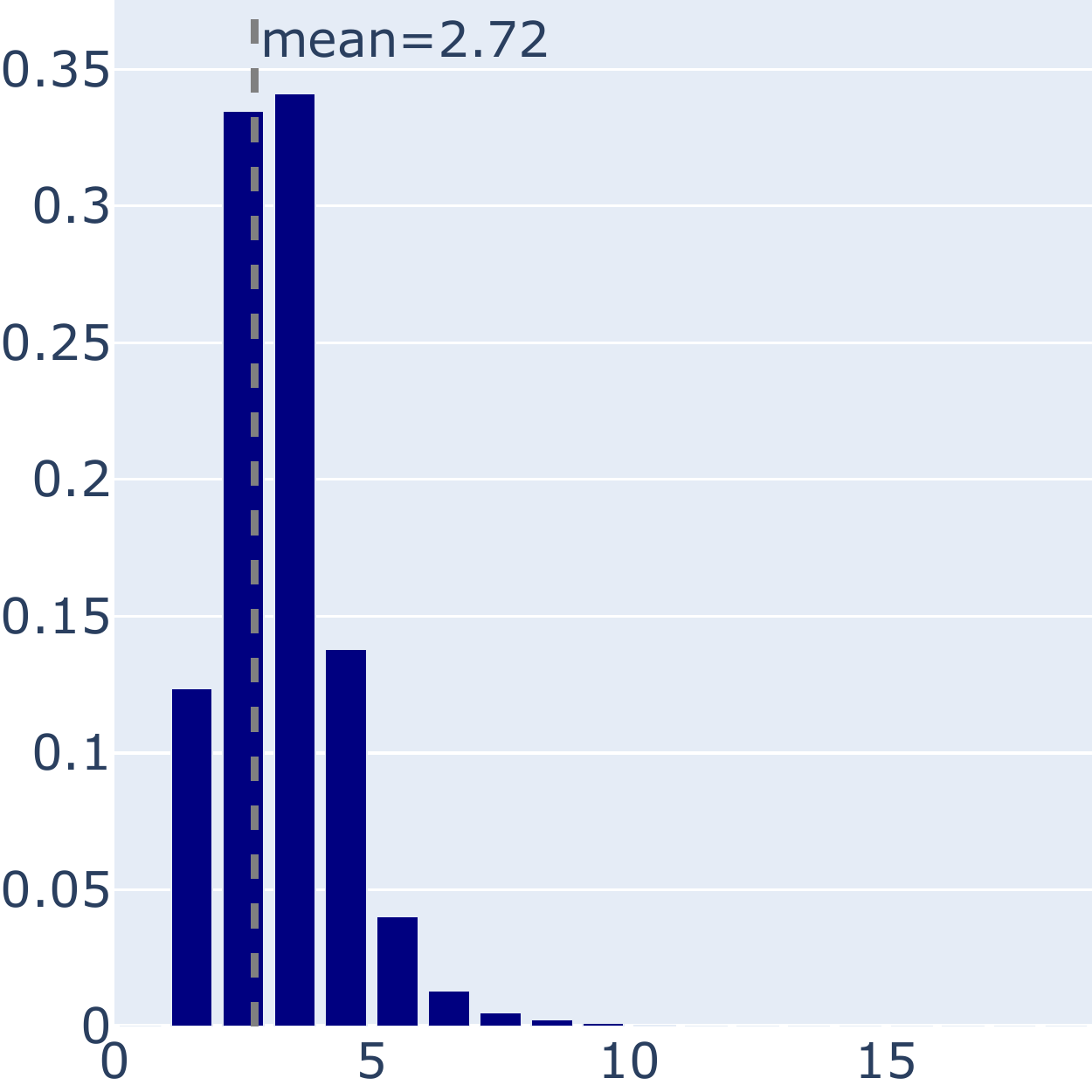}
        \subcaption[]{first $z$ (steps)}
    \end{subfigure}

    \begin{subfigure}{0.24\linewidth}
        \includegraphics[width=\linewidth]{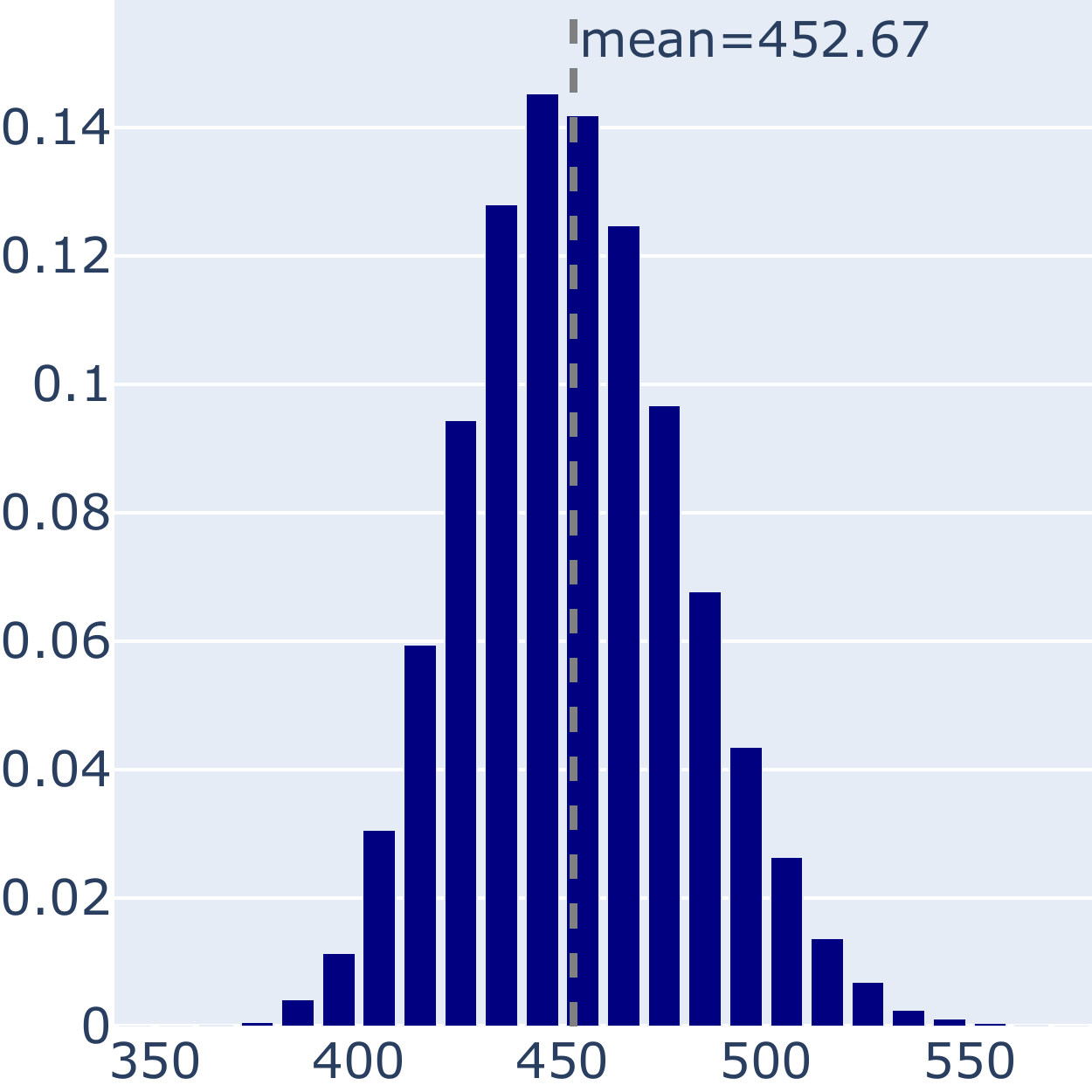} 
        \subcaption[]{document (tokens)}
    \end{subfigure}
    \begin{subfigure}{0.24\linewidth}
        \includegraphics[width=\linewidth]{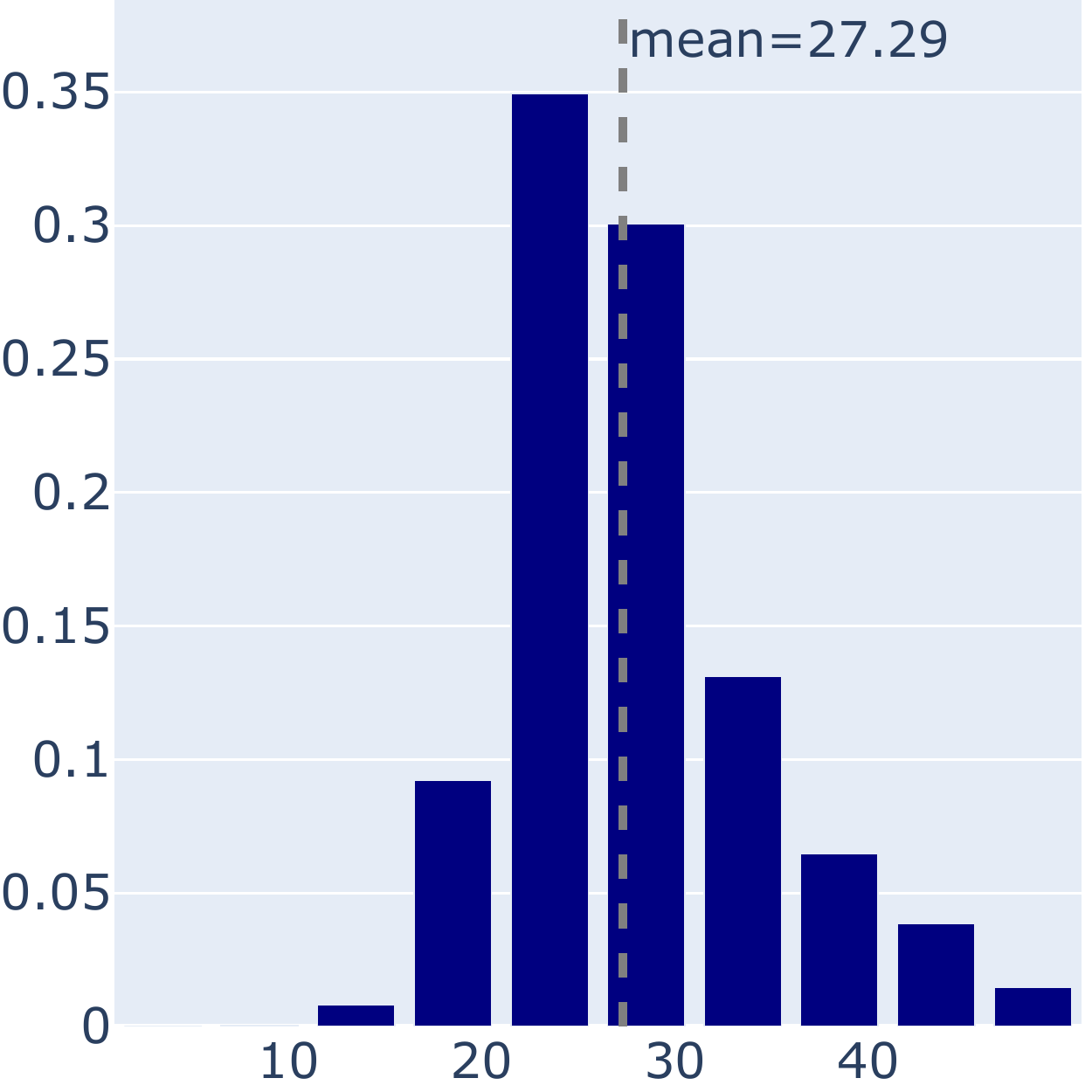}
        \subcaption[]{paragraph (tokens)}
    \end{subfigure}
    \begin{subfigure}{0.24\linewidth}
        \includegraphics[width=\linewidth]{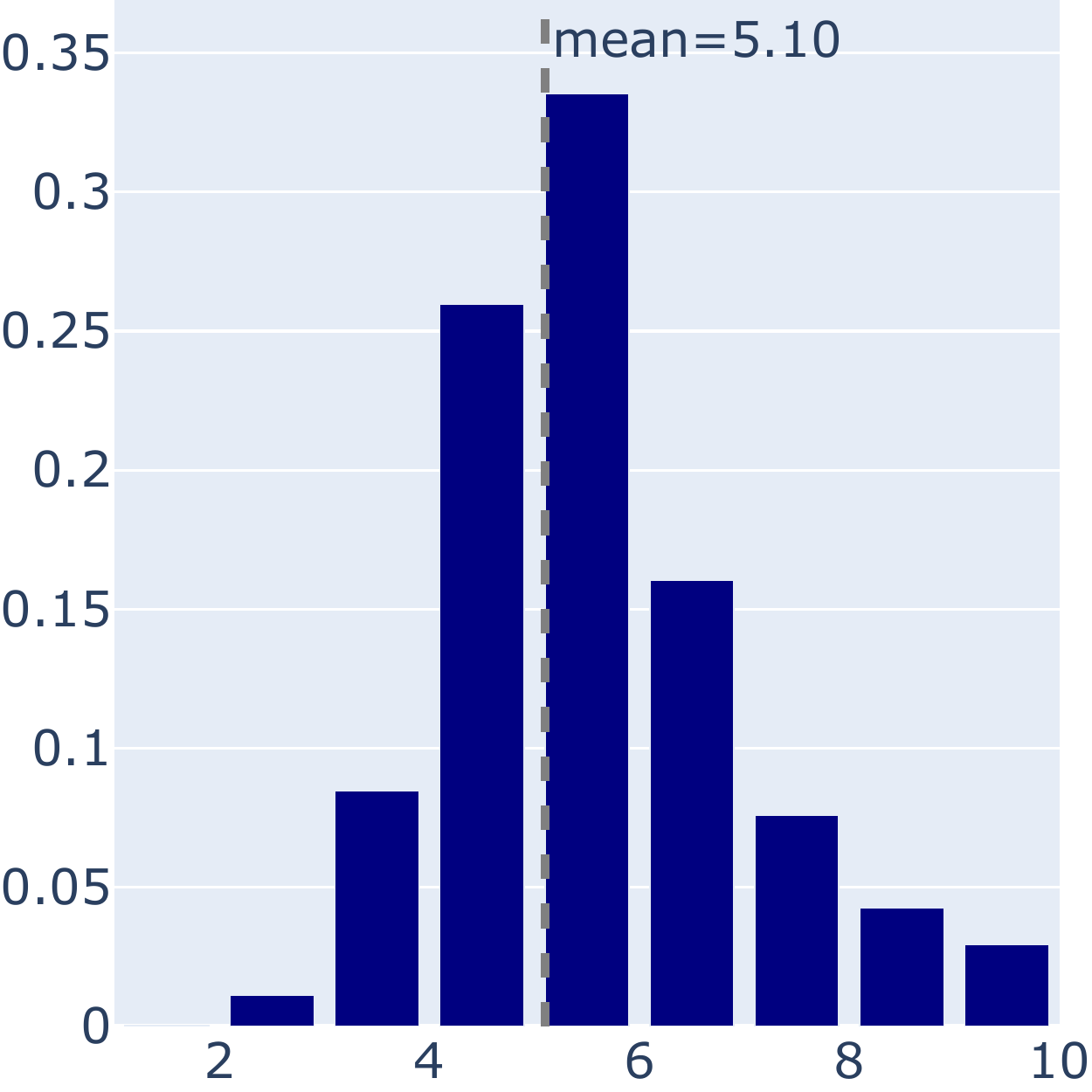}
        \subcaption[]{paragraph (steps)}
    \end{subfigure}
    \begin{subfigure}{0.24\linewidth}
        \includegraphics[width=\linewidth]{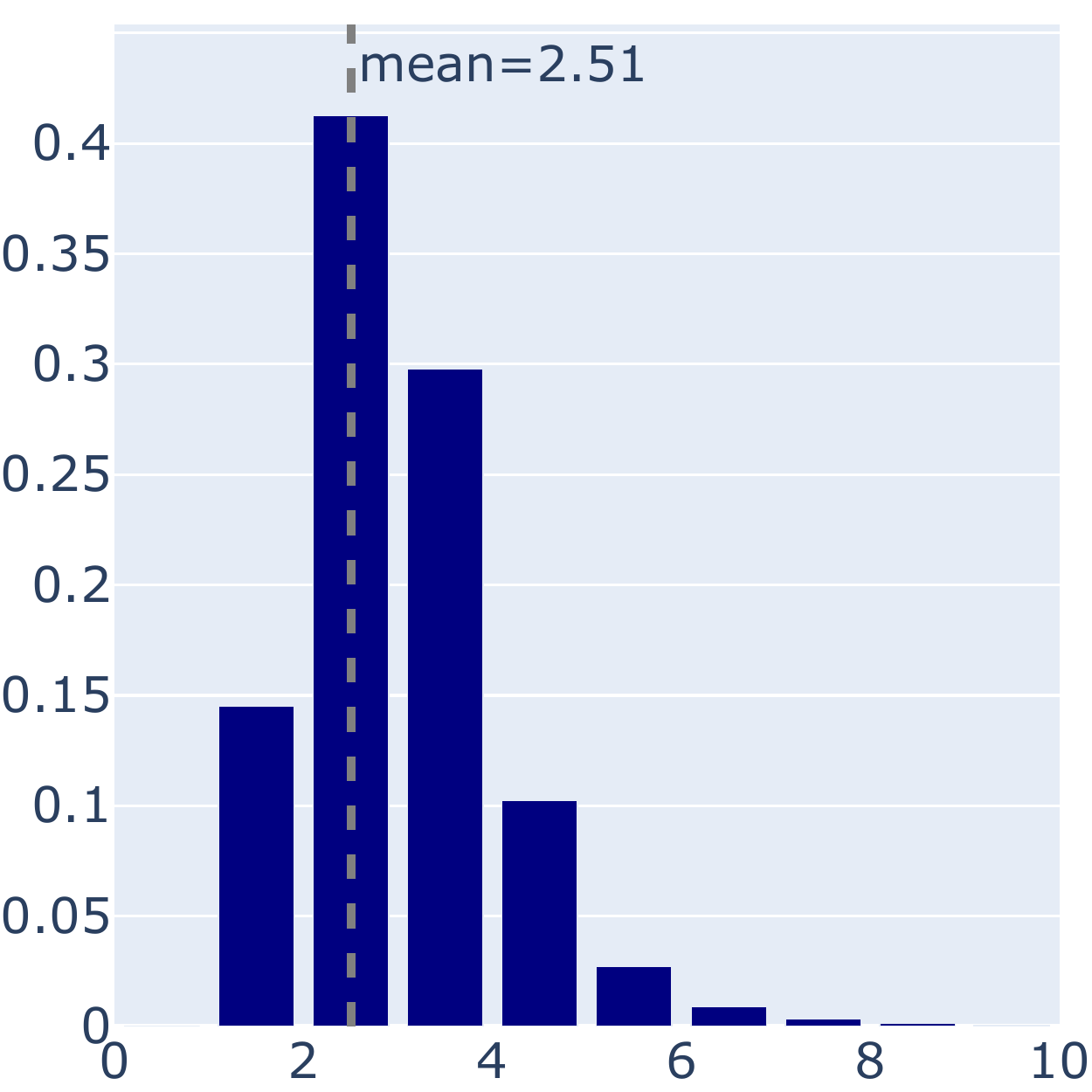}
        \subcaption[]{first $z$ (steps)}
    \end{subfigure}
    \caption{The distribution of lengths and the first step in each paragraph where $z$ is the consequence in the Calcutec dataset. The first/second row are the statistics before/after some steps are randomly dropped.}
    \label{fig:enter-label}
\end{figure*}

\subsection{Extra Setups}
\label{sec:extra-setups}

\paragraph{Unseen Inference Process.}
Based on Assumption~\ref{assumption:step-by-step} and the formalism of NLP tasks in \S\ref{sec:dist-task}, input-label pairs of a downstream task corresponds to prefix-reference pairs in a paragraph. 
To examine whether the trained model can generalize well when the induction process for the label is different from the induction process for the pronoun in the training data, we generate a training dataset where all the pronouns are induced from the premise with a left-branching proof tree with a depth equal to 2 (Figure~\ref{fig:tree-paragraph}), while the test data contains samples whose labels are induced from the input with balanced trees (Figure~\ref{fig:tree-downstream}).

\paragraph{Different Input Lengths.}
For each downstream tasks, we experiment with examples with different lengths. When the inference process is branching, having input length equal to $4$ makes the proving tree deeper.

\paragraph{No perturbations.} 
As described in \S\ref{sec:perturbations}, we apply some random perturbations on the proving process. We also experiment with the setup where we do not apply any perturbations.

\paragraph{With/Without Rewriting the First Step.} As described in \S\ref{sec:calcutec-train-gen}, we rewrite the first step of the proof. We also experiment with the setup where we do not rewrite the first step.

\paragraph{Model Size.}
We also experiment with different models sizes. We experiment with GPT-2 models that have 3, 4 and 5 layers.

\subsubsection{Results and Discussion}

\paragraph{Unseen Inference Process.}
Figure~\ref{fig:icl-balanced-y12-4} and Figure~\ref{fig:icl-balanced-y34-4} show that the ICL performance on the branching examples is similar to the performance on the branching examples.
It suggests that the model can generalize to examples that requires an unseen reasoning process.
Interestingly, Table~\ref{tab:cot-all} show that using chain-of-thoughts mitigates this gap.

\paragraph{Different Input Lengths.}  
Figure~\ref{fig:icl-branching-y12-4} and Figure~\ref{fig:icl-branching-y34-4} show that the model can still do ICL for the examples with length equal to $4$.
However, compared with the performance on examples with length equal to $3$ (Figure~\ref{fig:icl-branching-y12-3} and Figure~\ref{fig:icl-branching-y34-3}), the performance is worse.
This may be because solving these length-4 examples requires more reasoning steps.

\paragraph{With/Without Rewriting the First Step.}
Figure~\ref{fig:icl-no-fuse} shows that models trained with proofs that are rewritten has similar performance as models trained with the proofs that were rewritten (Figure~\ref{fig:icl-full}).
This suggests that rewriting the first step in the proof is not necessary for the model to acquire the ICL ability.

\paragraph{Model Size.}
Figure~\ref{fig:icl-depth} show that deeper models have better ICL performance. It aligns with the real-world observation that scaling helps model performance. 

\begin{figure}
    \centering
    \begin{subfigure}[b]{0.45\textwidth}
        \centering
        \begin{prooftree}
        \hypo{x_1}
        \hypo{x_2}
        \infer2{x_4}
        \hypo{x_3}
        \infer2{r_1}
        \hypo{x_1}
        \hypo{x_3}
        \infer2{x_6}
        \infer2{x_7}
        \hypo{\vdots}
        \infer2{\vdots}
        \end{prooftree}
        \caption{The proof tree a paragraph in the training dataset corresponds .}
        \label{fig:tree-paragraph}
    \end{subfigure}
    \begin{subfigure}[b]{0.45\textwidth}
        \centering
        \vspace{0.5cm}
        \begin{prooftree}
        \hypo{x_1}
        \hypo{x_2}
        \infer2{x_5}
        \hypo{x_3}
        \hypo{x_4}
        \infer2{x_6}
        \infer2{r_1}
        \end{prooftree}
        \caption{A balanced tree for a downstream task sample.}   
        \label{fig:tree-downstream}
    \end{subfigure}
    \caption{Proof trees examples.}
    \label{fig:my_label}
\end{figure}

\subsection{Hyper-parameters}
\label{sec:hyper-param}

We train our model using batch size 256, warm up ratio 5\%, and we truncate the sequence length to 512 tokens and the default parameters for the optimizer. We use the implementation of GPT-2 by Hugging Face transformers v4.27.2.
All models can be trained with 4 RTX 2080ti within 8 hours.

\begin{figure*}
    \centering
    \begin{subfigure}{0.32\textwidth}
        \centering
        \includegraphics[width=\textwidth]{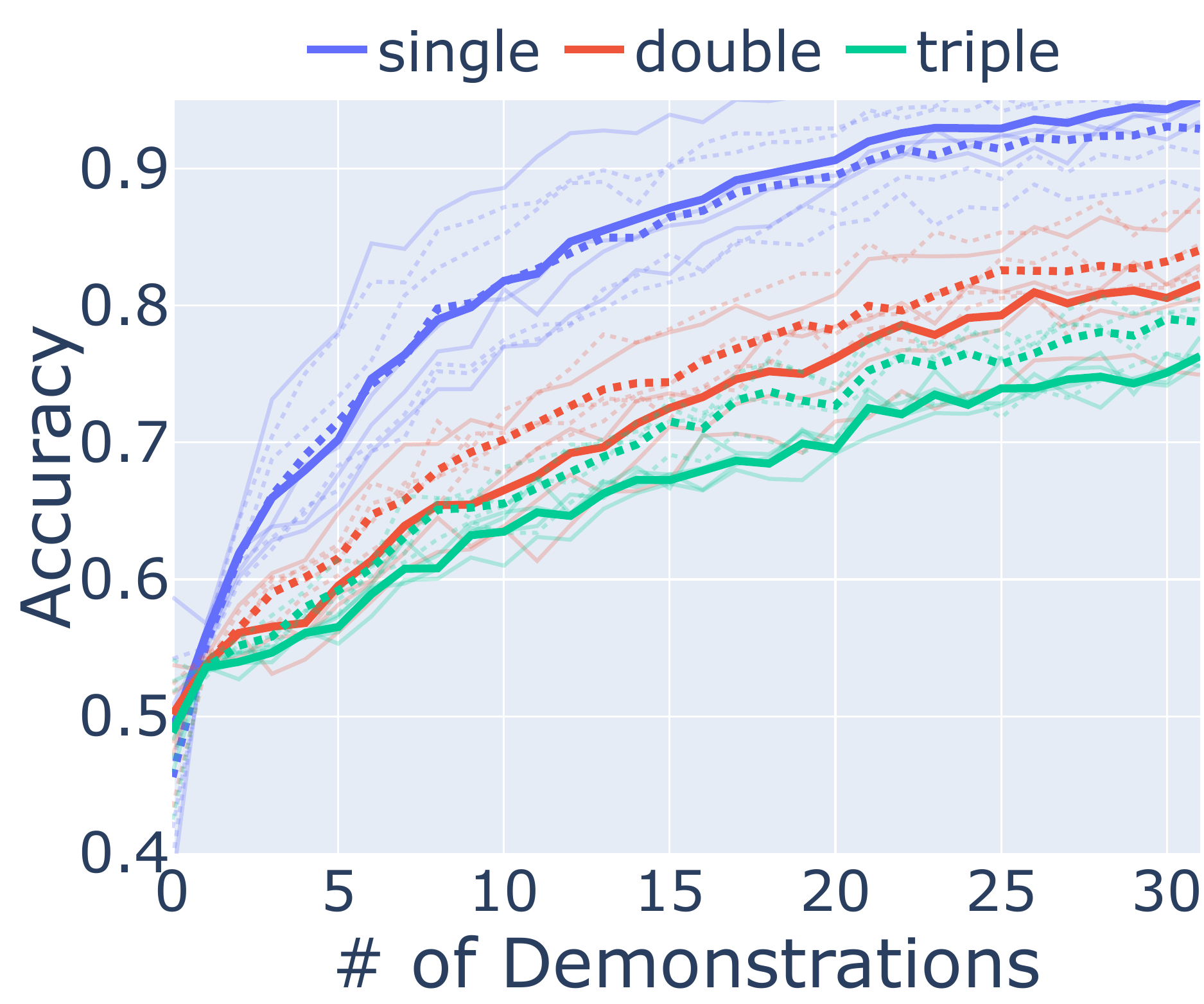}
        \caption{balanced, length = $4$, $y_1, y_2$}
        \label{fig:icl-balanced-y12-4}
    \end{subfigure}
    \begin{subfigure}{0.32\textwidth}
        \centering
        \includegraphics[width=\textwidth]{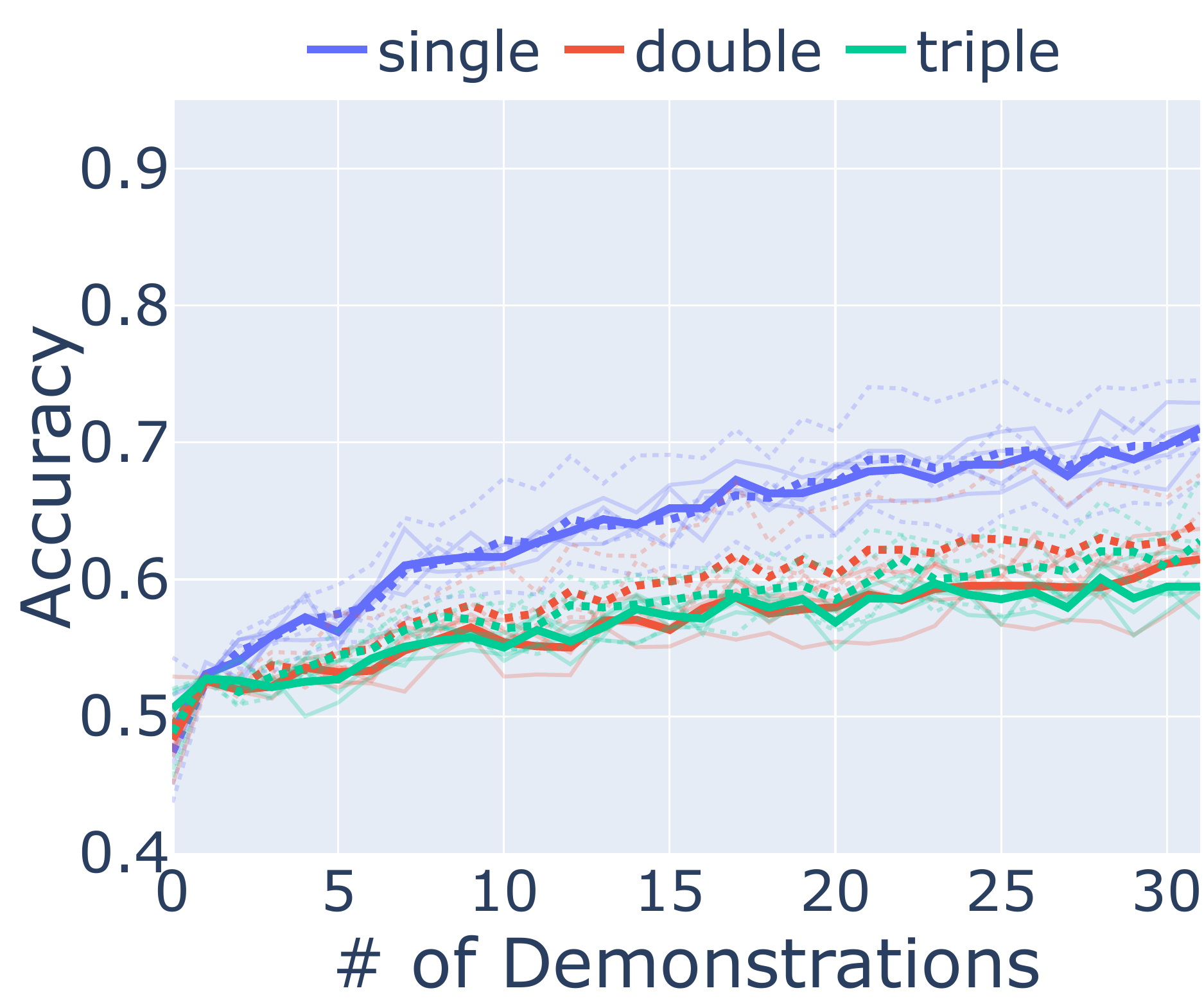}
        \caption{branching, length = $4$, $y_1, y_2$}
        \label{fig:icl-branching-y12-4}
    \end{subfigure}
    \begin{subfigure}{0.32\textwidth}
        \centering
        \includegraphics[width=\textwidth]{figures/branching-y12-3.pdf}
        \caption{branching, length = $3$, $y_1, y_2$}
        \label{fig:icl-branching-y12-3}
    \end{subfigure}

    \begin{subfigure}{0.32\textwidth}
        \centering
        \includegraphics[width=\textwidth]{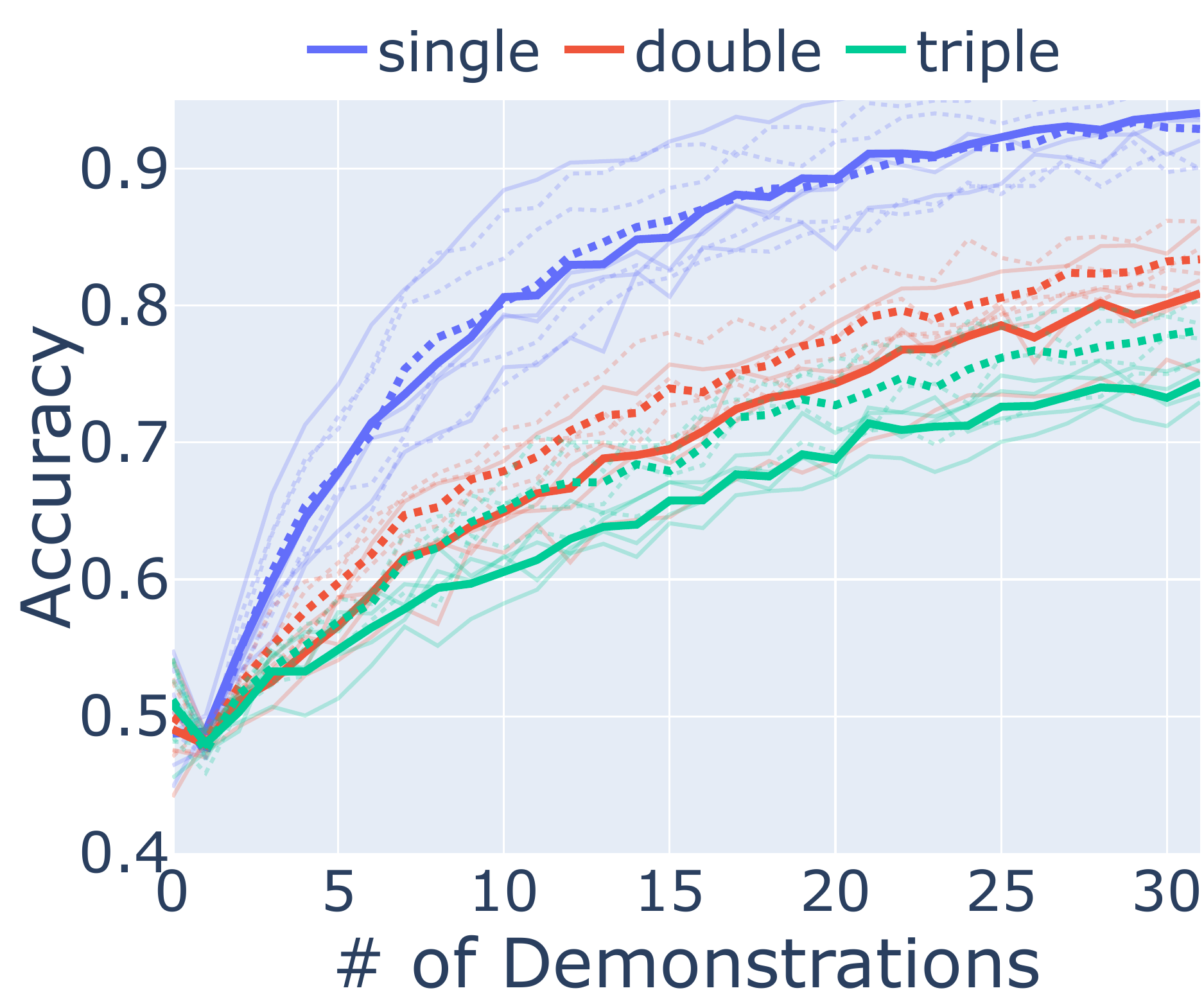}
        \caption{balanced, length = $3$, $y_3, y_4$}
        \label{fig:icl-balanced-y34-4}
    \end{subfigure}
    \begin{subfigure}{0.32\textwidth}
        \centering
        \includegraphics[width=\textwidth]{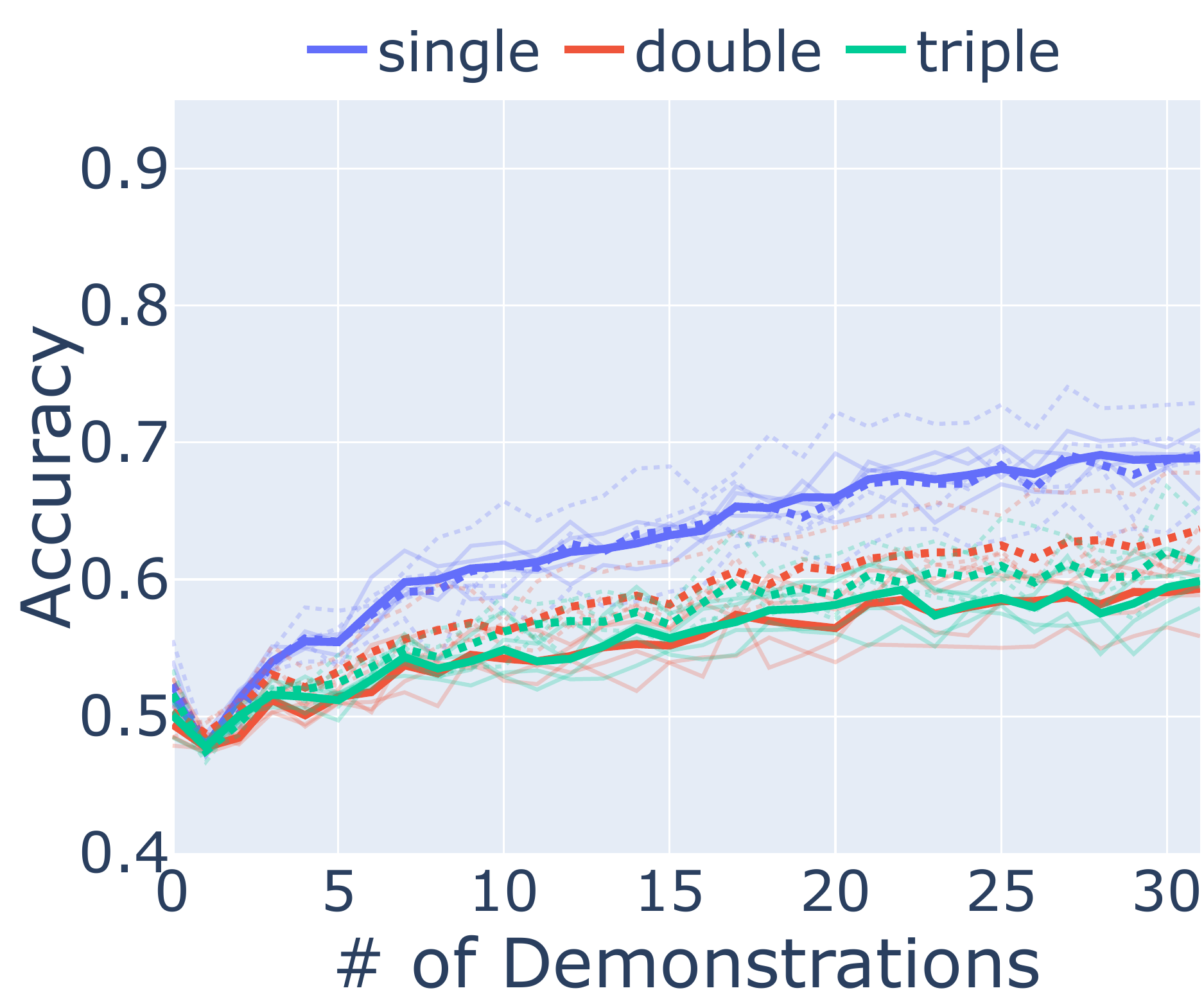}
        \caption{branching, length = $4$, $y_3, y_4$}
        \label{fig:icl-branching-y34-4}
    \end{subfigure}
    \begin{subfigure}{0.32\textwidth}
        \centering
        \includegraphics[width=\textwidth]{figures/branching-y34-3.pdf}
        \caption{branching, length = $3$, $y_3, y_4$}
        \label{fig:icl-branching-y34-3}
    \end{subfigure}
    
    \caption{
    In-context learning accuracy with Calcutec when using different verbalizers ($y_1, y_2$ or $y_3, y_4$) and input lengths (3 or 4).
    The dotted lines represent the performance of \textit{unseen combinations} described in \S{\ref{sec:inspecting-dist-shifts}}, while the different colors represent the number of formulas each class ($v_+$ or $v_-$) is associated to.
    The main lines represent the average accuracy of 5 tasks.
    We plot the performance of each task in lighter colors. 
    }
    \label{fig:icl-full}
\end{figure*}

\begin{figure*}[]
    \centering
    \begin{subfigure}{0.32\textwidth}
        \centering
        \includegraphics[width=\textwidth]{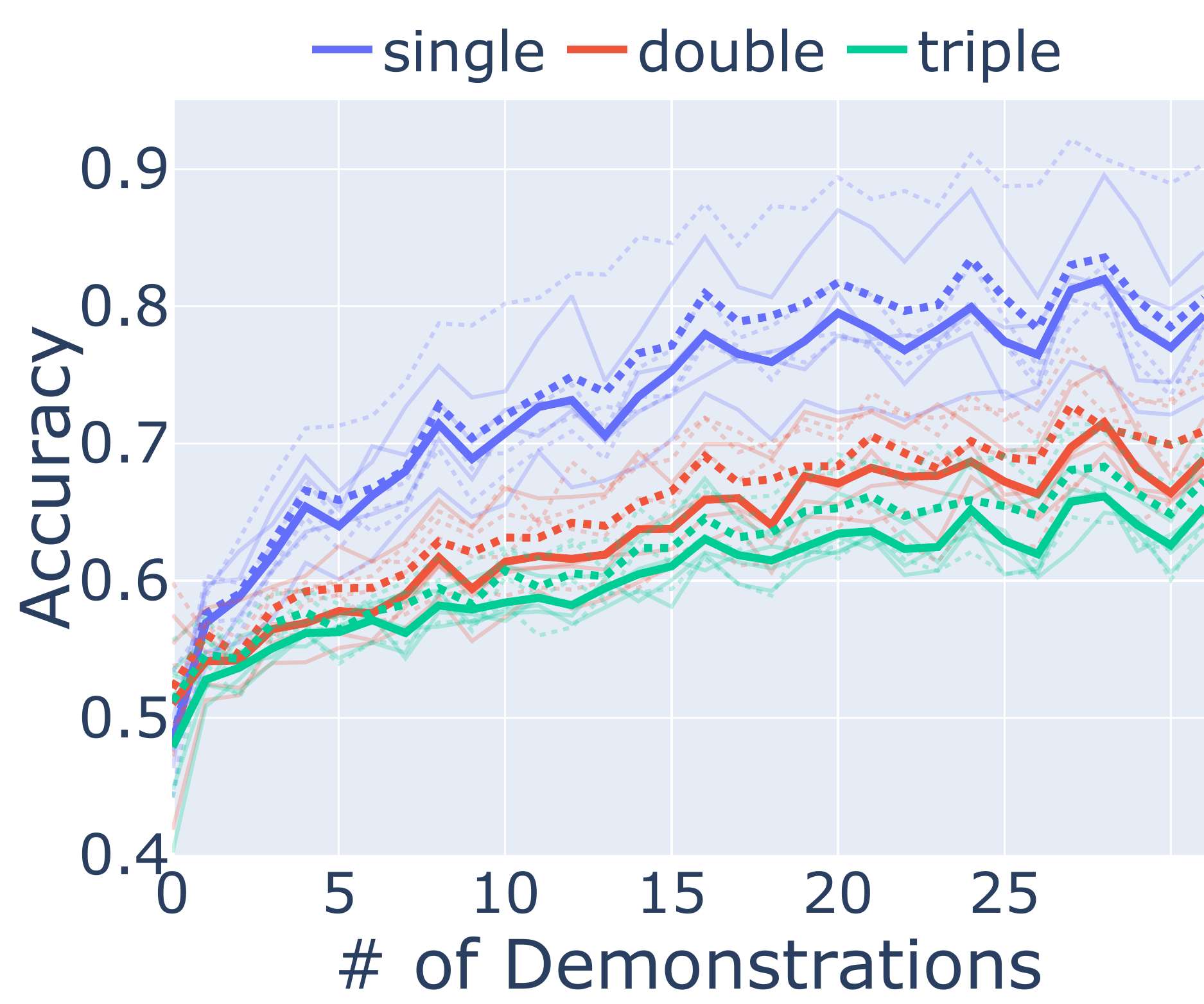}
        \caption{balanced, length = $4$, $r_1, r_2$}
    \end{subfigure}
    \begin{subfigure}{0.32\textwidth}
        \centering
        \includegraphics[width=\textwidth]{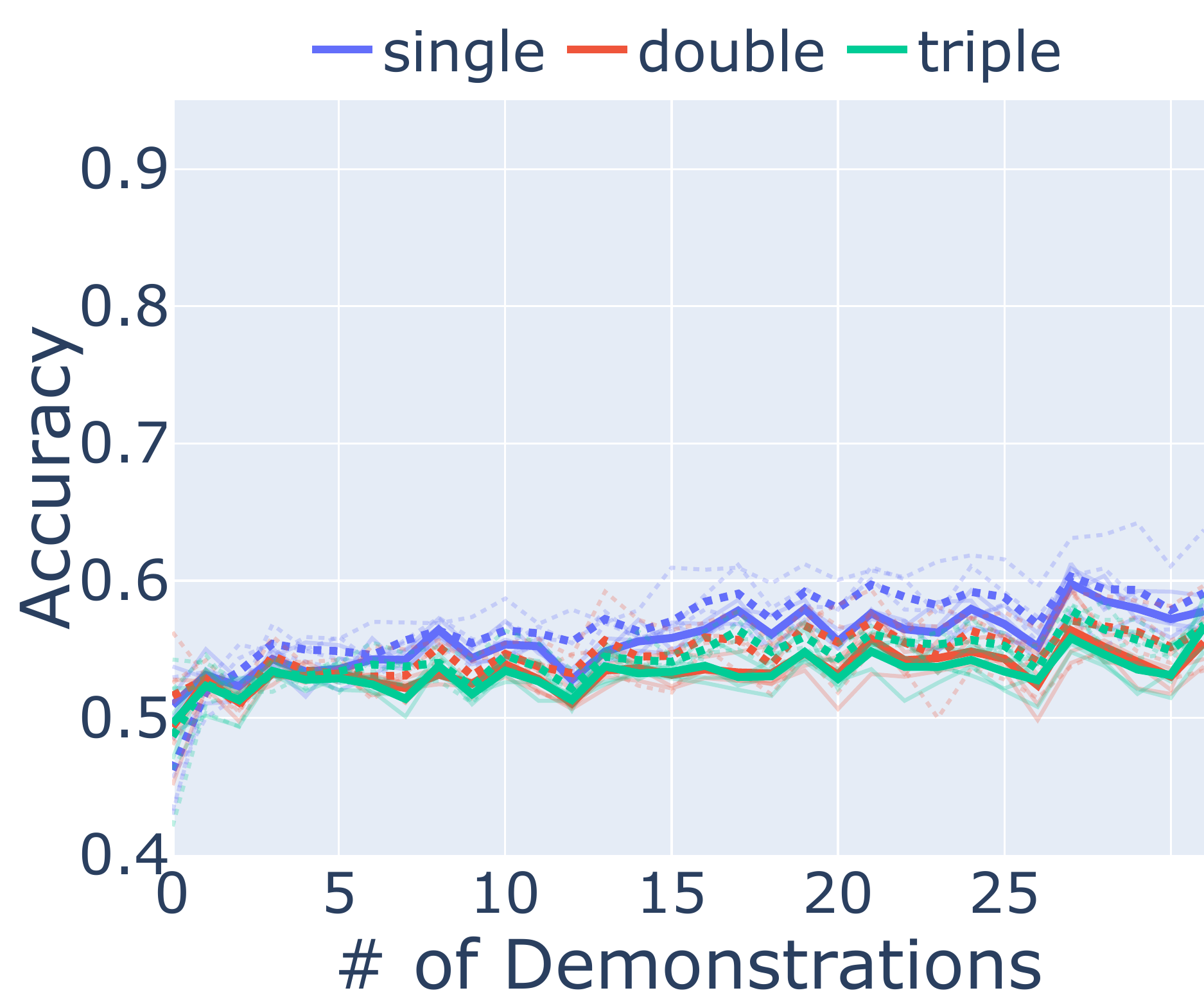}
        \caption{branching, length = $4$, $r_1, r_2$}
    \end{subfigure}
    \begin{subfigure}{0.32\textwidth}
        \centering
        \includegraphics[width=\textwidth]{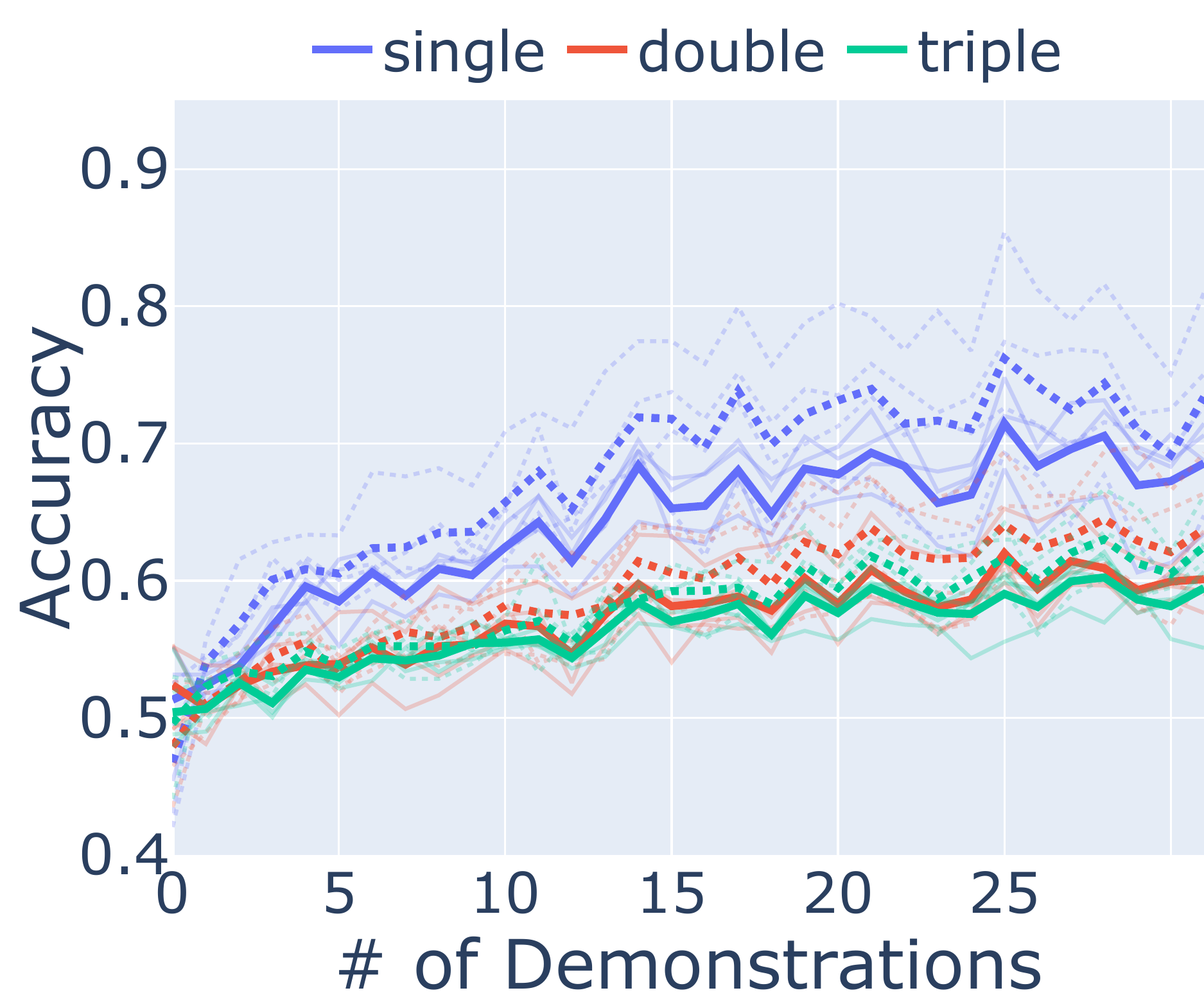}
        \caption{branching, length = $3$, $r_1, r_2$}
    \end{subfigure}

    \begin{subfigure}{0.32\textwidth}
        \centering
        \includegraphics[width=\textwidth]{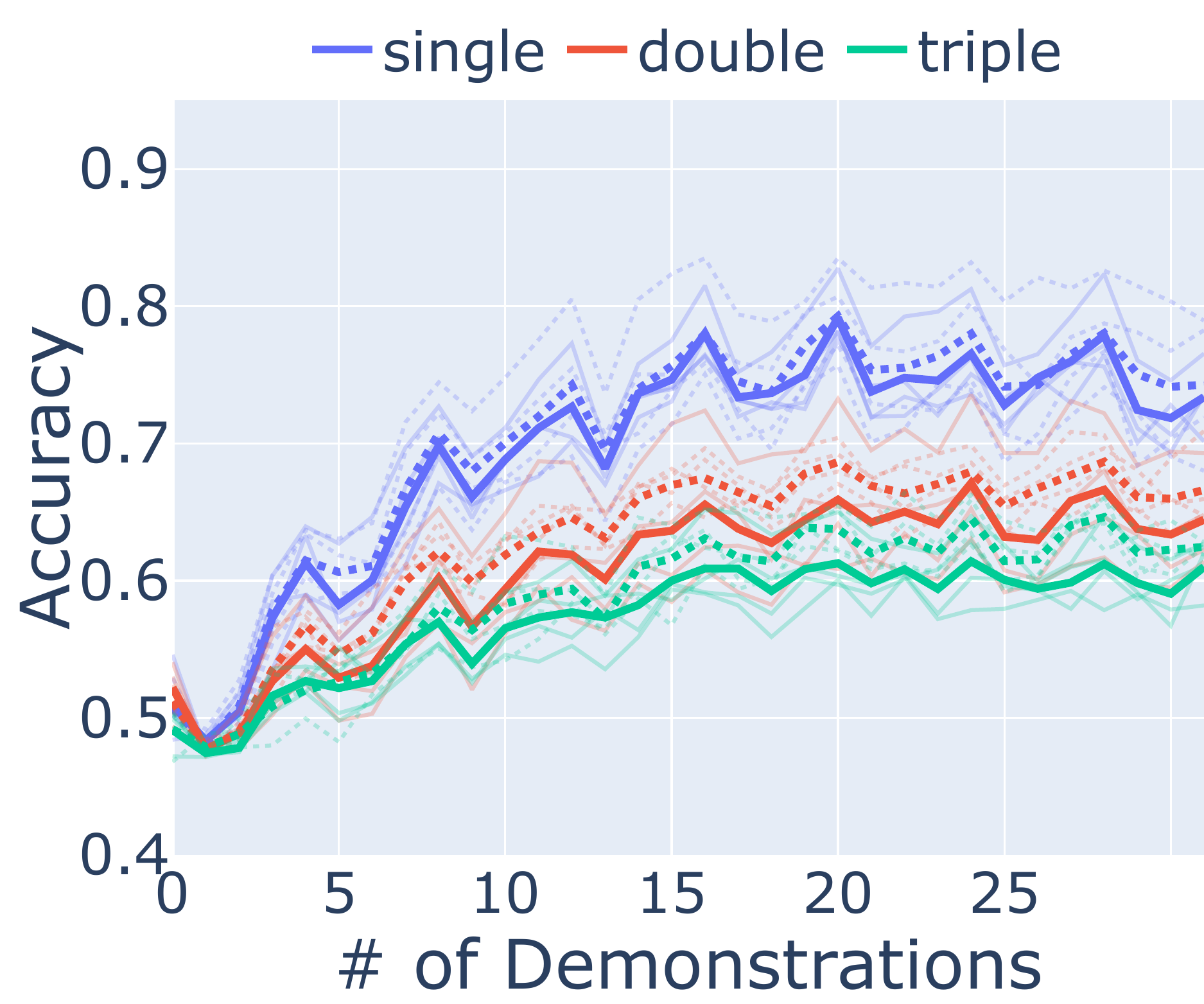}
        \caption{balanced, length = $4$, $r_3, r_4$}
    \end{subfigure}
    \begin{subfigure}{0.32\textwidth}
        \centering
        \includegraphics[width=\textwidth]{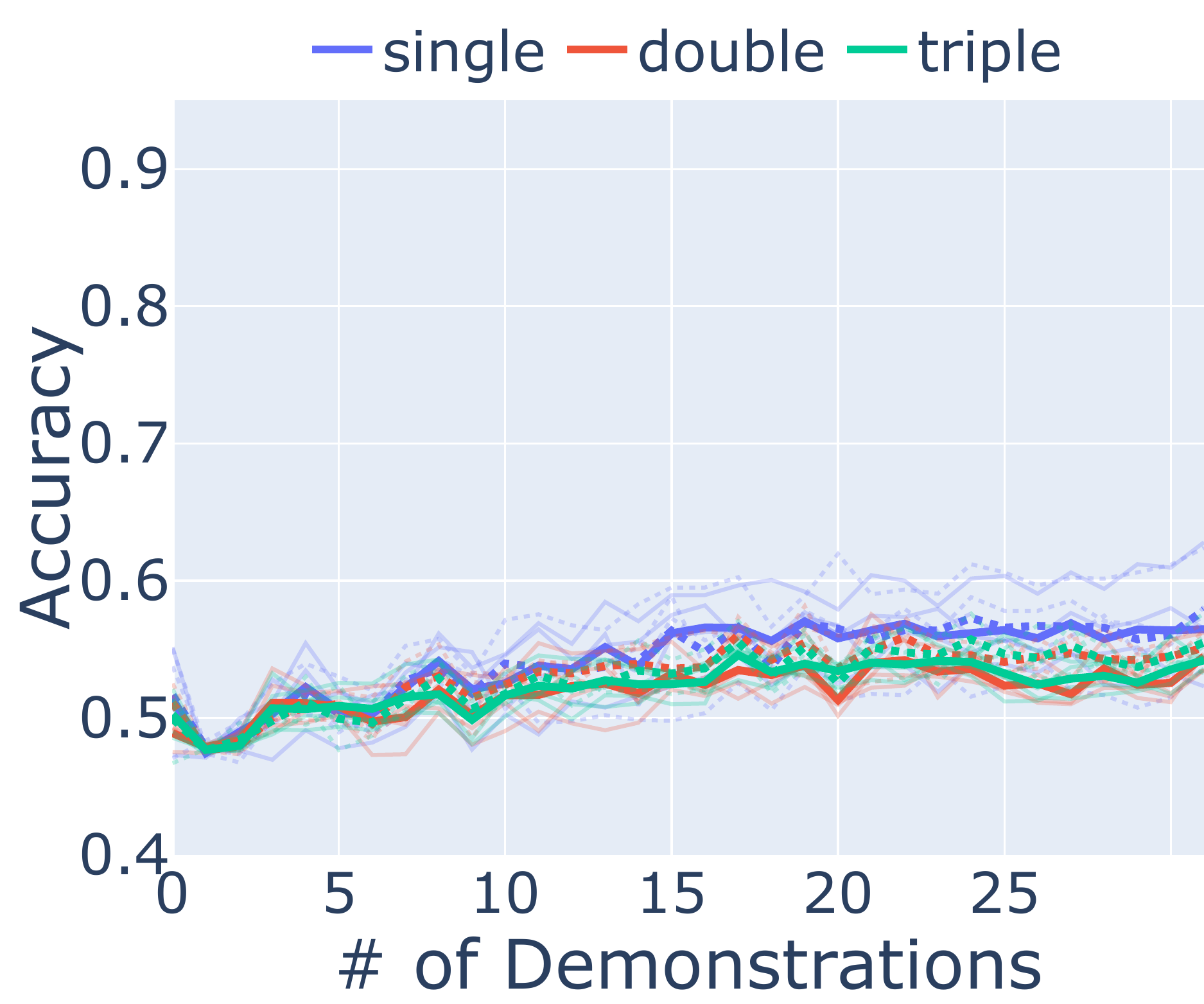}
        \caption{branching, length = $4$, $r_3, r_4$}
    \end{subfigure}
    \begin{subfigure}{0.32\textwidth}
        \centering
        \includegraphics[width=\textwidth]{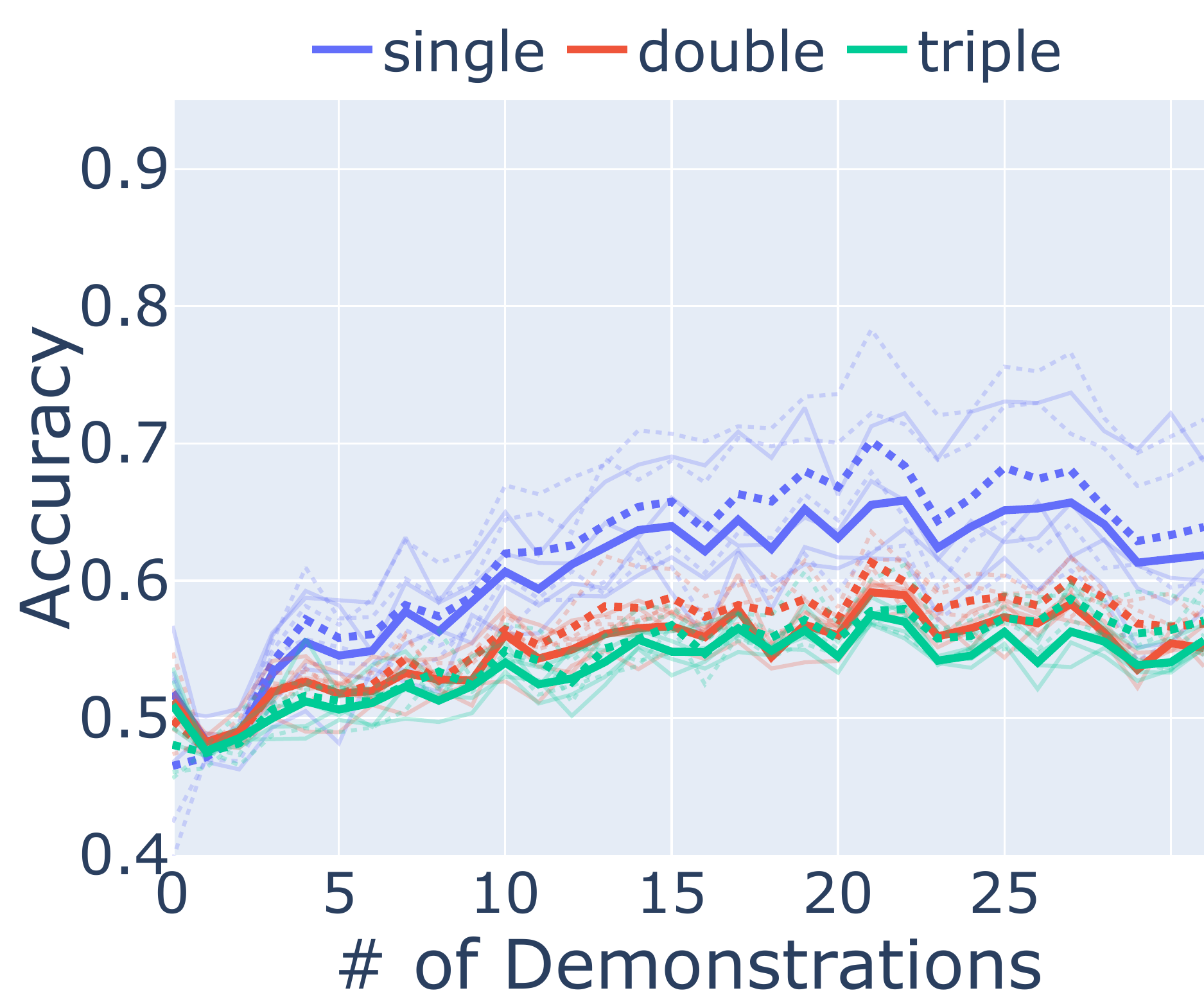}
        \caption{branching, length = $3$, $r_3, r_4$}
    \end{subfigure}
    
    \caption{
    In-context learning accuracy with Calcutec when no steps are dropped ($p_{skip} = 0$).
    }
    \label{fig:icl-no-drop}
\end{figure*}

\begin{figure*}
    \centering
    \begin{subfigure}{0.32\textwidth}
        \centering
        \includegraphics[width=\textwidth]{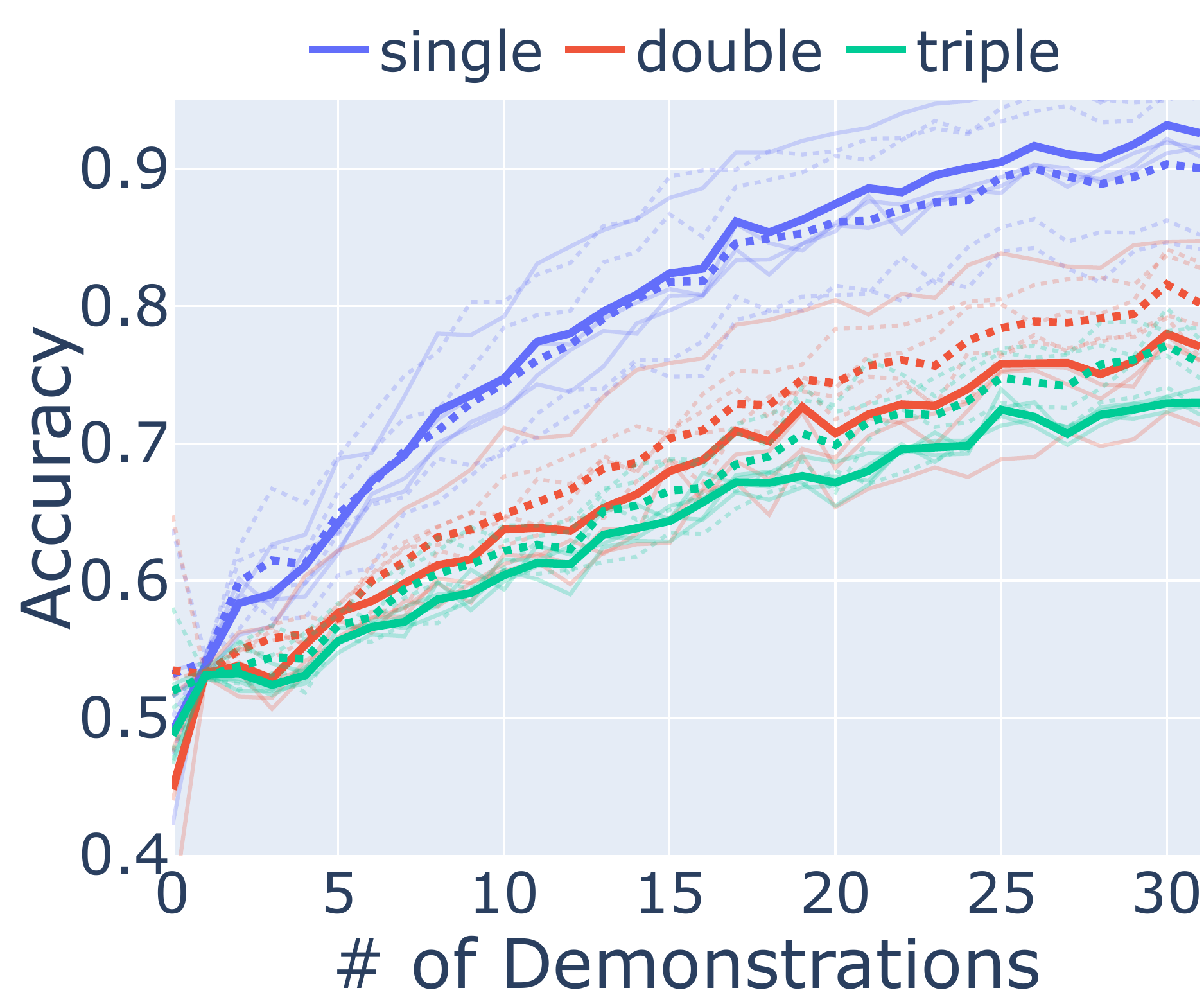}
        \caption{balanced, length = $4$, $r_1, r_2$}
    \end{subfigure}
    \begin{subfigure}{0.32\textwidth}
        \centering
        \includegraphics[width=\textwidth]{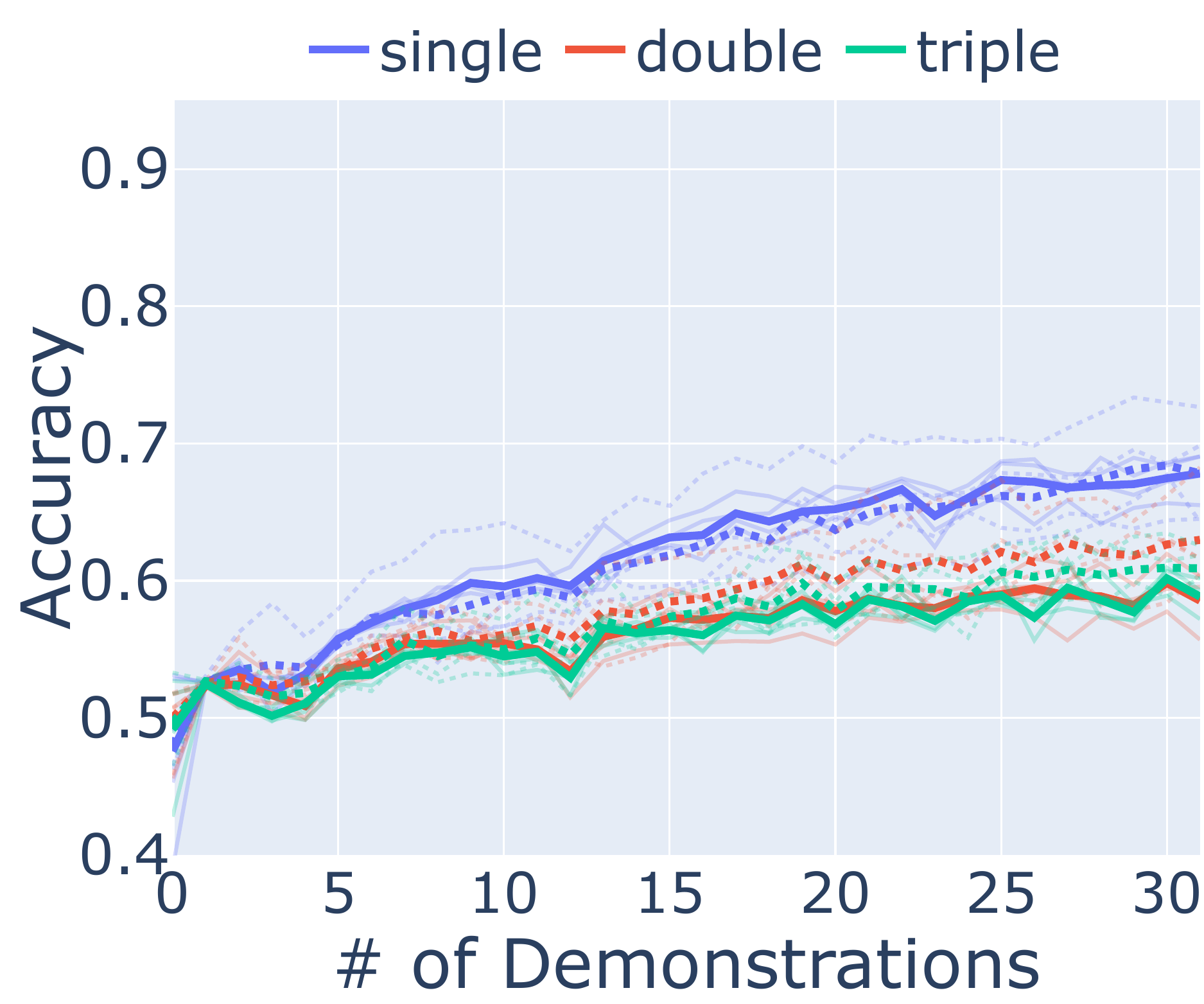}
        \caption{branching, length = $4$, $r_1, r_2$}
    \end{subfigure}
    \begin{subfigure}{0.32\textwidth}
        \centering
        \includegraphics[width=\textwidth]{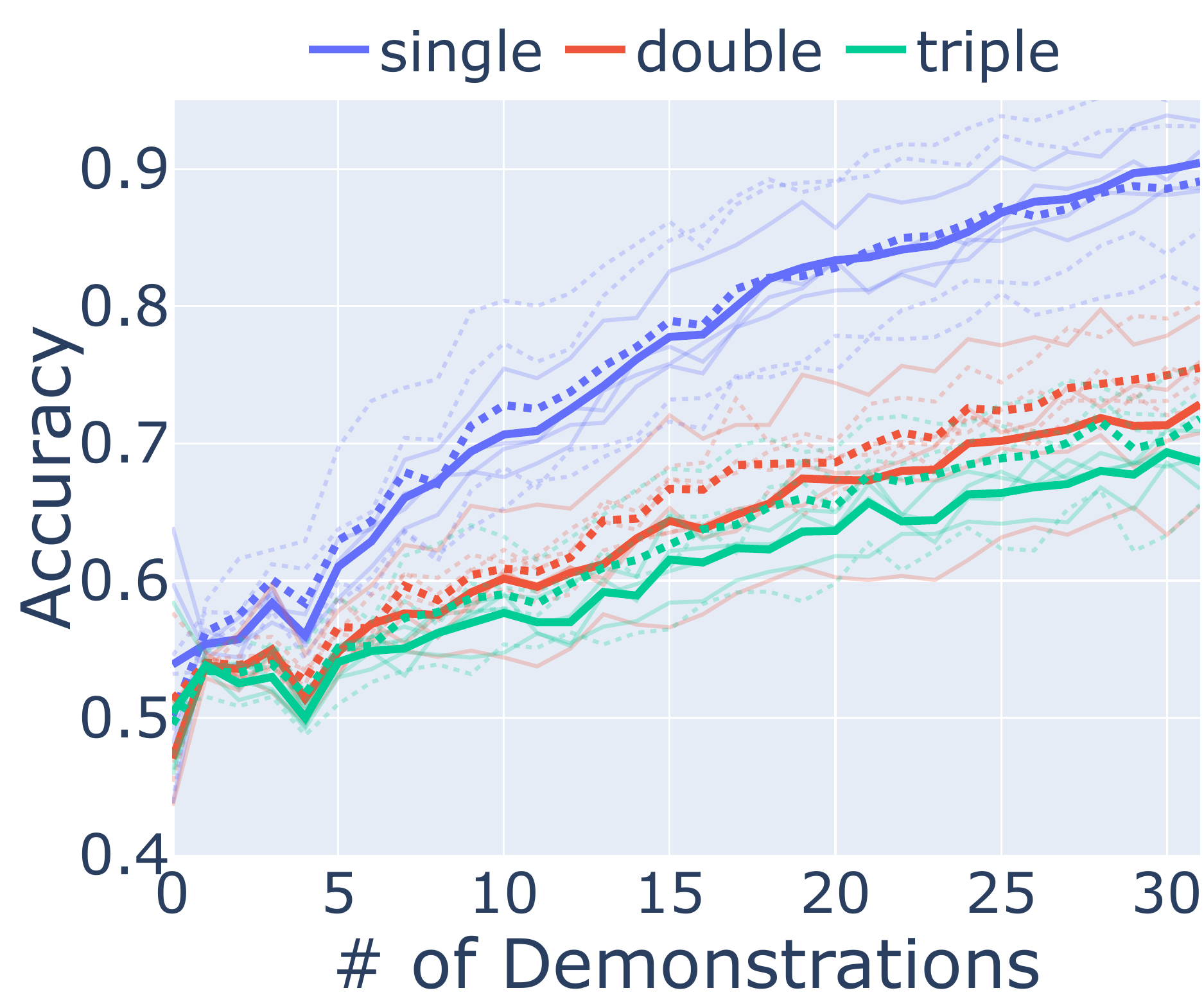}
        \caption{branching, length = $3$, $r_1, r_2$}
    \end{subfigure}

    \begin{subfigure}{0.32\textwidth}
        \centering
        \includegraphics[width=\textwidth]{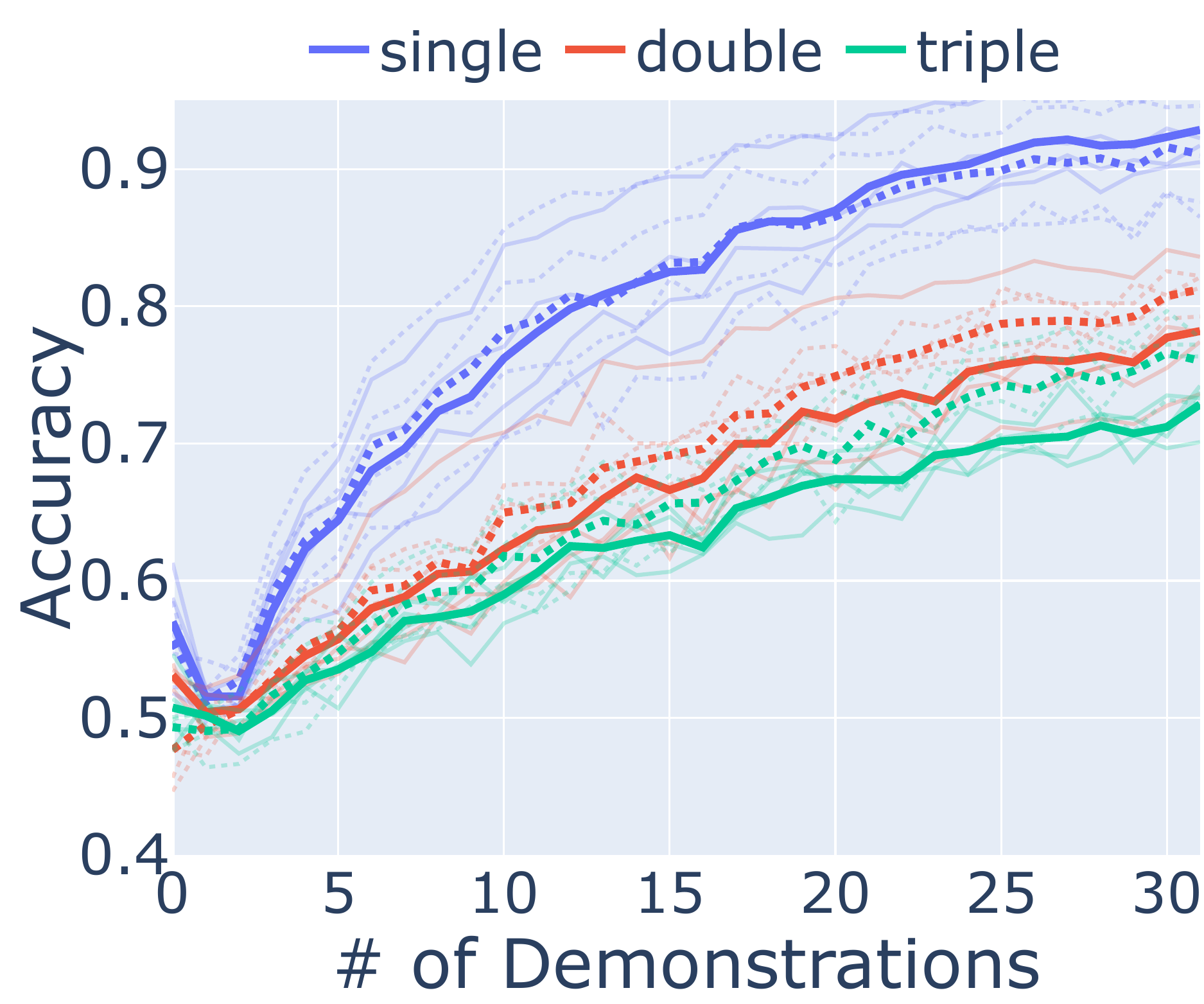}
        \caption{balanced, length = $4$, $r_3, r_4$}
    \end{subfigure}
    \begin{subfigure}{0.32\textwidth}
        \centering
        \includegraphics[width=\textwidth]{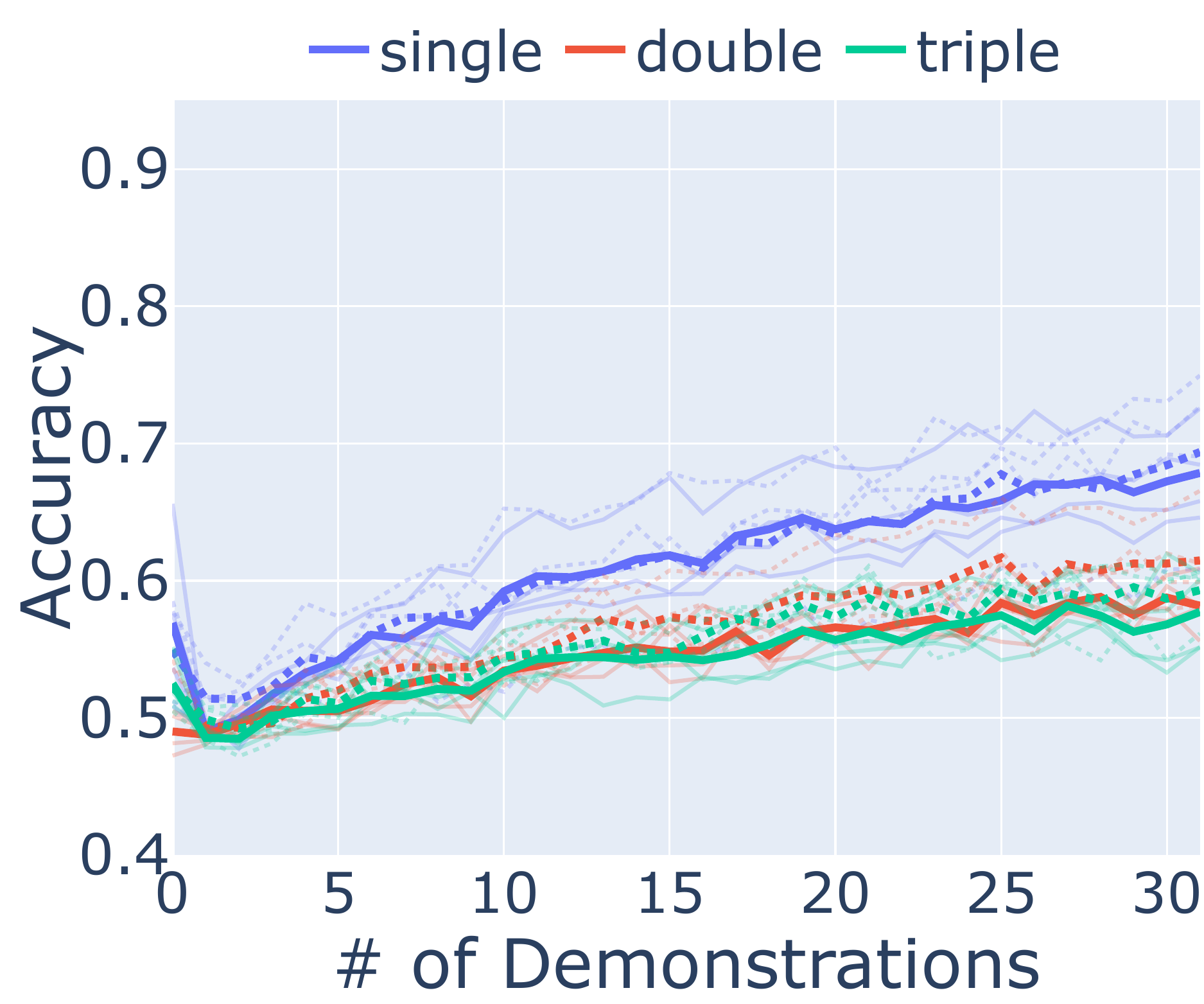}
        \caption{branching, length = $4$, $r_3, r_4$}
    \end{subfigure}
    \begin{subfigure}{0.32\textwidth}
        \centering
        \includegraphics[width=\textwidth]{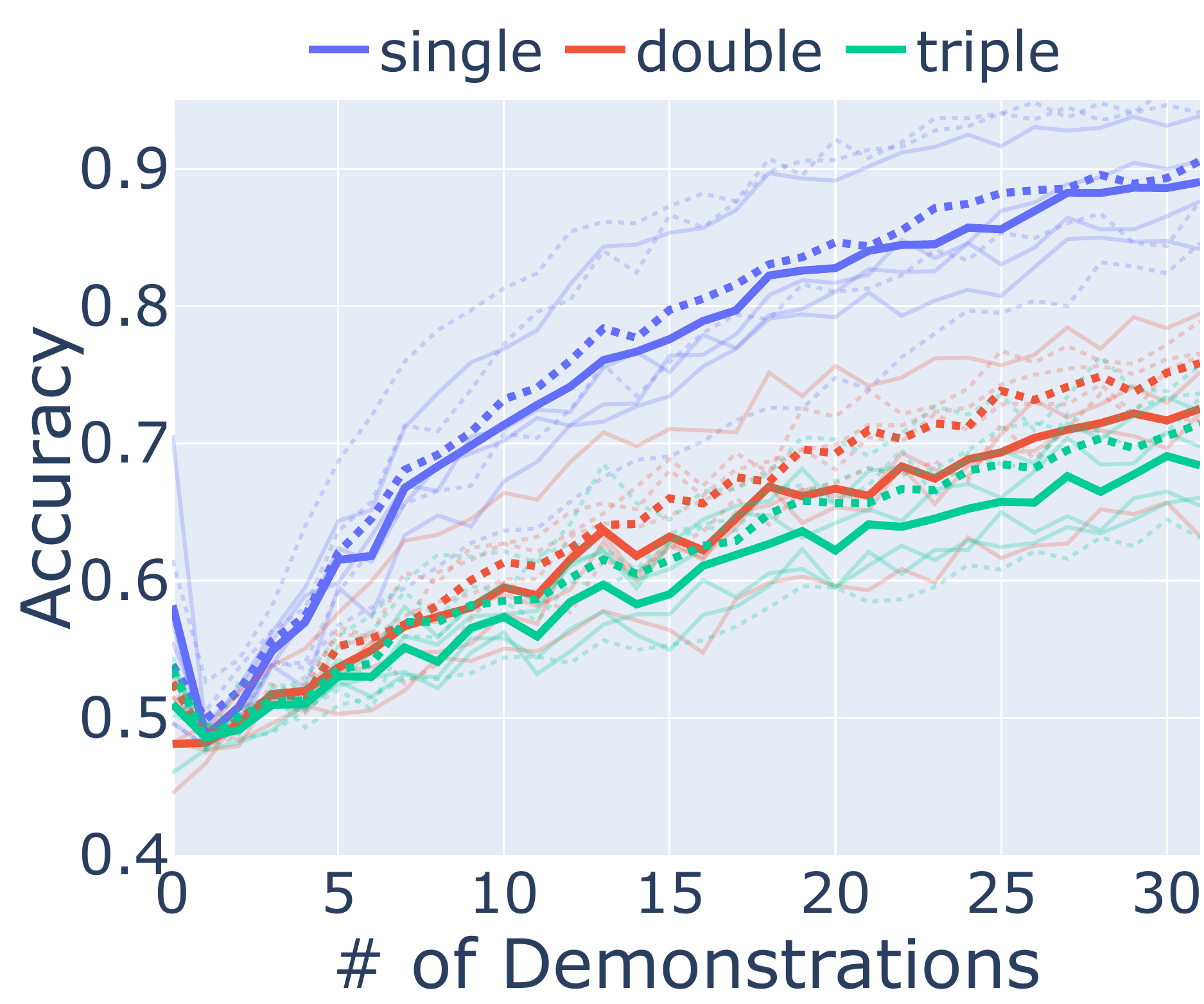}
        \caption{branching, length = $3$, $r_3, r_4$}
    \end{subfigure}
    
    \caption{
    In-context learning accuracy with Calcutec without rewriting the first step to include contain the premise of the proof.
    }
    \label{fig:icl-no-fuse}
\end{figure*}

\begin{figure*}
    \centering
    \begin{subfigure}{0.32\textwidth}
        \centering
        \includegraphics[width=\textwidth]{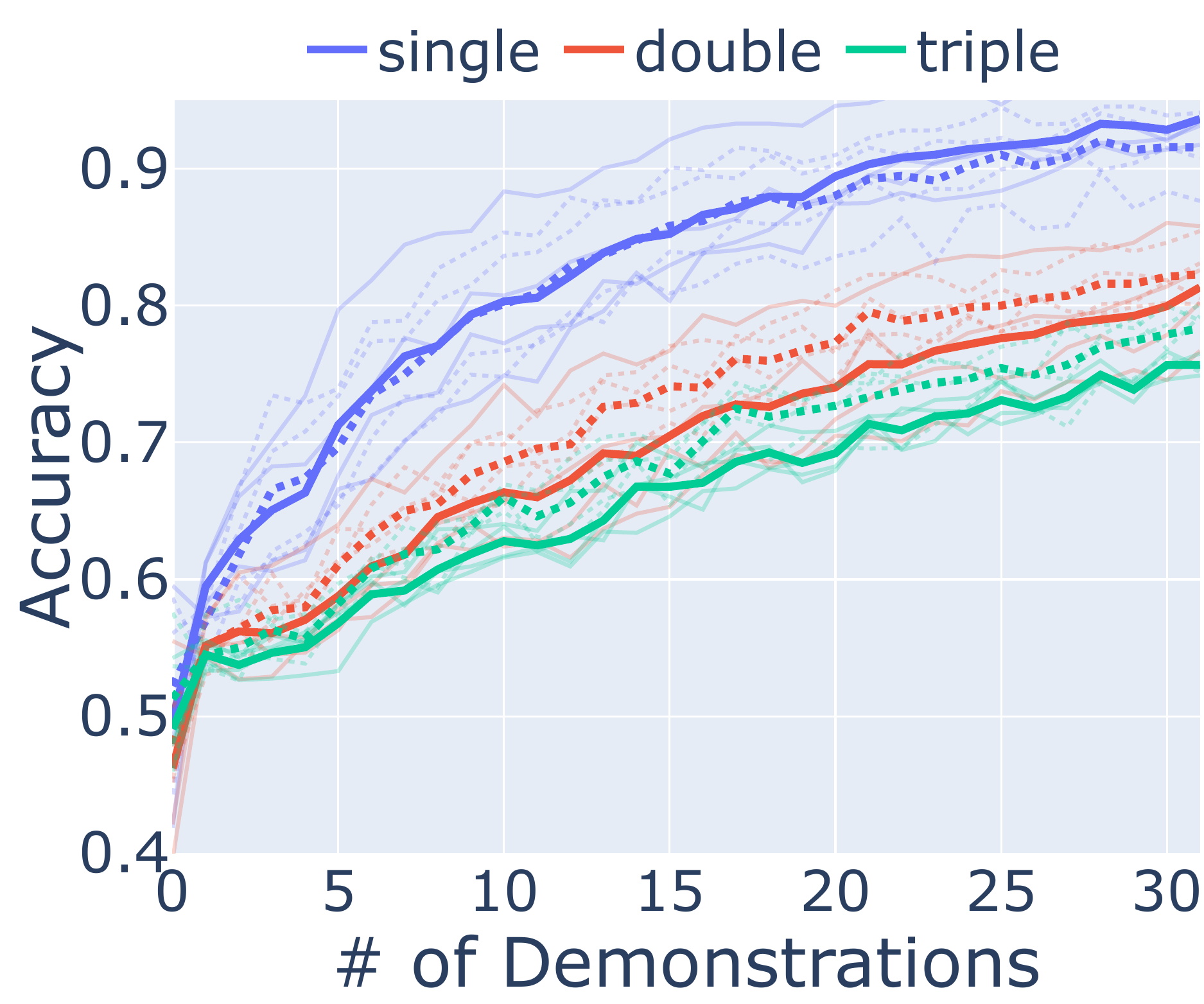}
        \caption{balanced, $r_1, r_2$, 25000 examples}
    \end{subfigure}
    \begin{subfigure}{0.32\textwidth}
        \centering
        \includegraphics[width=\textwidth]{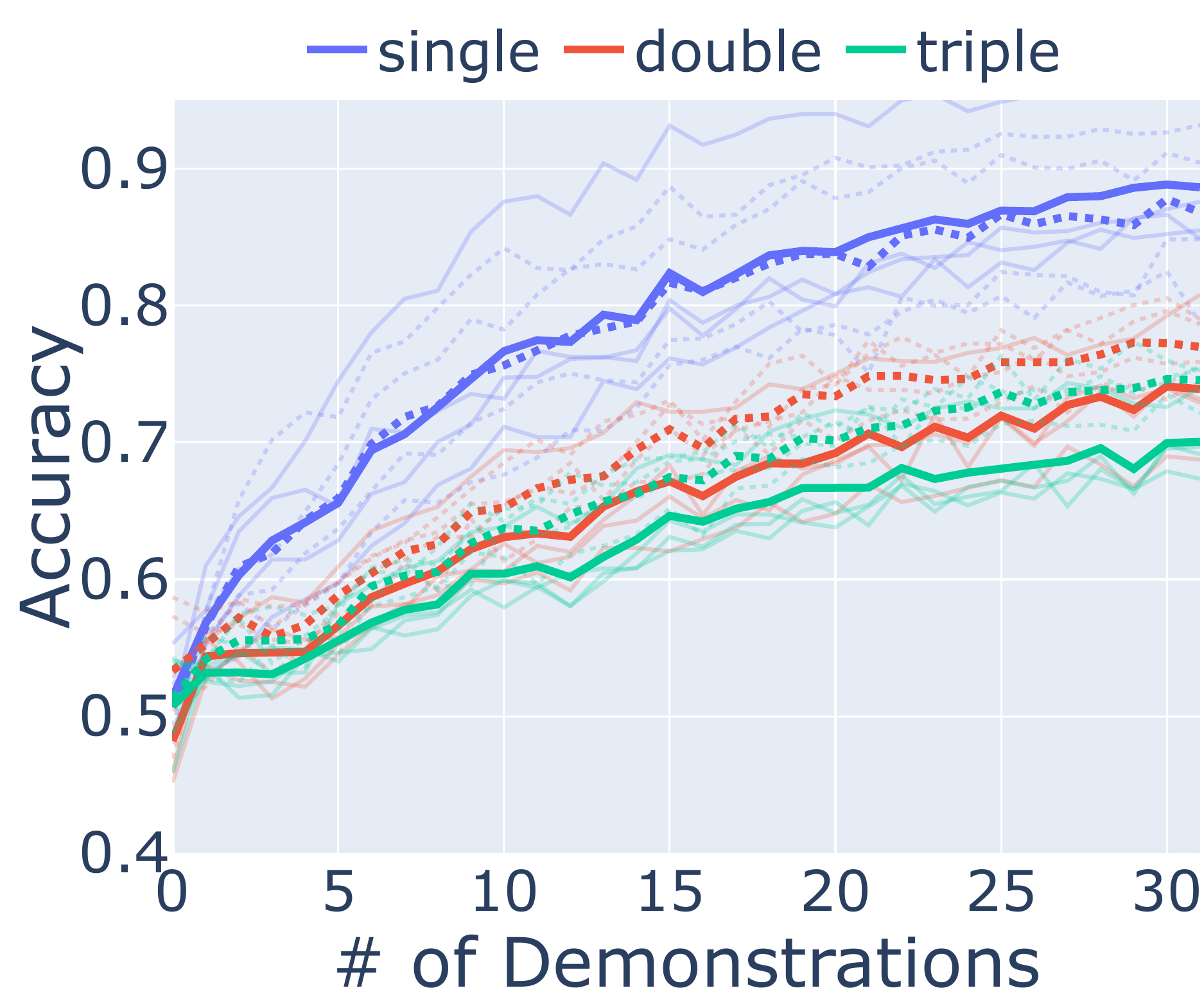}
        \caption{balanced, $r_1, r_2$, 10000 examples}
    \end{subfigure}
    \begin{subfigure}{0.32\textwidth}
        \centering
        \includegraphics[width=\textwidth]{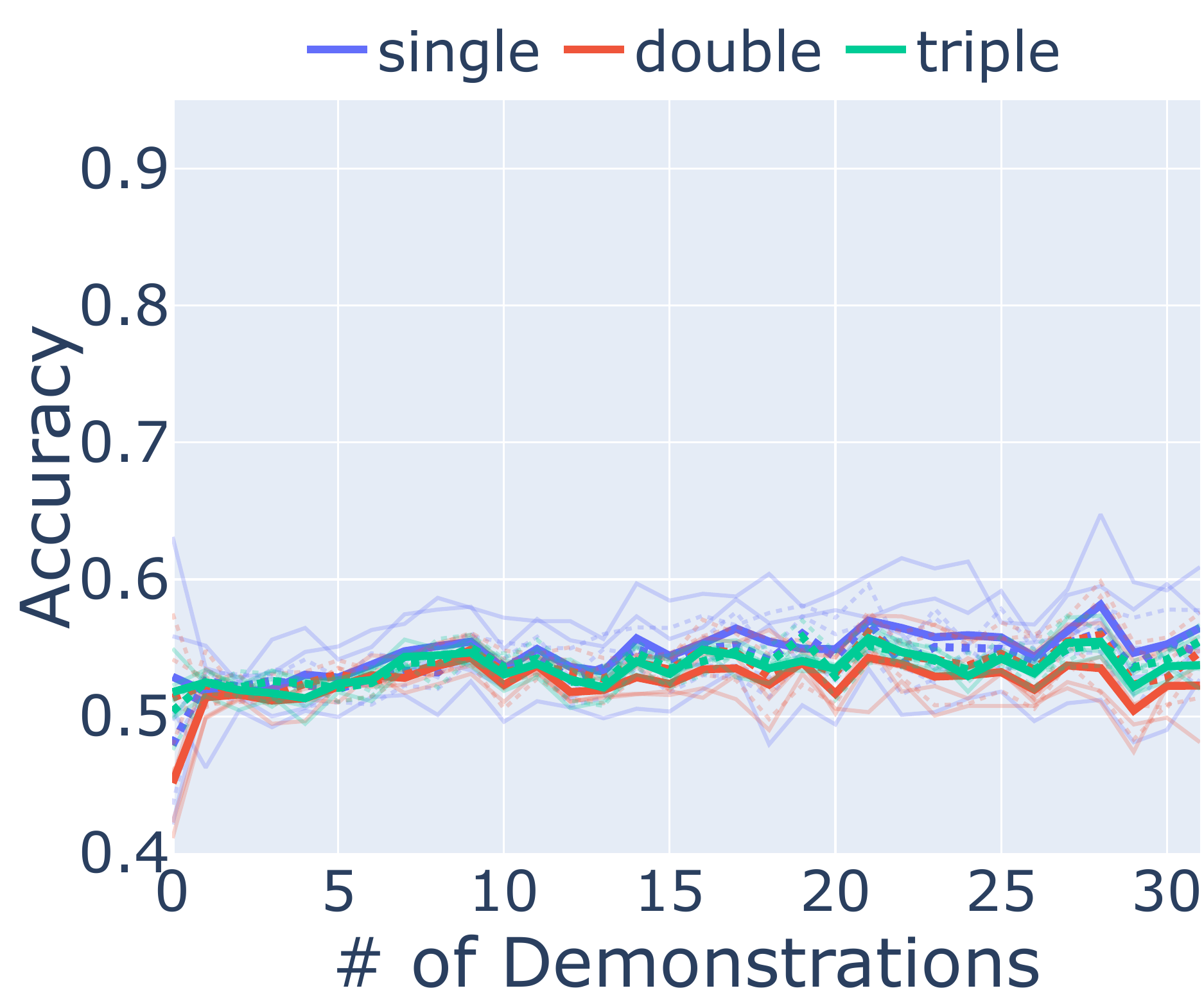}
        \caption{balanced, $r_1, r_2$, 5000 examples}
    \end{subfigure}

    \begin{subfigure}{0.32\textwidth}
        \centering
        \includegraphics[width=\textwidth]{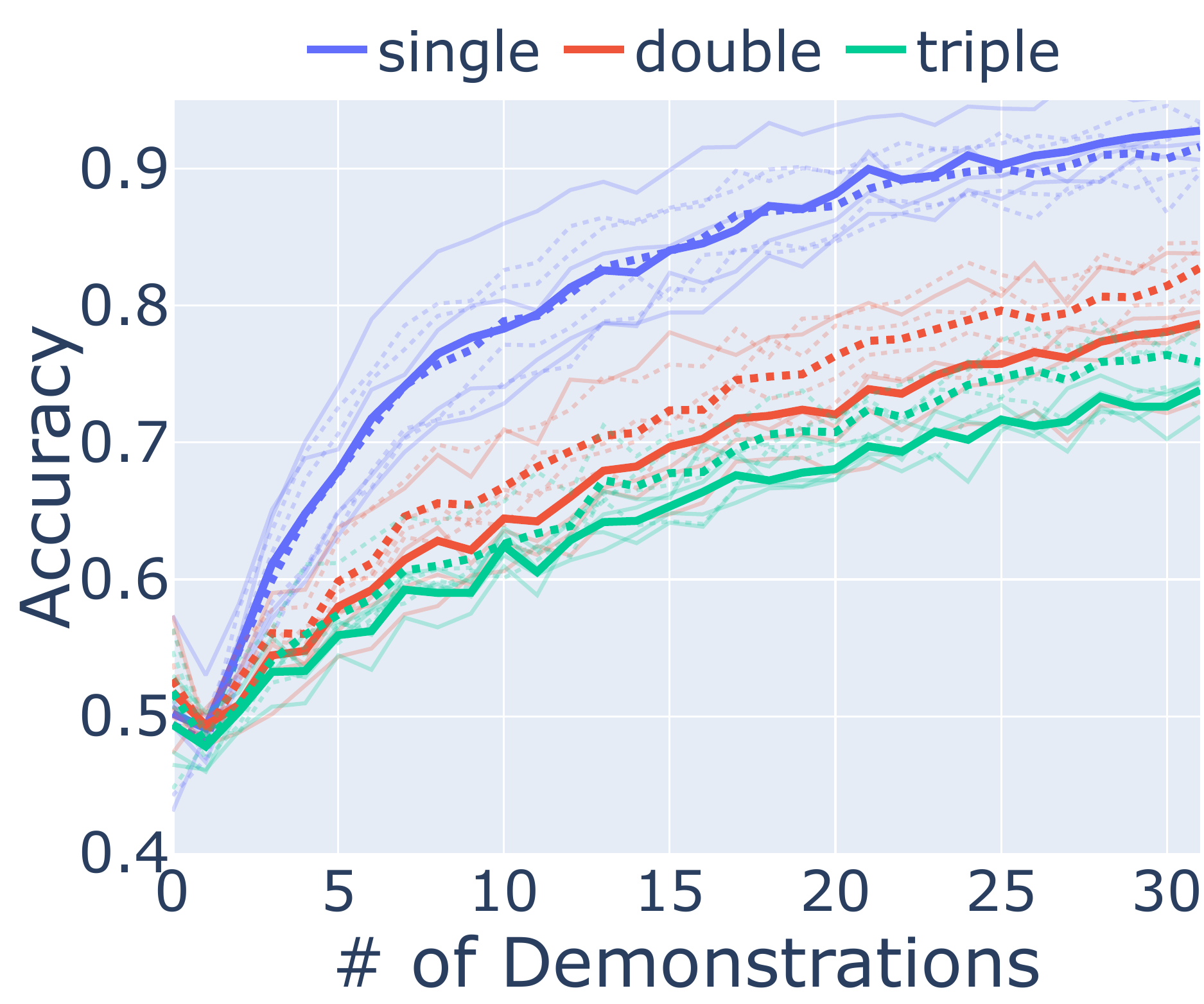}
        \caption{balanced, $r_3, r_4$, 25000 examples}
    \end{subfigure}
    \begin{subfigure}{0.32\textwidth}
        \centering
        \includegraphics[width=\textwidth]{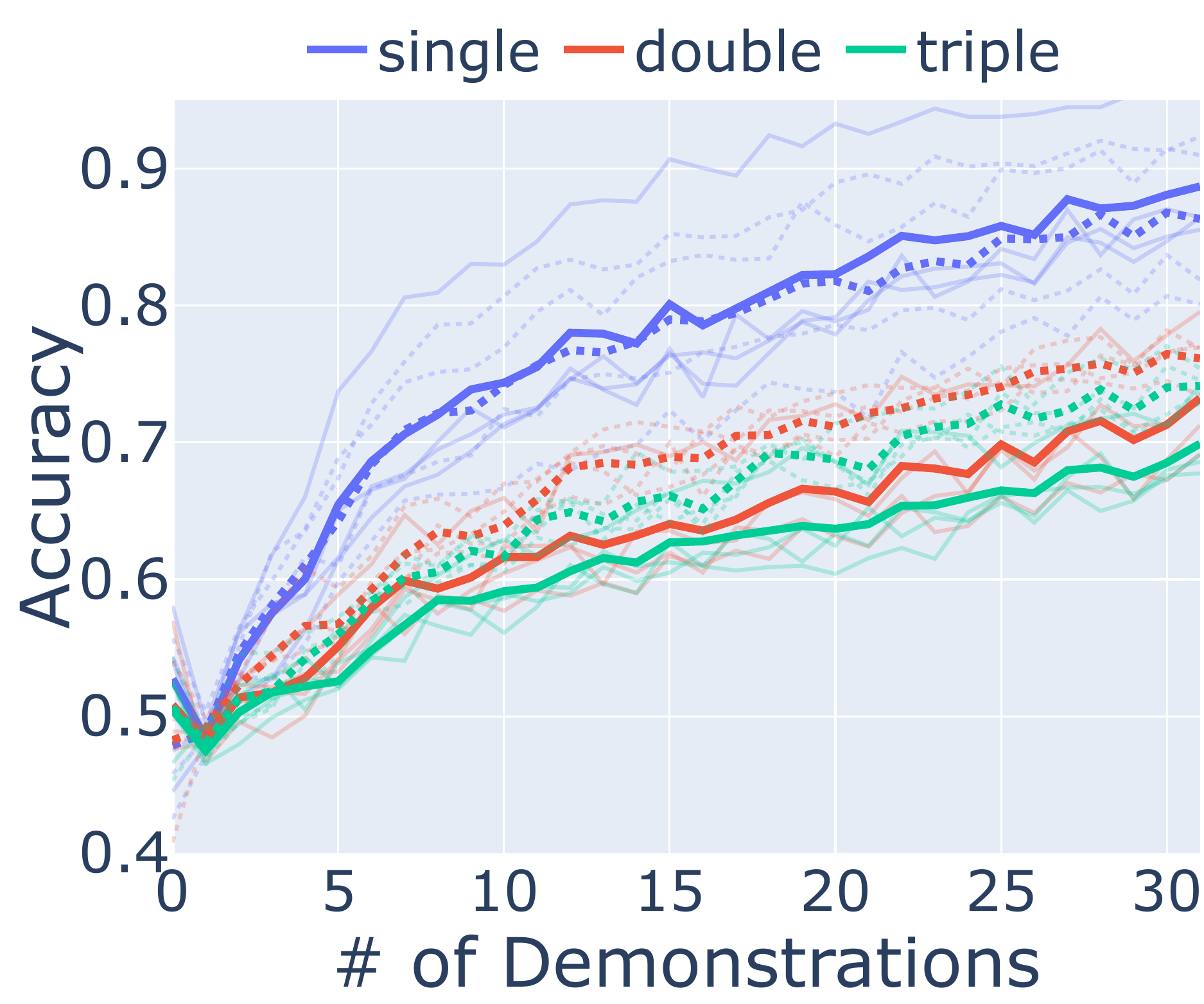}
        \caption{balanced, $r_3, r_4$, 10000 examples}
    \end{subfigure}
    \begin{subfigure}{0.32\textwidth}
        \centering
        \includegraphics[width=\textwidth]{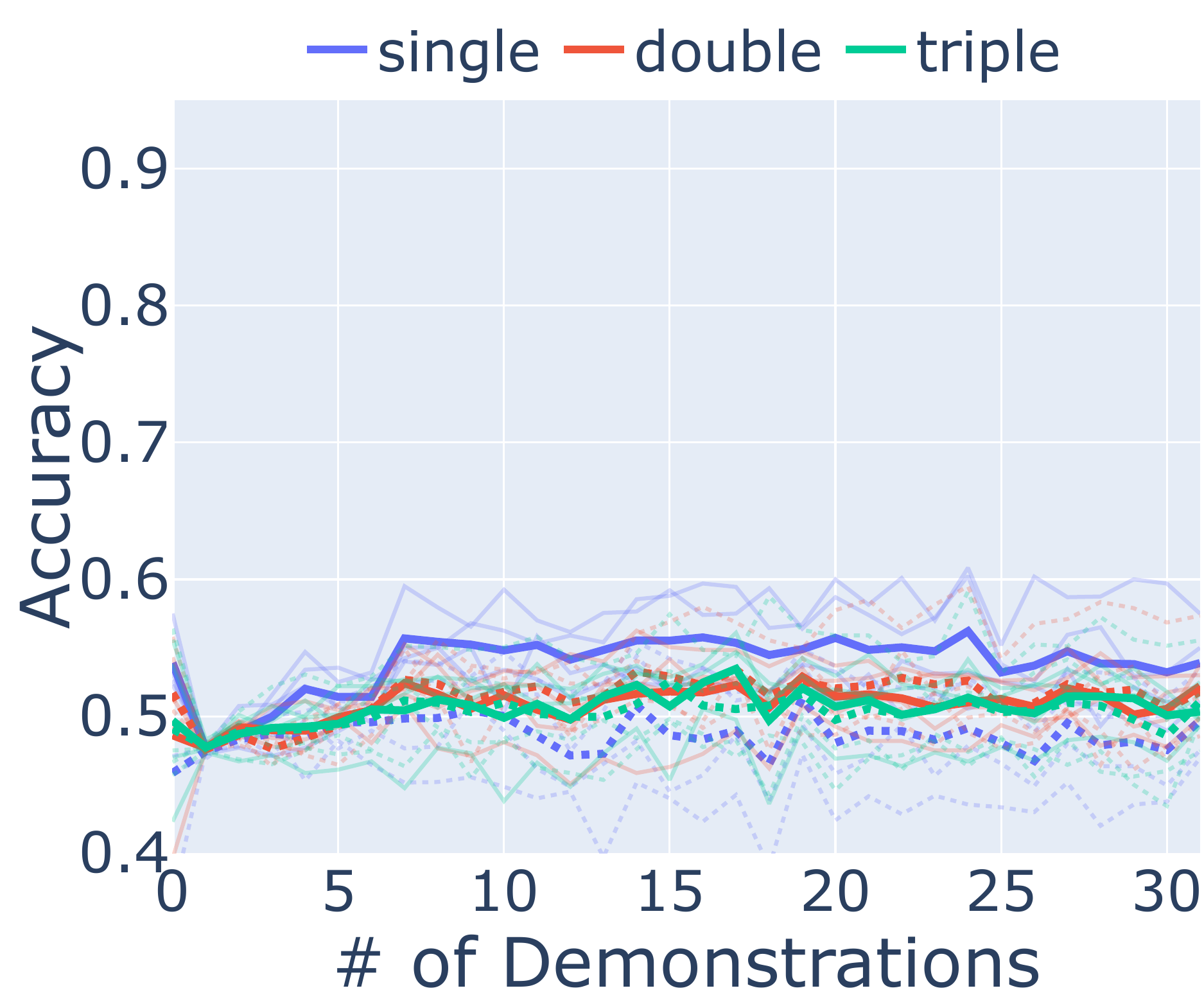}
        \caption{balanced, $r_3, r_4$, 5000 examples}
    \end{subfigure}

    \begin{subfigure}{0.32\textwidth}
        \centering
        \includegraphics[width=\textwidth]{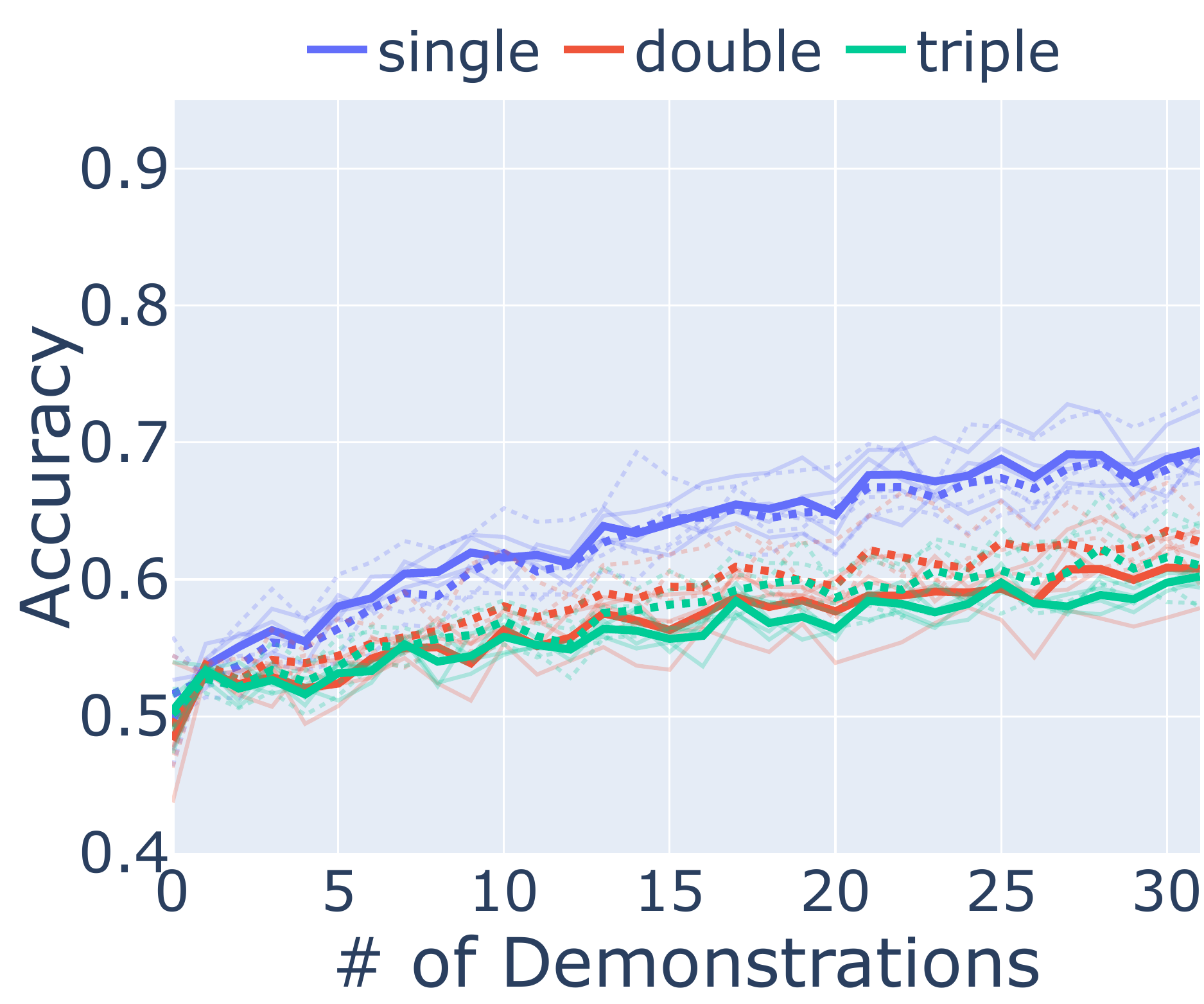}
        \caption{branching, $r_1, r_2$, 25000 examples}
    \end{subfigure}
    \begin{subfigure}{0.32\textwidth}
        \centering
        \includegraphics[width=\textwidth]{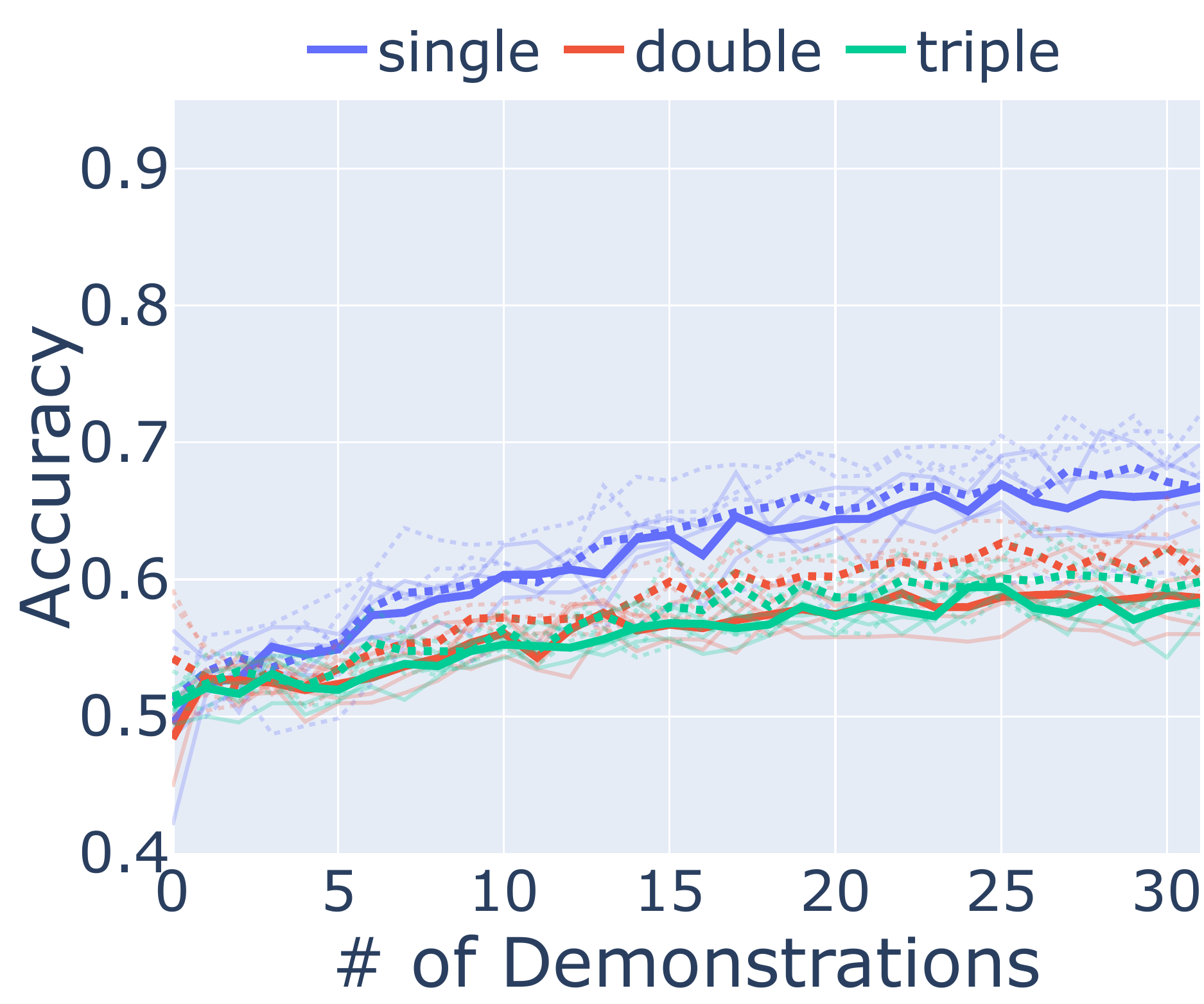}
        \caption{branching, $r_1, r_2$, 10000 examples}
    \end{subfigure}
    \begin{subfigure}{0.32\textwidth}
        \centering
        \includegraphics[width=\textwidth]{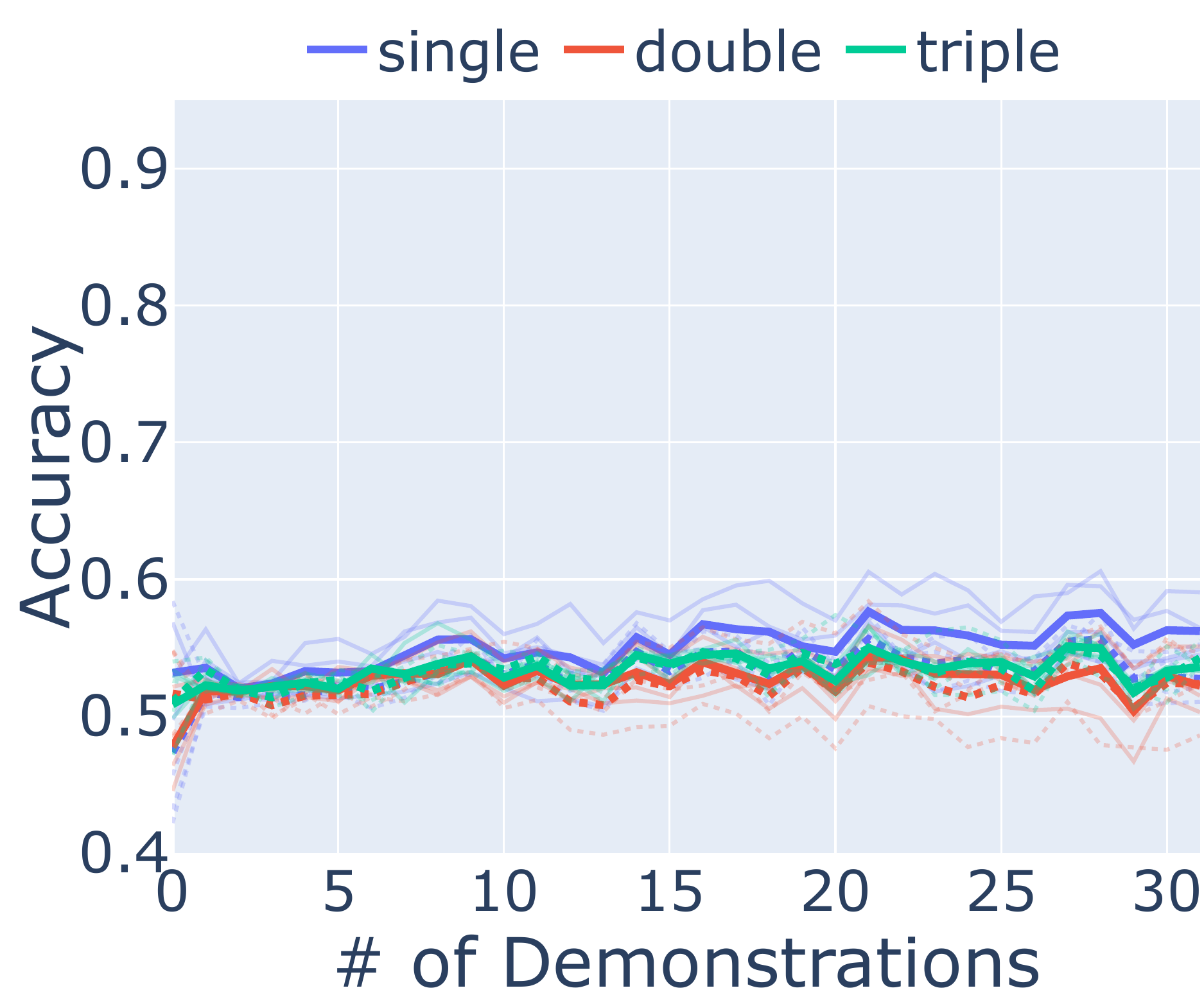}
        \caption{branching, $r_1, r_2$, 5000 examples}
    \end{subfigure}

    \begin{subfigure}{0.32\textwidth}
        \centering
        \includegraphics[width=\textwidth]{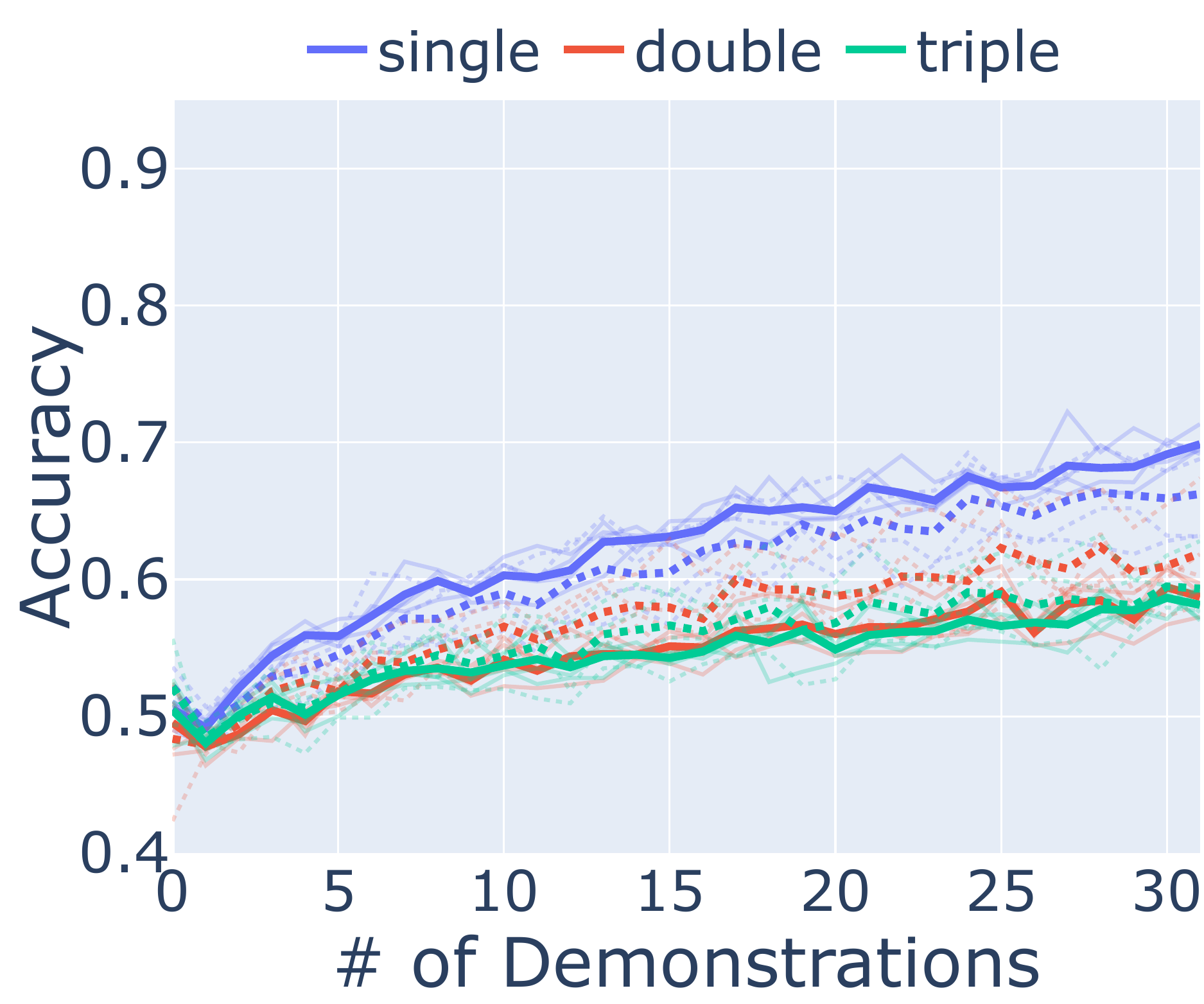}
        \caption{branching, $r_3, r_4$, 25000 examples}
    \end{subfigure}
    \begin{subfigure}{0.32\textwidth}
        \centering
        \includegraphics[width=\textwidth]{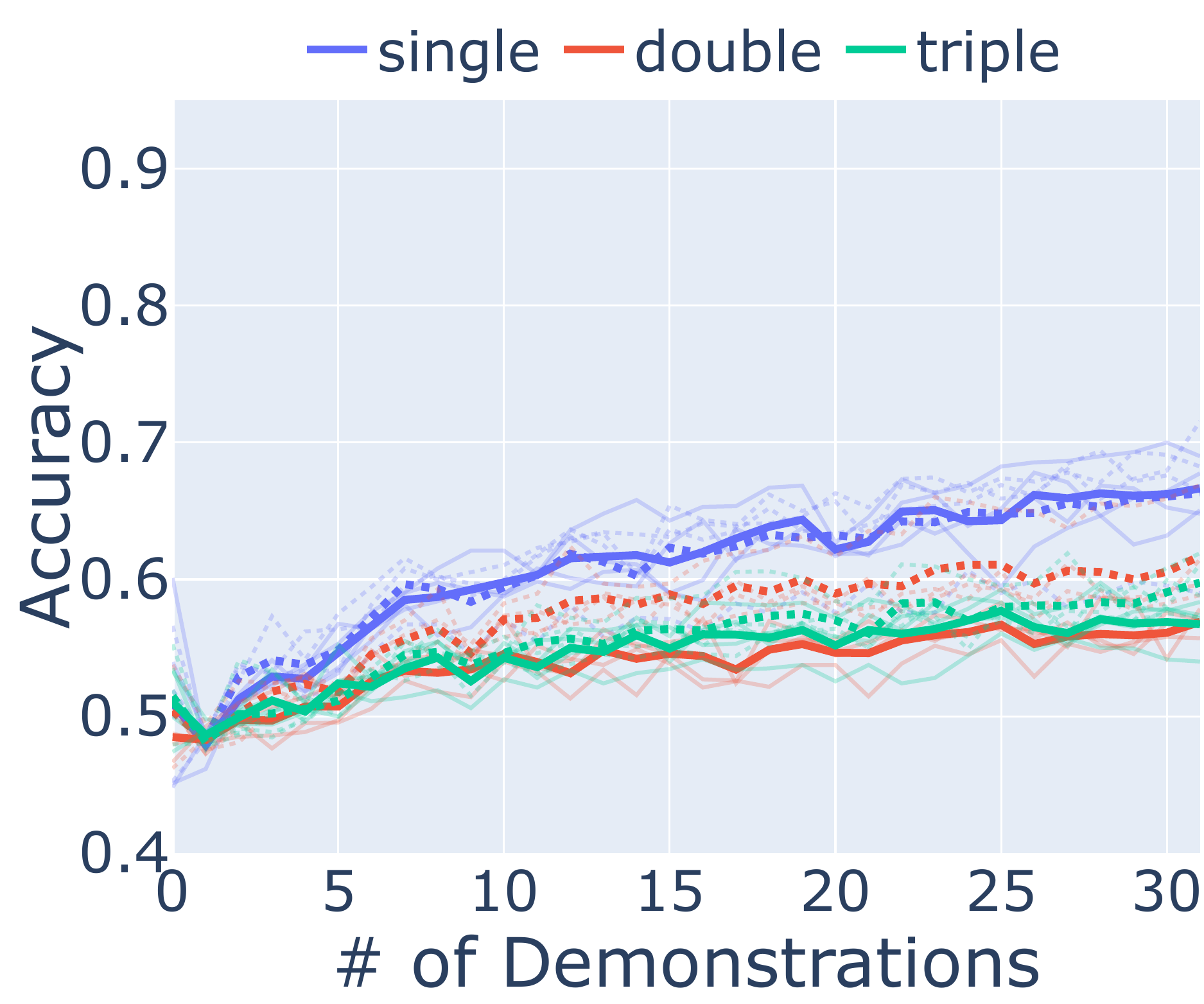}
        \caption{branching, $r_3, r_4$, 10000 examples}
    \end{subfigure}
    \begin{subfigure}{0.32\textwidth}
        \centering
        \includegraphics[width=\textwidth]{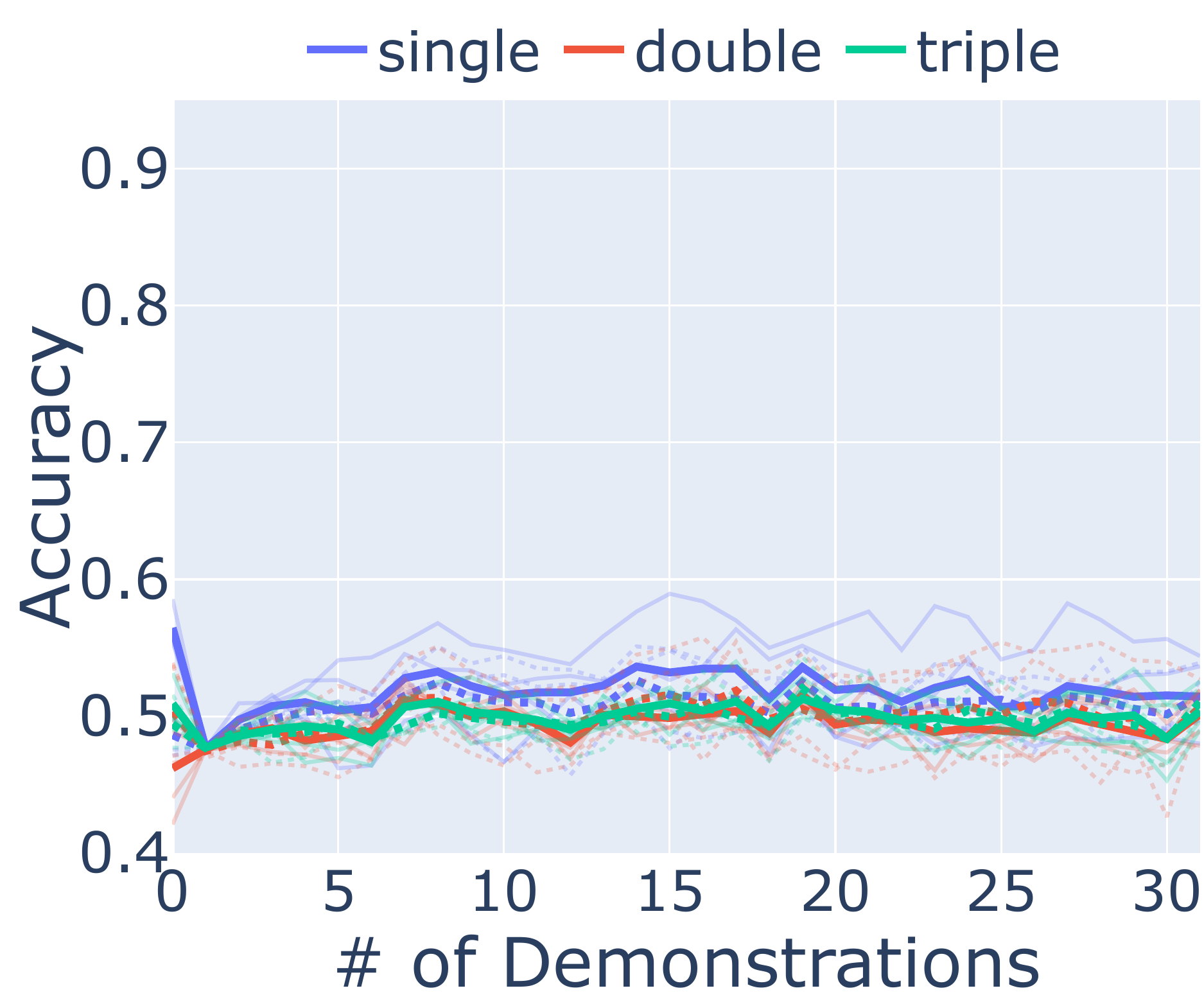}
        \caption{branching, $r_3, r_4$, 5000 examples}
    \end{subfigure}
    \caption{
    In-context learning accuracy with different sizes of 
    }
    \label{fig:icl-no-fuse}
\end{figure*}

\begin{figure*}
    \centering
    \begin{subfigure}{0.32\textwidth}
        \centering
        \includegraphics[width=\textwidth]{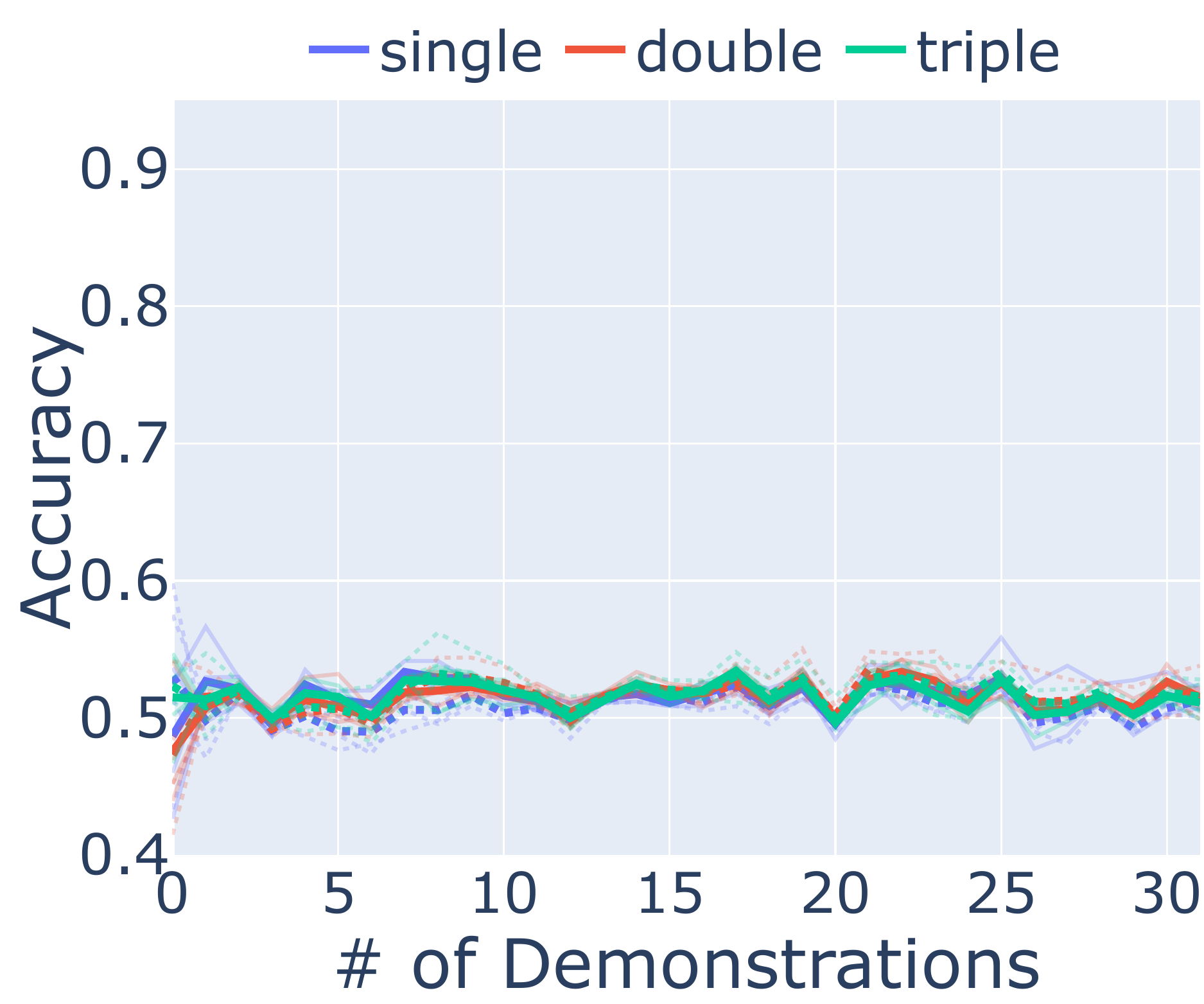}
        \caption{balanced, length = $4$, 3 layers}
    \end{subfigure}
    \begin{subfigure}{0.32\textwidth}
        \centering
        \includegraphics[width=\textwidth]{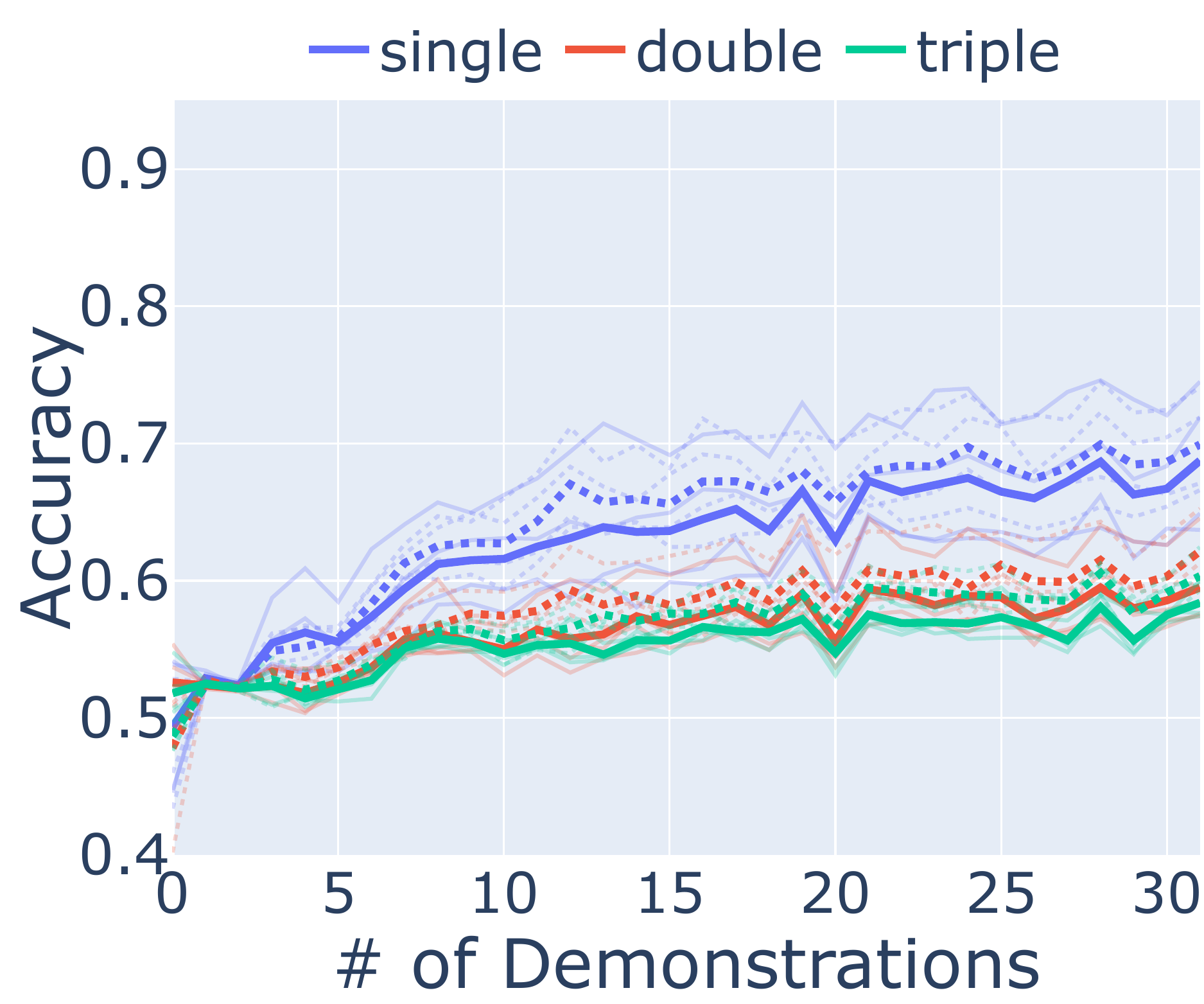}
        \caption{balanced, length = $4$, 4 layers}
    \end{subfigure}
    \begin{subfigure}{0.32\textwidth}
        \centering
        \includegraphics[width=\textwidth]{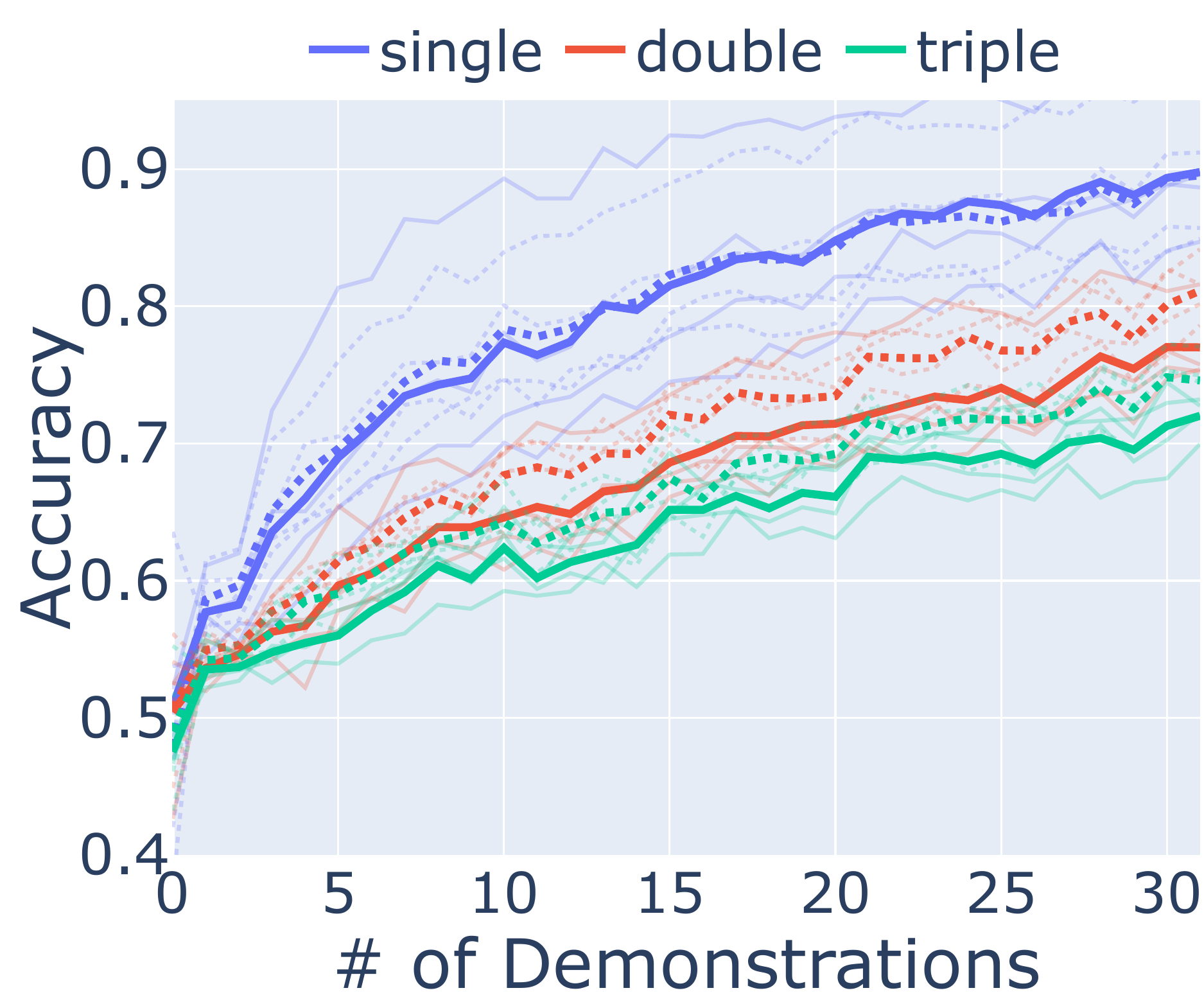}
        \caption{balanced, length = $4$, 5 layers}
    \end{subfigure}

    \begin{subfigure}{0.32\textwidth}
        \centering
        \includegraphics[width=\textwidth]{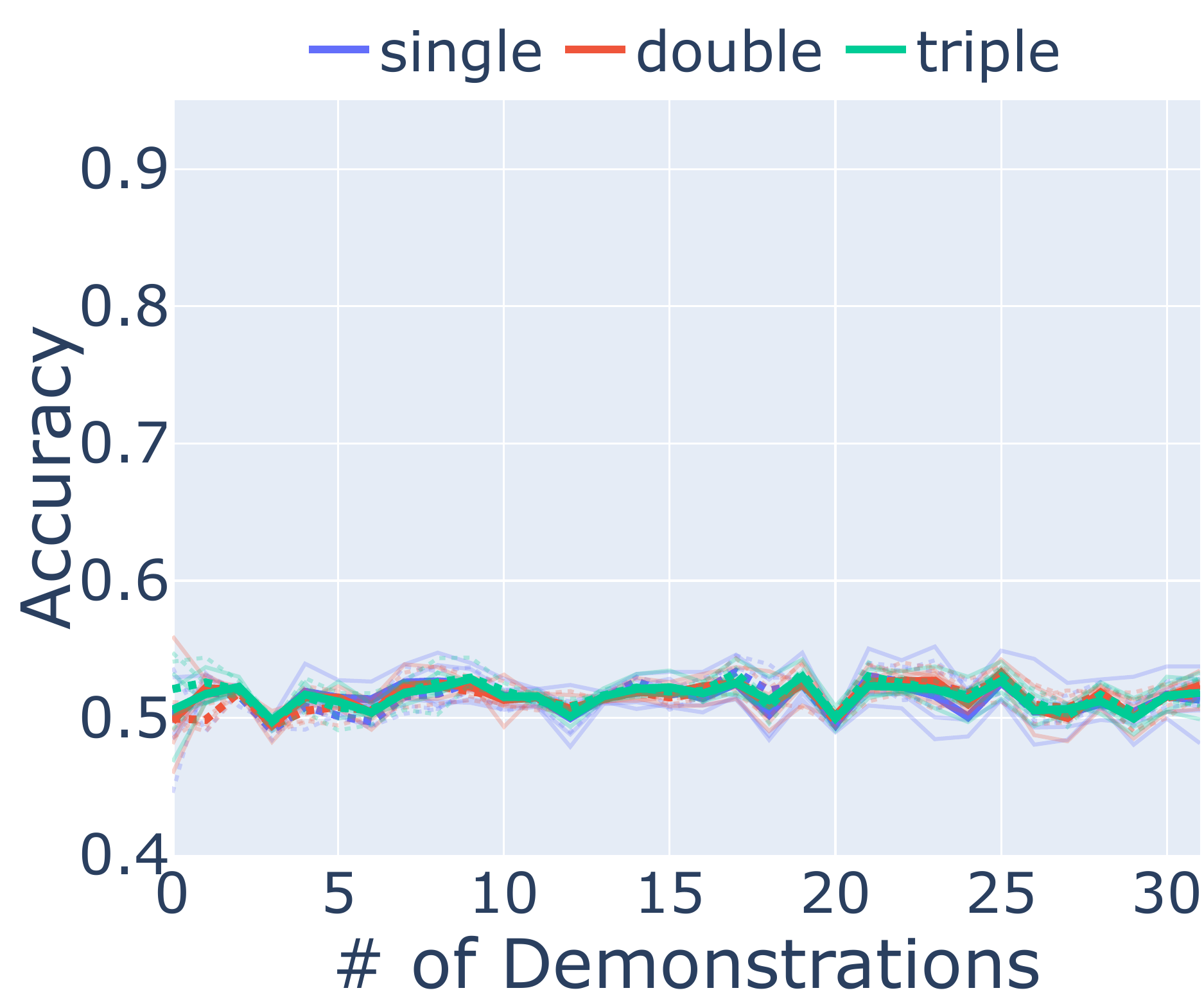}
        \caption{branching, length = $4$, 3 layers}
    \end{subfigure}
    \begin{subfigure}{0.32\textwidth}
        \centering
        \includegraphics[width=\textwidth]{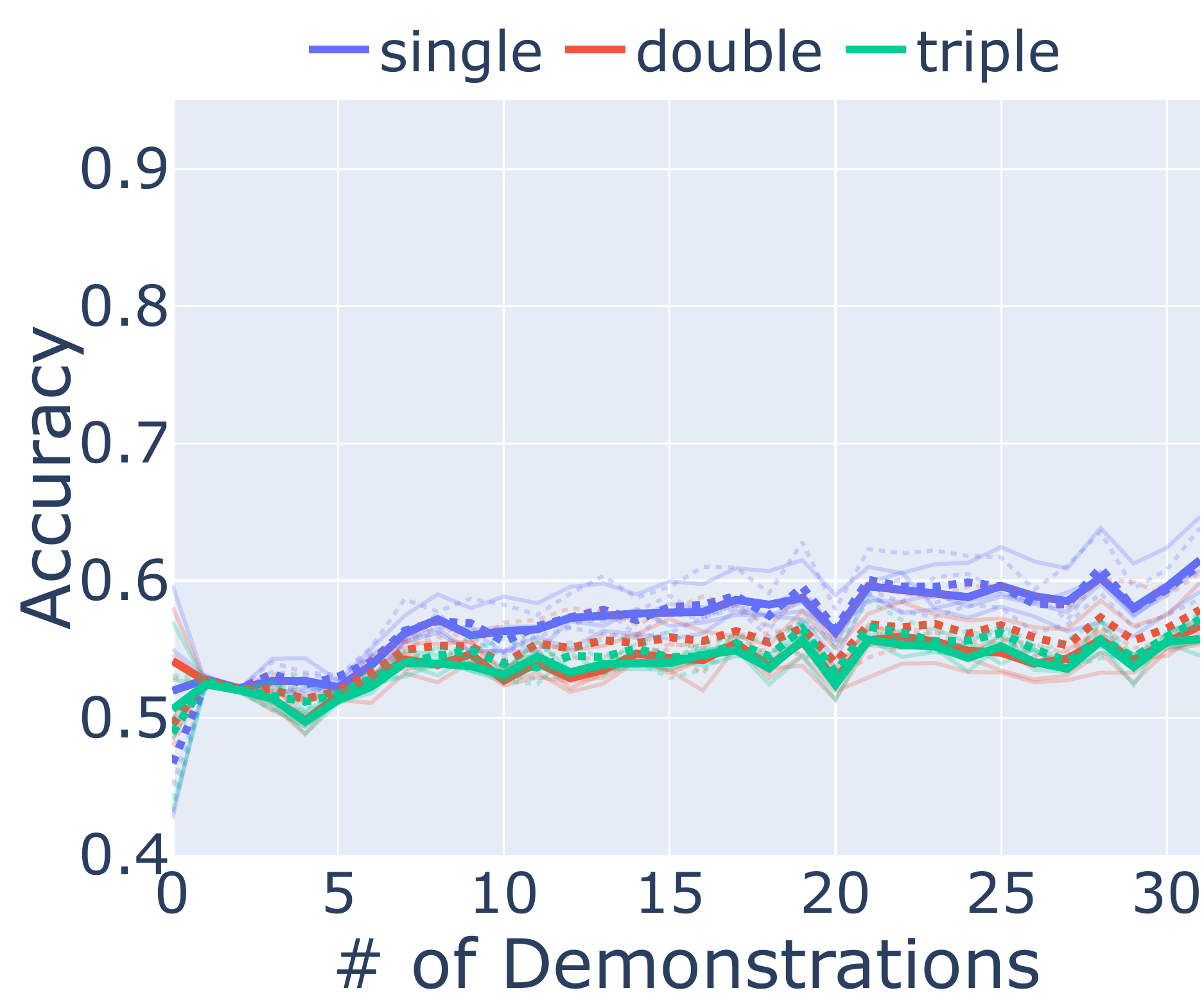}
        \caption{branching, length = $4$, 4 layers}
    \end{subfigure}
    \begin{subfigure}{0.32\textwidth}
        \centering
        \includegraphics[width=\textwidth]{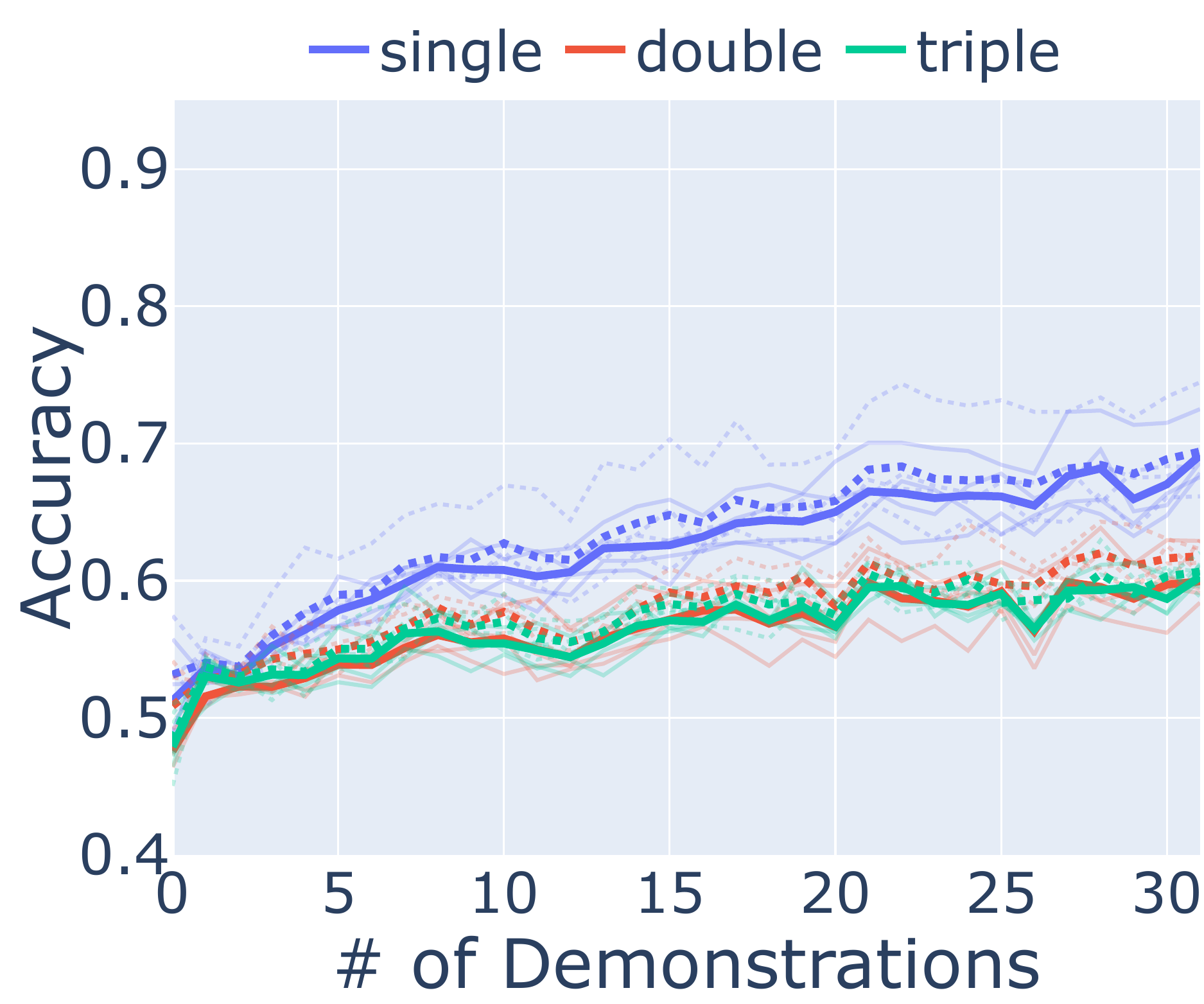}
        \caption{branching, length = $4$, 5 layers}
    \end{subfigure}
    
    \caption{The in-context learning performance when using models with different model depths.}
    \label{fig:icl-depth}
\end{figure*}

\begin{table*}[]
    \centering
\begin{tabular}{c | c  c | c c | c c | c c}
  \toprule
  & \multicolumn{4}{c|}{Branching} & \multicolumn{4}{c}{Balanced} \\
  & \multicolumn{2}{c|}{$r_1, r_2$} & \multicolumn{2}{c|}{$r_3, r_4$} & \multicolumn{2}{c|}{$r_1, r_2$} & \multicolumn{2}{c}{$r_3, r_4$} \\
  Task    & ICL  & CoT  & ICL  & CoT & ICL  & CoT  & ICL  & CoT  \\
  \midrule
  Single & 57.1 & 91.7 & 55.6 & 92.0 & 68.5 & 89.8 & 64.9 & 90.3 \\
  Double & 53.5 & 76.3 & 51.1 & 77.1 & 58.5 & 76.1 & 56.2 & 75.8 \\
  Triple & 53.0 & 73.0 & 51.7 & 73.4 & 57.0 & 68.2 & 54.2 & 67.0 \\ 
  \bottomrule
\end{tabular}
    \caption{The 4-shot accuracy of in-context learning (ICL) versus chain-of-thoughts (CoT).}
    \label{tab:cot-all}

\end{table*}

\begin{table*}[]
\centering
\begin{tabular}{c|ccc|ccc|ccc|ccc}
\toprule
\multicolumn{1}{l}{} & \multicolumn{6}{|c|}{branching}                                 & \multicolumn{6}{c}{balance}                                   \\
\multicolumn{1}{l}{} & \multicolumn{3}{|c|}{$r_1,r_2$} & \multicolumn{3}{c|}{$r_3,r_4$} & \multicolumn{3}{c|}{$r_1,r_2$} & \multicolumn{3}{c}{$r_3,r_4$} \\
\#-shot              & 2        & 4        & 6       & 2        & 4        & 6       & 2        & 4        & 6       & 2        & 4        & 6       \\
\midrule
single               & 49.1     & 89.5     & 84.0    & 59.5     & 92.0     & 86.9    & 58.5     & 86.2     & 85.5    & 50.3     & 90.3     & 89.9    \\
double               & 47.8     & 71.4     & 75.6    & 53.1     & 77.1     & 86.1    & 49.1     & 70.4     & 69.0    & 50.5     & 75.8     & 79.4    \\
triple               & 46.7     & 65.7     & 70.7    & 50.6     & 73.4     & 79.4    & 46.0     & 60.2     & 61.4    & 49.8     & 67.0     & 70.4   \\
\bottomrule
\end{tabular}
\caption{The CoT performance with 2, 4, or 6 examples.}
\end{table*}

\begin{algorithm*}
    \begin{algorithmic}
    \State Sample $N = \{ n_1, n_2, n_3, n_4, n_5 \}$ from either $\{0, 1, \cdots, 9\}$ or $\{6, 7, \cdots, 15\}$.
    \State Sample four operators $O = \{ o_1, o_2, o_3, o_4 \}$ from $\{+, \times \}$.
    \State Initialize an empty list $S$ for the storing the steps.
    
    \While{$\|N\| > 1$}
        \State Randomly sample an index $i$ of an operator from $O$. If $\times \in O$, then make sure that the sampled index corresponds to $\times$. 
        \State Create a step $s \leftarrow [n_i, n_{i+1}, n_i o_i n_{i+1} \mod 16]$ and append $s$ to $S$.
        \State Remove $n_i$ and $n_{i+1}$ from $N$.
        \State Remove $o_i$ from $O$.
        \State Insert $n_i o_i n_{i+1} \mod 16$ into $N$ at position $i$.
    \EndWhile

    \State return $S$.
    \end{algorithmic}
    \caption{Pseudo code for the generation process of a training example for the arithmetic task.}
    \label{alg:summation-train}
\end{algorithm*}

\begin{algorithm*}
    \begin{algorithmic}
    \State Given a list of steps $S = \langle s_1, s_2, \cdots, s_4 \rangle$. (Each step is represented with a list, where the last number in the list is the right-hand side of an equation.)
    \State Initialize n list $S' = \langle s_1 \rangle$ for the storing the randomly merged steps.
    \For{$i \in \{ 2, 3, 4, 5 \}$}
        \State Uniformly sample a number $r$ from $\{0, 1\}$.
        \If{$r = 0$}
            \State Append $s_i$ to $S'$.
        \Else
            \If{$S'[-1][-1] \in s_i[:-1]$}
                \State Remove $S'[-1][-1]$ from $s_i$.
                \State $S'[-1] \leftarrow \text{concatenate}(S'[-1][:-1], s_i)$
            \EndIf
        \EndIf
    \EndFor
    \State return $S'$.
    \end{algorithmic}
    \caption{Pseudo code for the generation process of a training example for the arithmetic task.}
    \label{alg:summation-merge}
\end{algorithm*}

\section{Details of the Arithmetic Task}
\label{sec:summation-details}

\subsection{Training Data}

When generating a training example, we first sample five numbers.
These five numbers are either from $\{0, \cdots ,9 \}$ or $\{6, \cdots, 15\}$.
We then repeat the following process until there is only one number remaining in the list:
We randomly pick two neighbor numbers from the list of five numbers and sum them and insert the result back into the list.
This process results in a sequence of four steps, each of which is an equation that adds two numbers and produces one number.
We outline the algorithm in Algorithm~\ref{alg:summation-train}.

After we generate the four steps, we randomly merge steps into one step.
We iterate through the steps.
If the left-hand side current step contains the result of a previous step, we merge these two step into a single step at probability 0.5.
For example, if the two steps are ``1 + 2 = 3 . 3 + 4 = 7.'' 
We merge these two steps as ``1 + 2 + 4 = 7.'' 
We outline the algorithm in Algorithm~\ref{alg:summation-merge}.

\subsection{Hyper-parameters}

We generate 80000 examples for training, 10000 for validation and 10000 for evaluation in the \textit{seen} setup.
For the \textit{unseen} setup, we generate $6^5 = 7776$ examples.
We train five Transformer models of the GPT-2-small architecture with five random seeds for 15 epochs.
We use batch size 64 and the default training hyper parameters in HuggingFace \texttt{transformers-4.46.3}, i.e.,
with learning rate \texttt{5e-5}, warm-up ratio \texttt{0}, AdamW optimizer.

\section{Dataset License}
\begin{itemize}
    \item SST-2: MIT
    \item MR: unavailable
    \item CR: unavailable
    \item Subj: unavailable
\end{itemize}

\end{document}